\documentclass{article}

\usepackage{silence}
\usepackage[margin=1in]{geometry}

\usepackage{amssymb,amsmath,amsthm,bbm,mathtools}
\usepackage{float}
\usepackage{algorithm}
\usepackage{algorithmicx}
\usepackage[noend]{algpseudocode}
\newtheorem{theorem}{Theorem}

\newtheorem{corollary}{Corollary}
\newtheorem{lemma}{Lemma}

\newtheorem{proposition}{Proposition}

\theoremstyle{definition}

\newtheorem{remark}{Remark}
\newtheorem{example}{Example}

\usepackage{chngcntr}
\usepackage{apptools}
\AtAppendix{\counterwithin{theorem}{section}}
\AtAppendix{\counterwithin{proposition}{section}}
\AtAppendix{\counterwithin{corollary}{section}}
\AtAppendix{\counterwithin{lemma}{section}}

\DeclareMathOperator*{\argmin}{arg\,min}
\DeclareMathOperator*{\argmax}{arg\,max}

\newcommand{\E}{\mathbb{E}}
\newcommand{\R}{\mathbb{R}}

\def\gP{{\mathcal{P}}}

\def\gX{{\mathcal{X}}}

\def\sR{{\mathbb{R}}}

\newcommand{\EE}{\mathbb{E}}

\newcommand{\RR}{\mathbb{R}}

\newcommand{\eps}{\varepsilon}

\usepackage[utf8]{inputenc}
\usepackage[T1]{fontenc}
\usepackage{booktabs}
\usepackage{amsfonts}
\usepackage{nicefrac}
\usepackage{microtype}
\usepackage[svgnames,dvipsnames]{xcolor}
\usepackage{graphicx}
\usepackage{subcaption}
\usepackage{multirow}
\usepackage[shortlabels]{enumitem}
\usepackage[normalem]{ulem}
\usepackage{comment}

\definecolor{AquaGreen}{RGB}{20, 116, 111}
\usepackage[
    unicode=true,              
    pdfencoding=auto,          
    psdextra,                  
    bookmarks=true,            
    bookmarksopen=true,        
    colorlinks=true,           
    linkcolor=RoyalBlue,       
    citecolor={blue!50!black}, 
    urlcolor=AquaGreen,
    pdfdisplaydoctitle=true    
]{hyperref}
\usepackage[capitalize,noabbrev]{cleveref}
\usepackage{crossreftools}
\usepackage[xspace]{ellipsis}
\allowdisplaybreaks

\usepackage{booktabs}
\usepackage{ragged2e}
\usepackage{tabularx}

\newcolumntype{Y}{>{\centering\arraybackslash}X}
\newcolumntype{P}{>{\raggedleft\arraybackslash}X}
\usepackage{multirow}
\newcolumntype{C}[1]{>{\centering\arraybackslash}m{#1}}
\newcolumntype{L}[1]{>{\raggedright\arraybackslash}m{#1}}
\usepackage{colortbl}

\def\sset{\mathsf{S}}

\newcommand{\OPT}{\mathrm{OPT}}
\newcommand{\op}{\mathrm{op}}
\newcommand{\OR}{\mathrm{OR}}

\newcommand{\LS}{\mathrm{LS}}
\newcommand{\MSEstrat}{\mathrm{MSE}_{\mathrm{strat}}}
\newcommand{\supp}{\operatorname{supp}}
\newcommand{\diag}{\operatorname{diag}}
\newcommand{\MSE}{\mathrm{MSE}}

\makeatletter
\newcommand{\blfootnote}[1]{\gdef\@thefnmark{}\@footnotetext{#1}}
\makeatother

\usepackage{natbib}
\setcitestyle{authoryear,open={(},close={)}}

\title{Strategic Feature Selection}

\author{Jivat Neet Kaur$^1$, Pratik Patil$^2$, Divya Shanmugam$^3$, \\Emma Pierson$^1$,
Michael I. Jordan$^{1, 7}$, Nika Haghtalab$^1$, \\Meena Jagadeesan$^{4, 5\,}$\footnote{Equal contribution}, Ahmed Alaa$^{1\,*}$, Serena Wang$^{6\,*}$}
\date{\normalsize{$^1$University of California, Berkeley, $^2$University of Texas, Austin, $^3$Cornell Tech, $^4$Stanford University,
$^5$University of Pennsylvania, $^6$Harvard University, $^7$Inria, Paris}}

\begin{document}

\maketitle

\begin{abstract}
When algorithmic predictors inform resource allocation in high-stakes domains such as healthcare, these predictors must account for strategic manipulation of input features.
The typical solution is to redesign the predictor itself to explicitly account for strategic interactions.
In practice, however, decision makers are often constrained to adjusting coarser levers within existing prediction pipelines. 
For example, healthcare organizations often select which features to exclude based on perceived manipulability, while using standard regularization procedures to shrink the coefficients of retained features.
In this work, we initiate a formal study of strategic classification through \emph{feature selection} and its interaction with \emph{ridge regularization}.
Our main finding is that excluding individual features based on their manipulability alone is generally suboptimal.
We provide a fine-grained characterization of the performance of a feature subset under optimal regularization, yielding new insights for policy design. 
Motivated by this characterization, we develop a practical algorithm for jointly choosing the feature set and the level of ridge regularization.
Through a real-world case study on a healthcare payments benchmark, we illustrate how our algorithm can guide the design of coarse policy levers in practice. 
Our results provide a principled, practical framework for mitigating the effects of strategic behavior in algorithmic decision-making systems.
\end{abstract}

\section{Introduction}
\label{sec:intro}

Algorithmic predictions are increasingly used to inform the allocation of benefits and services in high-stakes domains. 
When these systems determine eligibility or allocation amounts from individuals’ \emph{reported features}, they create incentives for strategic behavior. 
In settings such as large-scale government-run resource allocation programs, organizations serving individuals may manipulate reported features in order to increase the predicted allocation. 
For instance, consider the Medicare Advantage program, in which the government pays private insurers to provide coverage for beneficiaries based on patient risk scores~\citep{geruso2020upcoding}. 
In this setting, insurers are known to report more or higher severity diagnoses for their enrollees in order to receive higher payments from the government~\citep{doj2026kaiser}. 
This practice, known as ``upcoding,'' has become increasingly controversial and is projected to cost the government \$40 billion in additional spending in 2025, with no clear benefit to the beneficiaries~\citep{MedPAC2025}.

These types of incentives are captured by the framework of strategic classification~\citep{hardt2016strategic, podimata2025incentive}, which models the strategic interaction between a decision maker who specifies the prediction rule and an agent who can modify features in response to that rule. 
In the Medicare Advantage setting, the agent is an organization that can modify the reported features of the individuals it serves in order to obtain higher payments. Organizations incur costs to manipulate features: for example, in health insurance, insurers pay chart review contractors to identify additional diagnosis codes. 
The typical approach in this framework is to design a prediction rule that is less vulnerable to strategic manipulation.

Much of the strategic classification literature focuses on computing the \textit{strategic optimum}, i.e., prediction rules that are optimal once these strategic responses are taken into account.  
In large-scale policy settings, however, structural redesign faces significant institutional inertia and systems are rarely rebuilt from scratch.
Models are embedded in regulatory and operational workflows, updated incrementally, and substantial departures from standard pipelines tend to be difficult to justify and defend in regulatory rulemaking. 
As a result, decision makers often continue to use legacy prediction pipelines and instead leverage a small set of \emph{coarse policy levers} to address strategic responses.
This raises a different question than the one typically studied in strategic learning: not how to redesign the entire predictor, but how to address strategic behavior using only the coarse levers that institutions are constrained to use. 

In practice, it is common to combat strategic behavior using feature selection together with regularization. 
One common intervention is to \emph{exclude features} that are perceived to be especially manipulable from the model~\citep{bjorkegren2025manipulation, kronick2025are}. 
For example, to counteract the effect of upcoding, Centers for Medicare \& Medicaid Services (CMS) excludes features (i.e., diagnoses) that are at risk of inappropriate coding by health insurers~\citep{CMS2014AdvanceNotice, CMS2024AdvanceNotice}. 
In 2024, this included removing the health conditions corresponding to Protein-Calorie Malnutrition and Angina Pectoris from the payment model to limit the model's sensitivity to higher coding intensity~\citep{CMS2024AdvanceNotice}.
Another intervention is to \emph{use regularization to shrink coefficients} on the retained features, often implemented through constrained regression. 
For instance, CMS recently constrained the coefficients on related diagnoses such as uncomplicated diabetes and diabetes with complications so that they carry the same weight in the payment model~\citep{CMS2024AdvanceNotice}. 
Given that such heuristics are used in practice, this raises a central question: 
\begin{center}
    \emph{How should policymakers design feature selection together with regularization to combat strategic behavior?} 
\end{center}

\subsection{Our contributions}

In this work, we initiate a formal study of strategic learning using \emph{feature selection} and its interaction with \textit{ridge regularization}. 
Our main finding is that excluding individual features based on their manipulability alone is suboptimal: instead, features should be selected \emph{jointly} based on their manipulability and predictability relative to the other available features. 

\paragraph{Theoretical results.}
We theoretically characterize the performance of a feature set under optimally tuned regularization, comparing it to the strategic optimum had the pipeline been redesigned.
First, we establish a lower bound on this performance gap, identifying an irreducible loss that remains no matter which feature set is chosen (\Cref{thm:lower-bound}).
Next, we establish an upper bound that characterizes when a feature subset is near-optimal up to this irreducible gap (\Cref{thm:global-upper}).
Specifically, our characterization reveals that the performance of a feature set is not determined by manipulability alone, but by the joint predictability-manipulability structure of its features and by how effectively regularization can leverage that structure. 
Our results also illustrate special cases where the coarse levers exactly achieve the irreducible gap.

\paragraph{Policy implications.} 
Our theoretical results provide concrete insights for policymakers (Section \ref{sec:implications}). 
First, regularization can change the optimal feature set, so the subset that is best without regularization need not remain best under ridge regression; hence, feature selection and regularization must be tuned jointly (\Cref{fig:decomposition-tradeoff}). 
Second, a highly manipulable feature may still be kept if it is predictive and other retained features have similar manipulability (\Cref{fig:homogeneity-phase-main}). 
Finally, dropping a manipulable feature becomes desirable when a less manipulable proxy is available (\Cref{fig:supp-correlated-proxy}). 
These insights run counter to prevailing policy views that highly manipulable features should simply be excluded and that these two levers can be treated as separate design choices~\citep{cms2024riskadjustment, kronick2025are}.

\paragraph{Algorithmic and empirical results.} 
To solve the corresponding design problem, we develop a practical algorithm for jointly choosing the feature set and regularization level in the strategic setting. 
We propose a two-stage procedure that performs a continuous relaxation of the combinatorial support selection optimization problem, followed by local support refinement (\Cref{sec:algorithm}). 
We complement our theoretical analysis with a healthcare payments case study with simulated upcoding data calibrated to realistic Medicare Advantage coding patterns documented in recent health policy work~\citep{kronick2025are}  (\Cref{sec:experiments}). 
Empirical evaluation shows that our proposed method substantially improves strategic robustness while preserving much of the predictive accuracy of full-support models. 
Notably, our results reveal that high manipulability alone does not necessitate exclusion.

\paragraph{Discussion.} 
Altogether, our results provide a theoretically principled and practical framework for policymakers to design and evaluate feature selection and regularization policies in the presence of strategic behavior. 
They clarify when these coarse levers are sufficient to achieve near-optimal performance under strategic interactions, and how these levers should be designed jointly to optimize for strategic performance. 
At the same time, they highlight an important disconnect between existing heuristic practices and the true design problem, yielding nuanced implications for policy design. 
Finally, these levers offer robustness benefits under cost uncertainty (Section \ref{sec:cost-uncertainty}), illustrating their value as practical tools for strategic decision-making.

\subsection{Related work}

\paragraph{Strategic classification.} 
We build on a growing line of work in strategic classification. These works aim to learn optimal prediction rules when decision subjects can manipulate their features at a cost (e.g., \citep{hardt2016strategic, dong2018strategic,chen2020learning, shavit2020causal,levanon2021strategic, bjorkegren2025manipulation}). 
Closely related to our work, \citet{bjorkegren2025manipulation} study manipulation-robust prediction in a setting similar to ours, under linear decision rules and quadratic manipulation costs.    
These works fall within a broader literature on incentive-aware machine learning (see \citet{podimata2025incentive} for a survey), which also studies strategic behavior from an improvement and causality perspective~\citep{miller2020strategic,haghtalab2020maximizing, kleinberg2020how, shavit2020causal}, considering issues of incentive design when manipulations can lead to genuine improvement in individual outcomes and welfare. 
In contrast with all of these works, we study strategic classification through \emph{coarse policy levers} within a fixed prediction pipeline rather than by redesigning the pipeline to achieve the strategically optimal solution.

Closely related in spirit to our work,~\citet{handina2024understanding} show that in strategic environments, larger or more expressive model classes need not improve or may even worsen equilibrium performance. 
Our contribution is conceptually and mathematically distinct: rather than comparing abstract model classes, we study the concrete and widely used design levers of feature selection and regularization, and characterize the performance of feature selection under optimal regularization.

\paragraph{Policy.} 
Health policy also contends with strategic behavior, where payment models based on recorded diagnoses are used to allocate payments to insurers. 
In this setting, insurers have been shown to strategically engineer recorded diagnoses to increase payment allocations~\citep{doj2026kaiser}. 
The health policy literature has put forth several proposals to combat such behavior.
\citet{rose2016machine} shows that a payment formula with a reduced set of variables can retain much of the predictive performance of richer formulas, potentially reducing incentives for aggressive upcoding. 
More broadly, this literature has proposed feature selection and coefficient restriction in such payment models. 
These proposals suggest, for example, excluding diagnoses added to a patient's health record via chart review~\citep{meyers2021medicare}, incorporating patient survey data~\citep{bellerose2025combining, mcwilliams2025use}, or excluding diagnoses coded most differentially by insurers in order to reduce overpayments~\citep{kronick2025are}. 
Closer to the levers we study,~\citet{mcguire2021} show how constrained regression and variable selection can improve the performance of risk adjustment systems. 
In a similar vein,~\citet{arias2025selecting} proposes a LASSO-based method that balances predictive accuracy against potential adverse incentives, explicitly accounting for heterogeneity in the adverse incentives associated with each variable. 

Related concerns arise beyond health policy in problems where the goal is to allocate limited resources to eligible individuals in order to maximize welfare.
In public economics, proxy means tests are eligibility rules that use easier-to-observe variables such as household characteristics and asset ownership as proxies for income to target social programs to those in most need. 
Strategic misreporting of individuals' characteristics is an important policy concern in such settings~\citep{martinelli2009deception, camacho2011manipulation}. 
Work on proxy means tests emphasizes that the variables used to determine eligibility should be both predictive of welfare and verifiable in practice~\citep{grosh1995proxy}.
\citet{niehaus2013targeting} study how rules for targeting benefits are designed and show that adding more criteria can create a tradeoff between statistical gains and enforceability.

Our work extends this agenda by providing a formal strategic framework for characterizing the performance of feature selection and coefficient restriction decisions, and in doing so contributes new perspectives on longstanding policy issues.

\paragraph{Adversarial robustness.} 
Our work is also related to adversarially robust feature selection and poisoning-robust feature selection~\citep{zhang2016adversarial,xiao2015is}, which show that adversary-aware feature selection can improve robustness to evasion and poisoning attacks. 
Our setting differs, however, in the threat and perturbation models: these papers study exogenous attacks on the data, whereas we study endogenous manipulation induced by the deployed predictor itself.

\section{Model and Preliminaries}
\label{sec:model}

We study the strategic interaction between a decision maker who specifies a predictor and organizations that can modify the reported features of the individuals they serve. 
Building on the common linear specification in strategic learning~\citep{shavit2020causal,kleinberg2020how, harris2022strategic}, the decision maker publishes a predictor $f_{\theta,b}(x)=\theta^\top x + b$ parameterized by $\theta \in \sR^d$ and $b \in \sR$, which maps an individual's feature vector $x \in \gX \subseteq \sR^d$ to a predicted outcome.

We model the population as a distribution $\gP_{X,Y}$ over $(X, Y)$, where $X \in \RR^d$ denotes the feature vector and $Y \in \RR$ denotes the outcome. 
We assume that $X$ and $Y$ have finite second moments, and define the feature covariance matrix $\Sigma := \mathbb E[(X - \E[X])(X - \E[X])^\top] \succ 0$. 
For simplicity, we assume $\E[X] = 0$ and $\E[Y] = 0$. 
This centering is without loss of generality, as nonzero means can be absorbed into the intercept. 

Because our focus is on linear predictors, we define $\theta^*$ to be the coefficient of the best linear projection of $Y$ onto $X$:
\begin{equation}
\label{eq:projection-parameter}
    \theta^* := \argmin_{\theta \in \RR^d}\; \EE[(Y - \theta^\top X)^2].
\end{equation}
Since $\Sigma \succ 0$, this minimizer is unique and is given by $\theta^* = \Sigma^{-1} \EE[X Y]$.
We can therefore write $Y = \theta^{*\top} X + \eps$ where $\eps := Y - \theta^{*\top} X$ is the projection residual.
By the normal equations, the residual is orthogonal to the features: $\EE[X \eps] = 0$.
Under our centering conventions, $\EE[\eps] = 0$ as well.
Let $\sigma^2 := \EE[\eps^2]$.

Following~\citet{hardt2016strategic, bjorkegren2025manipulation}, we assume that the decision maker has access to \emph{unmanipulated} samples from $\gP_{X,Y}$. 
Such data can be obtained from a pre-deployment period, before the predictor $f_{\theta, b}$ is used, or from an alternative policy regime under which organizations are not incentivized to manipulate their features.
Using these unmanipulated samples, $\theta^*$ can be estimated. 

\paragraph{Organization's best response model.}
The organization observes an individual’s features $x$ and chooses an action $a \in \sR^d$, which changes the reported features to $x+a$ at cost $C(a) > 0$.
Adopting the quadratic manipulation costs in earlier work~\citep{shavit2020causal, bjorkegren2025manipulation}, we model the cost function as $C(a) = \frac{1}{2}a^\top H a$, where $H \in \mathbb{R}^{d \times d} \succ 0$ is the cost matrix. 
The matrix $H$ allows manipulation costs to differ across features and to exhibit interactions across coordinates.

Given a published predictor, we assume that organizations \emph{best-respond} to $f_{\theta, b}$ by choosing an action
\begin{equation}
\label{eq:agent_br}
    a^*(x;\theta) \in \argmax_{a \in \sR^d} \; \Bigl\{\theta^\top(x+a)+b-\tfrac12 a^\top H a\Bigr\}.
\end{equation}
From here on, we suppress the dependence on $x$ and write $a^*(\theta)$.

\paragraph{Decision maker's objective.}
The decision maker seeks to predict the true outcome as accurately as possible in the presence of strategic responses, which means minimizing prediction error under the deployed policy. 
Accordingly, the relevant loss is not the usual mean squared error $\mathrm{MSE}(\theta,b) := \EE[(\theta^\top X + b - Y)^2]$, but the mean squared error after the organization best responds:
\begin{equation}
\label{eq:mse-strat}
\mathrm{MSE}_{\mathrm{strat}}(\theta, b) = \mathbb{E}[(\theta^\top (X + a^*(\theta)) + b - Y)^2 ].
\end{equation}

As a benchmark, we use the notion of the \emph{strategic optimum} \citep{hardt2016strategic}, allowing the decision maker to optimize over all coefficients in $\RR^d$ and intercept in $\RR$. The unrestricted strategic optimum is
\begin{equation}
\label{eq:opt}
\OPT:=\min_{\theta\in\R^d,\,b\in\R}\MSEstrat(\theta,b).
\end{equation}

\subsection{Policy levers: Feature selection and regularization}
\label{sec:policy-levers}

Our analysis focuses on two \emph{coarse} policy levers that are leveraged in practice to combat strategic behavior within existing prediction pipelines. 
Throughout this section, the estimators are defined with respect to the unmanipulated distribution $\gP_{X,Y}$.

\paragraph{Feature selection.}
Intuitively, feature selection provides a direct way to limit strategic manipulation by excluding features whose vulnerability to gaming outweighs their predictive value. 
The decision maker chooses a subset of features $\mathsf S \subseteq [d]$ and fits a linear predictor using only $X_{\mathsf S}$. 
Let $\mathsf R=[d]\setminus\mathsf S$ denote the omitted coordinates, and define the coordinate subspace $\Theta(\mathsf S):=\{\theta\in\R^d:\supp(\theta)\subseteq \mathsf S\}$. 
The support-restricted least squares estimator is:
\begin{equation}
\label{eq:subset-ls-compact}
(\theta^{\mathrm{LS}}(\mathsf S), b^{\mathrm{LS}}(\mathsf S))
\in
\argmin_{\theta\in\Theta(\mathsf S), b \in \mathbb{R}}
\ \mathbb E[(Y-\theta^\top X -b)^2].
\end{equation}

\paragraph{Ridge regularization.}
Regularization reduces the magnitude of model coefficients and thereby weakens incentives to manipulate the retained features. 
We study the canonical form of regularization given by ordinary ridge regression with scalar regularization parameter $\lambda \geq 0$. The ordinary ridge estimator is:
\begin{equation}
\label{eq:ordinary-ridge-compact}
(\theta^{\mathrm{OR}}(\lambda), b^{\mathrm{OR}}(\lambda))
\in
\argmin_{\theta\in\sR^d, b \in \mathbb{R}}
\ \E[(Y-\theta^\top X -b)^2]+\lambda\|\theta\|_2^2.
\end{equation}
These two levers are often deployed jointly, with some features dropped entirely and the remaining coefficients regularized.
To capture this combination, we consider the \emph{support-restricted ridge} estimator. 
For a subset $\mathsf S \subseteq [d]$ and regularization level $\lambda \ge 0$, define:
\begin{equation}
\label{eq:subset-ridge-compact}
(\theta^{\mathrm{OR}}(\mathsf S,\lambda), b^{\mathrm{OR}}(\mathsf S,\lambda))
\in
\argmin_{\theta\in\Theta(\mathsf S), b \in \mathbb{R}}
\ \E[(Y-\theta^\top X-b)^2]+\lambda\|\theta\|_2^2.
\end{equation}
This estimator nests the previous two cases:
(i) When $\lambda = 0$, it reduces to the support-restricted least squares estimator: $\theta^{\mathrm{OR}}(\mathsf S,0)=\theta^{\mathrm{LS}}(\mathsf S)$.
(ii) When $\mathsf S=[d]$, it reduces to ordinary ridge: $\theta^{\mathrm{OR}}([d],\lambda)=\theta^{\mathrm{OR}}(\lambda)$.
In our theoretical and empirical results, we will show that combining feature selection with regularization can yield larger gains than using either lever alone.

All three estimators \eqref{eq:subset-ls-compact}, \eqref{eq:ordinary-ridge-compact}, and \eqref{eq:subset-ridge-compact} admit closed-form expressions, as detailed in \Cref{lem:policy-lever-closed-forms}.

\section{Theoretical Results}
\label{sec:theory}

We compare the strategic $\mathrm{MSE}$ achieved by the support-restricted ordinary ridge estimator to the strategic optimum $\OPT$. 
Our goal is to characterize when the coarse policy levers of feature selection and ridge regularization can come close to this benchmark. 
In \Cref{sec:lower-bound}, we establish a lower bound on the optimality gap, which identifies an irreducible gap incurred by our estimator class. 
In \Cref{sec:upper-bound}, we establish an upper bound that gives interpretable conditions under which optimally tuned support-restricted ridge can nearly match $\OPT$ up to this irreducible loss. 

\subsection{Lower bound on the excess strategic MSE}
\label{sec:lower-bound}

We start by characterizing the fundamental limits of these coarse policy levers through a lower bound. 
As a tool for interpretation, we introduce the notation $\widetilde{\OPT}:= \min_{\theta\in\mathbb R^d} \mathrm{MSE}_{\mathrm{strat}}(\theta, 0)$ and $\widetilde{\OPT}(\mathsf S):= \min_{\theta\in\Theta(\mathsf S)} \mathrm{MSE}_{\mathrm{strat}}(\theta, 0)$.

\begin{theorem}[Lower bound on excess strategic MSE]
\label{thm:lower-bound}
For any subset $\mathsf S \subseteq [d]$ and $\lambda \ge 0$,
\begin{equation}
\label{eq:mse-strat-lb}
\mathrm{MSE}_{\mathrm{strat}}(\theta^{\mathrm{OR}}(\mathsf S,\lambda), b^{\mathrm{OR}}(\mathsf S,\lambda))
-
\OPT \ge \widetilde{\OPT} - \OPT.
\end{equation}
\end{theorem}

\begin{proof}[Proof sketch.]
Although $\OPT$ is defined by optimizing over both coefficients and an intercept, every estimator considered here is fit using a centered support-restricted ridge program and therefore has intercept $0$ under our centering convention, i.e., $b^{\mathrm{OR}}(\mathsf S,\lambda) = 0$ (see \Cref{lem:policy-lever-closed-forms}).
For a fixed support $\mathsf S$, the estimator $\theta^{\mathrm{OR}}(\mathsf S,\lambda)$ belongs to the class $\Theta(\mathsf S)$, and therefore cannot outperform the support-restricted zero-intercept benchmark $\widetilde{\OPT}(\mathsf S)$.
Moreover, $\widetilde{\OPT}(\mathsf S)$ is monotonically decreasing as the support $\mathsf S$ grows.
\end{proof}

\Cref{thm:lower-bound} shows that the limits of the coarse levers can be characterized by a gap $\widetilde{\OPT} - \OPT$.
This irreducible gap is small whenever the strategic shift induced by the best linear projection is small, for example when the manipulation costs are high in the predictive directions of $\theta^*$: $0\le \widetilde{\OPT}-\OPT \le (\theta^{*\top}H^{-1}\theta^*)^2$ (stated formally in \Cref{prop:irreducible-error-bounds}).

In the remainder of the section, we therefore upper bound the additional loss from support restriction and ridge regularization relative to $\widetilde{\OPT}$; the total excess relative to $\OPT$ is obtained by adding this irreducible gap. 
For notational simplicity, throughout the remainder of the paper we write $\mathrm{MSE}_{\mathrm{strat}}(\theta^{\mathrm{OR}}(\mathsf S,\lambda))$ for $\mathrm{MSE}_{\mathrm{strat}}(\theta^{\mathrm{OR}}(\mathsf S,\lambda), b^{\mathrm{OR}}(\mathsf S,\lambda)) = \mathrm{MSE}_{\mathrm{strat}}(\theta^{\mathrm{OR}}(\mathsf S,\lambda), 0)$, using that $b^{\mathrm{OR}}(\mathsf S,\lambda) = 0$.

\subsection{Upper bound on the excess strategic MSE}
\label{sec:upper-bound}

We now turn to upper bounds on the $\mathrm{MSE}_{\mathrm{strat}}(\theta^{\mathrm{OR}}(\mathsf S,\lambda))$. 
Since $\widetilde{\OPT}$  serves as a lower bound per \Cref{thm:lower-bound}, we bound the gap between $\mathrm{MSE}_{\mathrm{strat}}(\theta^{\mathrm{OR}}(\mathsf S,\lambda))$ and $\widetilde{\OPT}$.  
This upper bound illustrates that feature selection should be based on \emph{relative} manipulability and predictability.

\begin{theorem}[Upper bound on the excess strategic MSE]
\label{thm:global-upper}
For any subset $\mathsf S \subseteq [d]$,
\begin{equation}
\label{eq:mse-strat-ub}
\min_{\lambda \ge 0}\, \mathrm{MSE}_{\mathrm{strat}}(\theta^{\mathrm{OR}}(\mathsf S,\lambda))
- \widetilde{\OPT}
\le
-\;\underbrace{\big(\Gamma([d])-\Gamma(\sset)\big)}_{\text{\footnotesize\itshape manipulability gain}}
+
\underbrace{L_{\mathrm{pred}}(\sset)}_{\text{\footnotesize\itshape predictive loss}}\;
+
\underbrace{C_H(\mathsf S)\,\delta_H(\mathsf S)^2}_{\text{\footnotesize\itshape heterogeneity gap}},
\end{equation}
where the strategic burden $\Gamma$, predictive loss $L_{\mathrm{pred}}$, and heterogeneity defect $\delta_H$ for a support $\mathsf S$ are:
\begin{enumerate}
[nosep,itemsep=2pt,leftmargin=2em]
\item
$
\Gamma(\mathsf S)
:=
\min_{\eta\in\RR^{|\mathsf S|}}
\left\{
(\eta-\theta_{\mathsf S}^{\LS}(\mathsf S))^\top
\Sigma_{\mathsf{SS}}
(\eta-\theta_{\mathsf S}^{\LS}(\mathsf S))
+
(\eta^\top H^{-1}_{\mathsf{SS}}\eta)^2
\right\},
$
\item
$L_{\mathrm{pred}}(\mathsf S)
:=
(\theta^*_{\mathsf R})^\top
\Sigma_{\mathsf R\mid\mathsf S}
\theta^*_{\mathsf R}$
and
$
\Sigma_{\mathsf R\mid\mathsf S}
:=
\Sigma_{\mathsf{RR}}
-
\Sigma_{\mathsf{RS}}\Sigma_{\mathsf{SS}}^{-1}\Sigma_{\mathsf{SR}}$
for $\mathsf R=[d]\setminus\mathsf S$,
\item
$
\delta_H(\mathsf S)
:=
\inf_{\gamma\ge0}
\big\|
\Sigma_{\mathsf{SS}}^{-1/2}
(H^{-1}_{\mathsf{SS}}-\gamma I)
\Sigma_{\mathsf{SS}}^{-1/2}
\big\|_{\op},
$
\end{enumerate}
and $C_H(\mathsf S):=4\,L_H(\mathsf S)\,\|\Sigma_{\mathsf{SS}}^{-1/2} H^{-1}_{\mathsf{SS}} \Sigma_{\mathsf{SS}}^{-1/2}\|_{\op}^2\,\|\Sigma_{\mathsf{SS}}^{1/2}\theta_{\mathsf S}^{\mathrm{LS}}(\mathsf S)\|_2^6$,  $L_H(\mathsf S):=1+6\|\Sigma_{\mathsf{SS}}^{-1/2} H^{-1}_{\mathsf{SS}} \Sigma_{\mathsf{SS}}^{-1/2}\|_{\op}^2\|\Sigma_{\mathsf{SS}}^{1/2}\theta_{\mathsf S}^{\mathrm{LS}}(\mathsf S)\|_2^2$.
\end{theorem}

\Cref{thm:global-upper} gives a fine-grained characterization of the performance of a feature set under optimal regularization.
Note here we write $\theta_{\mathsf S}^{\LS}(\mathsf S) := \Sigma_{\mathsf{SS}}^{-1}\Sigma_{\mathsf S y}$ for the retained coordinates of the support-restricted least squares estimator $\theta^{\LS}(\mathsf S)=(\theta_{\mathsf S}^{\LS}(\mathsf S),0_{\mathsf R})$, where $\Sigma_{\mathsf S y}:=\E[X_{\mathsf S}Y] = \Sigma_{\mathsf{SS}}\theta^*_{\mathsf S} + \Sigma_{\mathsf{SR}}\theta^*_{\mathsf R}$ for $\mathsf R=[d]\setminus\mathsf S$.

The first term, $-(\Gamma([d])-\Gamma(\sset))$, is the \emph{manipulability gain}. 
For a given feature set, $\Gamma(\sset)$ captures how much performance is lost because the predictor must move away from the best-fitting coefficients on $\mathsf S$ in order to account for manipulability: the first term $(\eta-\theta_{\mathsf S}^{\LS}(\mathsf S))^\top \Sigma_{\mathsf{SS}} (\eta-\theta_{\mathsf S}^{\LS}(\mathsf S))$ measures the fit lost by deviating from the least squares fit on $\mathsf S$, and the second term $(\eta^\top H^{-1}_{\mathsf{SS}}\eta)^2$ measures the loss due to manipulability.
Thus, manipulability gain is the \emph{strategic burden} relieved by restricting attention to $\mathsf S$ instead of the full feature set $[d]$. 
This reflects the existing intuition that excluding manipulable features lowers the strategic $\mathrm{MSE}$. 
As we observe, however, manipulability gain alone does not determine excess strategic $\mathrm{MSE}$. The performance of a feature set depends on the joint prediction-manipulation structure.

The second term, $L_{\mathrm{pred}}(\sset)$, is the usual non-strategic \emph{predictive loss} from omitting features. 
It is small when the omitted features are redundant given the retained support.
More precisely, it measures the predictive signal in the excluded coordinates $\mathsf R$ that cannot be linearly recovered from $\sset$. 

The third term, $C_H(\mathsf S)\,\delta_H(\mathsf S)^2$, is the \emph{heterogeneity gap} on feature set $\mathsf S$. 
The quantity $\delta_H(\sset)$ is small when the features in $\mathsf S$ have sufficiently homogeneous manipulation costs (i.e., $H^{-1}_{\mathsf{SS}}$ is well approximated by a scalar multiple of the identity). 
In this regime, a single regularization level can shrink the retained coefficients in a coordinated way, so optimally tuned ridge regularization reduces the heterogeneity gap and thereby lowers excess strategic $\mathrm{MSE}$. \Cref{fig:decomposition-tradeoff} illustrates the three terms.

\begin{figure}[!t]
    \centering
    \includegraphics[width=0.9\linewidth, trim=0cm 0.5cm 0cm 0.5cm, clip]{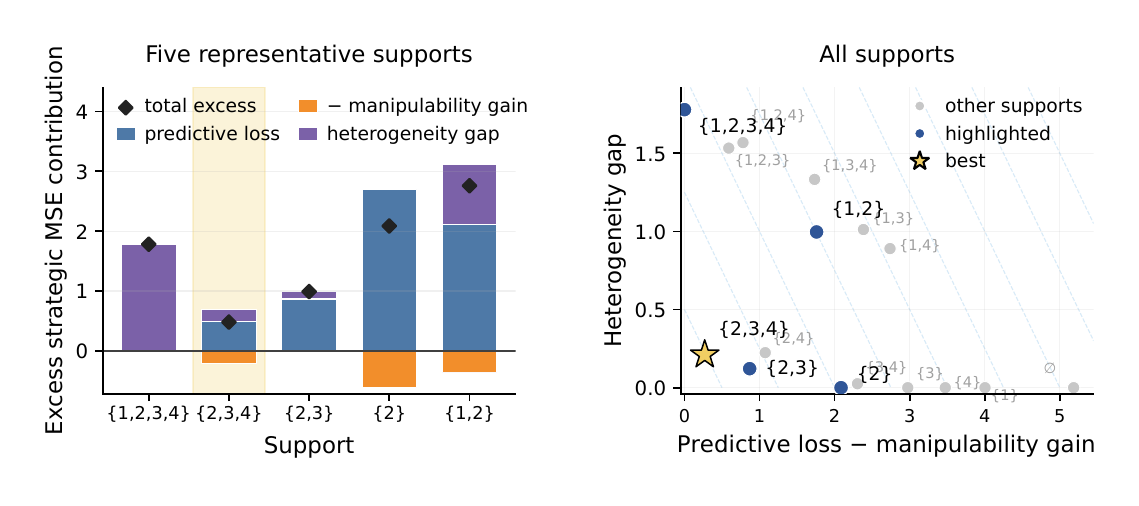}
    \caption{
    \textbf{The best subset depends on the joint predictability-manipulability structure.} 
    In a $d = 4$ example, feature $1$ is highly predictive but also highly manipulable. 
    Feature $2$ is strongly correlated with feature $1$ and therefore acts as a less manipulable proxy. 
    Features $3$, $4$ are less manipulable features with moderate predictive value.
    (Left) Excess strategic MSE decomposition for five representative supports.
    (Right) All supports plotted by predictive loss minus manipulability gain and heterogeneity gap; dashed lines indicate equal excess strategic $\MSE$.
    The optimal set $\{2,3,4\}$ is neither the full support nor the support with the largest manipulability gain, but the one that best balances predictive loss, manipulability gain, and cost heterogeneity. See Appendix~\ref{sec:fig:decomposition-tradeoff-details} for details.
    }
    \label{fig:decomposition-tradeoff}
\end{figure}

The complete proof for \Cref{thm:global-upper} is deferred to Appendix~\ref{sec:proof-upper}. 
Here, we provide the main intuition through a proof sketch.

\begin{proof}[Proof sketch.]
There are two steps.
First, for any fixed $\lambda$, we decompose the excess strategic $\MSE$ by adding and subtracting the benchmark $\widetilde{\OPT}(\mathsf S)$:
\begin{align*}
\mathrm{MSE}_{\mathrm{strat}}(\theta^{\mathrm{OR}}(\mathsf S,\lambda)) - \widetilde{\OPT} = \underbrace{\Bigl(\widetilde{\OPT}(\mathsf S)-\widetilde{\OPT}\Bigr)}_{\text{\footnotesize\itshape support restriction cost}} + \underbrace{\Bigl(\mathrm{MSE}_{\mathrm{strat}}(\theta^{\mathrm{OR}}(\mathsf S,\lambda))-\widetilde{\OPT}(\mathsf S)\Bigr)}_{\text{\footnotesize\itshape scalar-ridge expressivity cost}}.
\end{align*}
The support restriction identity in \Cref{prop:support-restriction} gives $\widetilde{\OPT}(\mathsf S)-\widetilde{\OPT} = L_{\mathrm{pred}}(\mathsf S) - (\Gamma([d])-\Gamma(\mathsf S))$.
Second, for any fixed support $\mathsf S$, we control the scalar-ridge expressivity cost.
\Cref{lem:fixed-support-oracle-gr} shows that the benchmark $\widetilde{\OPT}(\mathsf S)$ corresponds to a generalized ridge solution with penalty proportional to $H^{-1}_{\mathsf{SS}}$, whereas ordinary ridge uses a scalar penalty $\lambda I$.
The fixed-support approximation bound in \Cref{thm:fixed-support} then gives the bound $\min_{\lambda \ge 0} \mathrm{MSE}_{\mathrm{strat}}(\theta^{\mathrm{OR}}(\mathsf S,\lambda)) - \widetilde{\OPT}(\mathsf S) \le C_H(\mathsf S)\delta_H(\mathsf S)^2$.
\end{proof}

\paragraph{When are the coarse levers near-optimal?}
\Cref{thm:global-upper} gives an immediate $\epsilon$-optimality condition.
If there exists a support $\mathsf S$ such that $L_{\mathrm{pred}}(\mathsf S) - (\Gamma([d])-\Gamma(\mathsf S)) + C_H(\mathsf S)\delta_H(\mathsf S)^2 \le \epsilon$, then optimally tuned support-restricted ridge satisfies $\min_{\mathsf T\subseteq[d],\,\lambda\ge0} \mathrm{MSE}_{\mathrm{strat}}(\theta^{\mathrm{OR}}(\mathsf T,\lambda)) \le \widetilde{\OPT}+\epsilon$.
Thus, relative to the unrestricted strategic optimum, the coarse levers are near-optimal up to the irreducible gap $\widetilde{\OPT}-\OPT$ whenever there exists a support with little predictive loss, sufficient reduction in strategic burden, and sufficiently homogeneous retained manipulation costs.

The heterogeneity gap can vanish in natural cases: if the retained inverse-cost matrix is isotropic, then scalar ridge exactly attains the support-restricted zero-intercept oracle.

\begin{example}[Isotropic cost model] 
\label{ex:isotropic-cost}
Consider an isotropic cost model in which all features have the same manipulability. Formally, for any support $\mathsf{S} \subseteq [d]$, $H^{-1}_{\mathsf{SS}}=\gamma I$ for some $\gamma \ge 0$. 
Therefore, on any support contained in $[d]$, the retained manipulation costs are homogeneous, so scalar ridge is exactly expressive:
\[
\min_{\lambda\ge 0}
\mathrm{MSE}_{\mathrm{strat}}\bigl(\theta^{\OR}(\mathsf S,\lambda)\bigr)
=
\widetilde{\OPT}(\mathsf S),
\qquad
\mathsf S\subseteq [d].
\]
\end{example}

Thus a highly manipulable group need not be dropped if it is predictive and internally homogeneous; it can instead be retained and regularized. A formal statement for the isotropic cost model is given in Appendix~\ref{sec:zero-ridge-gap-isotropy}.

\begin{remark}
\Cref{ex:isotropic-cost} shows that the heterogeneity gap is exactly zero under an isotropic cost model on the retained support. 
However, isotropy is not necessary for exact optimality: a strengthened version of the upper bound (derived in Appendix~\ref{sec:necessary-sufficient-ridge-optimality}) yields a necessary-and-sufficient characterization, showing that the heterogeneity gap is zero if and only if the support-restricted oracle lies in an eigenspace of the inverse-cost matrix on the retained support. 
\end{remark}

The isotropic case in \Cref{ex:isotropic-cost} is useful but idealized.
A more realistic idealization is a two-level inverse-cost model in which one group of features is much easier to manipulate than the rest.

\begin{example}[Two-level cost model]
\label{ex:one-infty-inverse-cost}
Partition the features as $[d] = {\mathsf S}_m \cup {\mathsf S}_b$, where ${\mathsf S}_m$ denotes a highly manipulable group and ${\mathsf S}_b$ denotes a baseline homogeneous group.
For simplicity, suppose that in this example we work in whitened coordinates, so that $\Sigma = I$, and assume
\[
H^{-1}
=
\begin{pmatrix}
\kappa I_{{\mathsf S}_m} & 0\\
0 & I_{{\mathsf S}_b}
\end{pmatrix},
\qquad
\kappa>1.
\]
The limiting regime $\kappa\to\infty$ corresponds to a manipulable group with zero manipulation cost.
\end{example}

\Cref{ex:one-infty-inverse-cost} separates two effects.
On any support contained entirely in ${\mathsf S}_m$ or entirely in ${\mathsf S}_b$, the retained manipulation costs are fully homogeneous, so scalar ridge is exactly expressive:
\[
\min_{\lambda\ge 0}
\mathrm{MSE}_{\mathrm{strat}}\bigl(\theta^{\OR}(\mathsf S,\lambda)\bigr)
=
\widetilde{\OPT}(\mathsf S),
\qquad
\mathsf S\subseteq {\mathsf S}_m
\text{ or }
\mathsf S\subseteq {\mathsf S}_b.
\]
By contrast, any support that contains variables from both groups has two distinct inverse-cost levels.
In the whitened coordinates above:
\[
\delta_H(\mathsf S)
=
\begin{cases}
0, & \mathsf S\subseteq {\mathsf S}_m \text{ or } \mathsf S\subseteq {\mathsf S}_b,\\
(\kappa-1)/2, & \mathsf S\cap{\mathsf S}_m\neq\emptyset
\text{ and }
\mathsf S\cap{\mathsf S}_b\neq\emptyset.
\end{cases}
\]
Thus, as $\kappa$ grows, the gap between the inverse-cost levels on $\mathsf S_m$ and $\mathsf S_b$ grows.
On a mixed support, the oracle would shrink the highly manipulable features more aggressively than the baseline features, but scalar ridge has only one shrinkage parameter and cannot match both levels well.
Hence, the heterogeneity gap increases with $\kappa$. A formal statement for the two-level inverse-cost family is given in Appendix~\ref{app:two-level-inverse-cost}.

\begin{remark}
\Cref{ex:one-infty-inverse-cost} illustrates why feature selection can become the more important lever when heterogeneity increases.
If the highly manipulable group ${\mathsf S}_m$ is sufficiently predictive and internally homogeneous, the decision maker may retain it and regularize it aggressively.
If its predictive value is small, or if its signal can be substituted by the baseline group ${\mathsf S}_b$, the decision maker may instead exclude ${\mathsf S}_m$.
\end{remark}

The main message is that highly manipulable features should not always be dropped; rather, the optimal response depends jointly on their predictiveness, their substitutability by less manipulable features, and the homogeneity of the retained manipulation costs.

\section{Policy Implications for Feature Selection under Optimal Regularization}
\label{sec:implications}

Our characterization yields practical implications for policymakers seeking to make algorithmic decision-making systems more robust to strategic behavior using these coarser levers.

\begin{figure}[!t]
    \centering
    \includegraphics[width=0.99\linewidth]{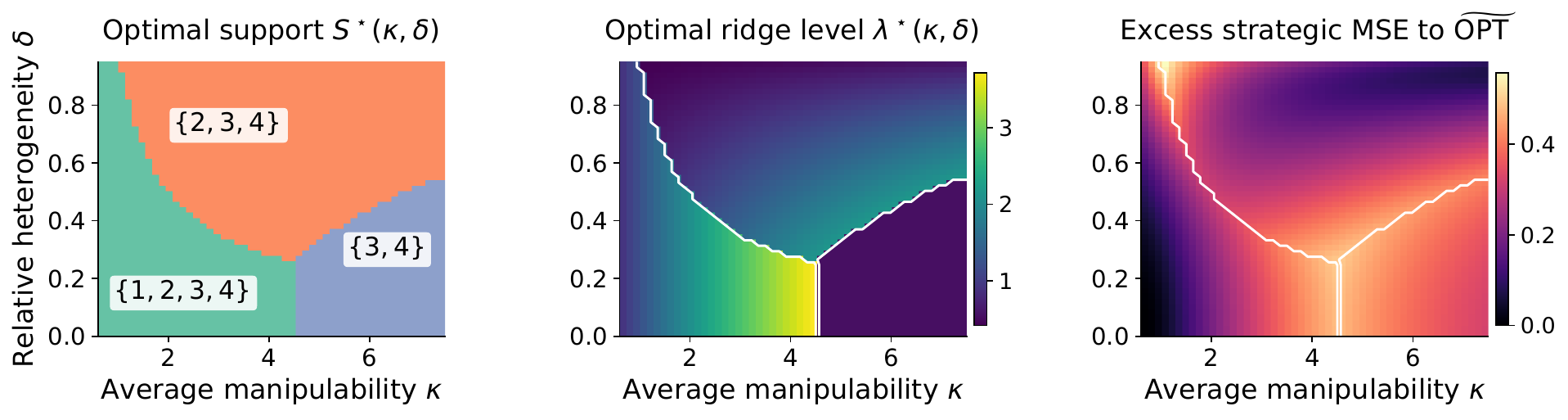}
    \caption{
    \textbf{Manipulable feature groups with homogeneous costs can be retained and regularized.} We compute the optimal support and regularization level exactly for a $d=4$ example with diagonal costs. 
    Features $\{1,2\}$ are a manipulable pair with average manipulability $\kappa$ and relative cost heterogeneity $\delta$ within the pair; features $3$ and $4$ have low manipulability. We vary $\kappa$ on the $x$-axis and $\delta$ on the $y$-axis.
    (Left) The exact optimal support $\mathsf{S}^{\star}(\kappa,\delta)$ after joint optimization over $\mathsf S$ and $\lambda$.
    (Middle) The corresponding optimal ridge level $\lambda^{\star}(\kappa,\delta)$, where white contours indicate support boundaries.
    (Right) Excess strategic MSE relative to {\small $\widetilde{\OPT}$}.
    When the manipulable pair is homogeneous ($\delta$ small), it is often retained even at high manipulability, with the optimal response being to increase regularization. The resulting support-restricted ridge estimator remains close to $\widetilde{\OPT}$ over a broad range.
    See Appendix~\ref{sec:fig:homogeneity-phase-main-details} for details.
    }
    \label{fig:homogeneity-phase-main}
\end{figure}

\paragraph{Feature selection and regularization must be tuned jointly.}
The support $\mathsf S$ appears in all three terms of \Cref{thm:global-upper}.
Changing $\mathsf S$ changes the predictive loss, the strategic burden, and the cost heterogeneity of the retained features.
Therefore, the best feature set cannot generally be determined by ranking features only by predictive value or 
manipulability.
In particular, regularization changes which subset is optimal, so feature selection and ridge tuning must be designed jointly (\Cref{fig:decomposition-tradeoff}).

\paragraph{Predictive manipulable groups can be retained when their manipulation costs are homogeneous.}
Contrary to common intuition, our characterization shows that high manipulability alone does not imply that a predictive feature group should be excluded (see \Cref{fig:homogeneity-phase-main}). 
When the subgroup is sufficiently predictive and its manipulation costs are sufficiently homogeneous, so that $\delta_H(\mathsf S)$ is small, scalar ridge can shrink the group. 
In this regime, the optimal response is often to retain the subgroup and increase regularization, rather than remove it outright. 
Feature exclusion becomes the preferred lever if the subgroup is insufficiently predictive, or its manipulation geometry is too heterogeneous for a single scalar parameter to regularize. 

This suggests a broader policy implication: when certain features cannot be excluded because of legal or institutional constraints, the appropriate response need not be all-or-nothing exclusion. 
Rather, the decision maker can use interventions such as targeted audits, verification, or tighter reporting requirements to raise or equalize manipulation costs on those features, making those features easier to control through regularization.

\paragraph{Less manipulable correlated proxies can replace manipulable features.} 
Our characterization shows that correlation can fundamentally change which support is optimal. 
When a less manipulable feature is sufficiently correlated with a manipulable one, the predictive loss $L_{\mathrm{pred}}(\sset)$ from dropping the latter can be small because much of its signal is recoverable from the proxy. 
If the resulting retained support also lowers the strategic burden or reduces cost heterogeneity, then the support that keeps the proxy and drops the manipulable feature can dominate the full support.
\Cref{fig:supp-correlated-proxy} illustrates this support switch.
The key insight is that a decision maker can preserve much of the predictive value while reducing strategic vulnerability by using correlated alternative features as proxies.

\begin{figure}[!t]
    \centering
    \includegraphics[width=0.99\linewidth]{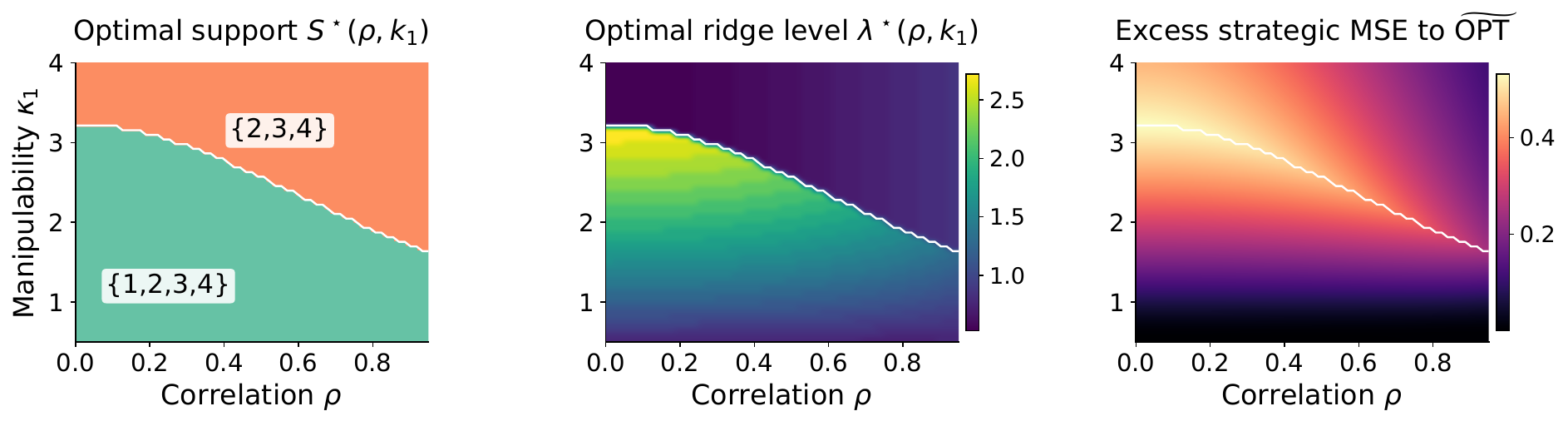}
    \caption{
    \textbf{Less manipulable correlated proxies can replace manipulable features.} We compute the optimal support and regularization level exactly for a $d=4$ example with correlated features. 
    Feature $1$ is predictive but manipulable, feature $2$ is a correlated less manipulable proxy for feature $1$ with zero direct coefficient in $\theta^*$, and features $3$ and $4$ are less manipulable baseline features.
    We vary the correlation $\rho$ between features $1$ and $2$ on the $x$-axis, and the manipulability $k_1$ of feature $1$ on the $y$-axis.
    (Left) The optimal support $\mathsf{S}^*(\rho,k_1)$ after joint optimization over $\mathsf S$ and $\lambda$.
    (Middle) The corresponding optimal ridge level $\lambda^*(\rho,k_1)$, where white contours indicate support boundaries.
    (Right) Excess strategic MSE relative to $\widetilde{\mathrm{OPT}}$.
    As correlation increases, the less manipulable feature becomes a more effective substitute for the manipulable feature.
    When $k_1$ is sufficiently large, the optimal support switches from $\{1,2,3,4\}$ to $\{2,3,4\}$, dropping the manipulable feature while retaining its proxy.
    The resulting support-restricted ridge estimator remains close to $\widetilde{\mathrm{OPT}}$ over a broad region of the parameter space.
    See Appendix~\ref{sec:supp-correlated-proxy} for details.
    }
    \label{fig:supp-correlated-proxy}
\end{figure}

Taken together, these results clarify the distinct role of the two policy levers.
Scalar ridge regularization is most effective on supports whose retained manipulation geometry is close to homogeneous, while feature selection is most effective when it removes strategically heterogeneous or easily manipulated features at limited predictive loss.
The best design therefore balances the two levers jointly rather than relying on either in isolation. 
In the next section, we turn to the corresponding algorithmic problem.
Appendix~\ref{app:ridge-tuning-properties} gives complementary local results showing that infinitesimal ridge regularization improves strategic MSE over support-restricted least squares, and that the locally optimal ridge level increases with manipulation intensity.

\section{Computing the Best Subset and Regularization Level}
\label{sec:algorithm}

The theory identifies the tradeoffs governing support choice and regularization.
We now turn to the corresponding algorithmic problem.
For a support budget $s\in\{1,\dots,d\}$, the exact design problem is
\[
(\mathsf S^\star,\lambda^\star)
\in
\arg\min_{\substack{\mathsf S\subseteq[d]\\ |\mathsf S|\le s}}
\ \min_{\lambda\ge 0}
\Phi(\mathsf S,\lambda),
\qquad
\Phi(\mathsf S,\lambda):=\MSEstrat(\theta^{\OR}(\mathsf S,\lambda),0).
\]
For a fixed support $\mathsf S$, the optimization over $\lambda$ is one-dimensional (and smooth); the computational challenge is the combinatorial search over supports.

\paragraph{A weighted relaxation.}
To avoid exhaustive search, we introduce a continuous weight vector $w\in[0,1]^d$ with support budget $\mathbf 1^\top w\le s$.
For a fixed $\lambda>0$, define the weighted ridge estimator:
\begin{equation}
\label{eq:weighted-ridge}
\theta_\lambda(w)
\in
\arg\min_{\theta\in\R^d}
\biggl\{\E[(Y-\theta^\top X)^2]+\lambda\sum_{i=1}^d \varphi(\theta_i,w_i)\biggr\},
\end{equation}
where $\varphi(t,u)=t^2/u$ for $u>0$, while $\varphi(0,0)=0$ and $\varphi(t,0)=+\infty$ for $t\neq 0$.
Thus $w_i=0$ forces $\theta_i=0$.
When all $w_i>0$, writing $W=\diag(w)$ gives $\theta_\lambda(w)=(\Sigma+\lambda W^{-1})^{-1}\Sigma\theta^*$ (see \Cref{lem:weighted-ridge-closed-form} for more details).
At fixed $\lambda$, we minimize the exact strategic objective $\MSEstrat(\theta_\lambda(w),0)$ over the weight polytope $\{w\in[0,1]^d:\mathbf 1^\top w\le s\}$.

\begin{proposition}[Binary weights recover exact support-restricted ridge]
\label{prop:binary-equivalence}
Let $w\in\{0,1\}^d$ and define $\mathsf S(w):=\{i\in[d]:w_i=1\}$. 
Then $\theta_\lambda(w)=\theta^{\OR}(\mathsf S(w),\lambda)$, and hence the weighted objective $\MSEstrat(\theta_\lambda(w),0)$ exactly matches the fixed-$\lambda$ support-restricted ridge objective $\Phi(\mathsf S(w),\lambda)$ on binary supports.
\end{proposition}

The continuous problem is therefore a relaxation of the exact fixed-$\lambda$ subset selection problem.
In particular, if 
\[
\Phi_{\lambda,s}^\star
:=
\min_{\substack{\mathsf S\subseteq[d]\\|\mathsf S|\le s}}
\Phi(\mathsf S,\lambda),
\qquad
L^{\mathrm{wt}}_{\lambda,s}
:=
\min_{w\in\mathcal W_s}
\MSEstrat(\theta_\lambda(w),0),
\]
then $L^{\mathrm{wt}}_{\lambda,s} \le \Phi_{\lambda,s}^\star$ (see \Cref{prop:weighted-relaxation-lower-bound} for a formal treatment and derivation).
We use the relaxed solution as a soft initialization, then return to the discrete support-restricted ridge by rounding, refining, and refitting.
We provide these details next.

\begin{algorithm}[!t]
\caption{Support-restricted ridge design by weighted relaxation and local refinement}
\label{alg:weighted-relaxation-local-refinement}
\begin{algorithmic}[1]
\Require Support budget $s$ and ridge grid $\Lambda\subset(0,\infty)$.
\For{each $\lambda\in\Lambda$}
    \State Solve or approximately solve the weighted relaxation:
    $
    \widehat w_\lambda
    \approx
    \arg\min_{w\in\mathcal W_s}
    \MSEstrat(\theta_\lambda(w),0)
    $.
    \State Initialize the support by rounding: $\mathsf S \leftarrow \operatorname{Top}_s(\widehat w_\lambda)$.
    \While{there exists $\mathsf T\in\mathcal N_s(\mathsf S)$ with
    $\Phi(\mathsf T,\lambda)<\Phi(\mathsf S,\lambda)$}
        \State Update to the best improving neighbor: $\mathsf S \leftarrow \arg\min_{\mathsf T\in\mathcal N_s(\mathsf S)\cup\{\mathsf S\}} \Phi(\mathsf T,\lambda)$.
    \EndWhile
    \State Set $\widehat{\mathsf S}_\lambda\leftarrow \mathsf S$ and refit exact support-restricted ridge: $\widehat\theta_\lambda := \theta^{\OR}(\widehat{\mathsf S}_\lambda,\lambda)$.
    \State Evaluate
    $\widehat R_\lambda:=\MSEstrat(\widehat\theta_\lambda,0)$.
\EndFor
\State Choose $\widehat\lambda\in\arg\min_{\lambda\in\Lambda}\widehat R_\lambda$.
\Ensure
$(\widehat\lambda,\widehat{\mathsf S}_{\widehat\lambda},
\widehat\theta_{\widehat\lambda})$.
\end{algorithmic}
\end{algorithm}

\paragraph{Our procedure.}
Towards outlining the procedure, let us define some useful notation.
Given a weight vector $w$, let $\operatorname{Top}_s(w)$ denote the indices of the $s$ largest entries of $w$.
For a current support $\mathsf S$, let $\mathcal N_s(\mathsf S)$ denote the add/drop/swap neighborhood of $\mathsf S$, consisting of supports obtained by adding one feature, dropping one feature, or swapping one retained feature for one excluded feature, while respecting the budget $s$.
Our procedure is then summarized in \Cref{alg:weighted-relaxation-local-refinement}.

The weighted relaxation plays a purely algorithmic role: it provides a soft initialization for support search at fixed $\lambda$.
The final estimator is always a standard support-restricted ridge fit of the form $\theta^{\OR}(\widehat{\mathsf S}_{\widehat\lambda},\widehat\lambda)$, so the procedure remains faithful to the estimator class studied in \Cref{sec:theory}.
Each accepted local refinement move strictly decreases the exact fixed-$\lambda$ objective $\Phi(\mathsf S,\lambda)$, so the refinement step terminates after finitely many moves and never worsens the rounded support (see \Cref{prop:local-refinement-termination}).
We provide implementation details and a convex reformulation for the diagonal-covariance case in Appendix~\ref{sec:app-algorithm} and synthetic benchmarks in Appendix~\ref{sec:app-algorithm-illustrations}.

\section{Case Study: Strategic Feature Selection in Healthcare Payments}
\label{sec:experiments}

We study healthcare payments in the Medicare Advantage (MA) program, as introduced earlier in the paper. 
In MA, the government determines payments to health insurers for the patients they cover using a least squares regression based on the diagnoses reported for their enrollees.
Here, diagnosis codes are the model features, which creates incentives to increase payment by assigning additional or more severe diagnosis codes to a patient, a practice known as upcoding. 

The payment model includes 115 variables corresponding to health conditions, termed as hierarchical condition categories (HCCs). 
We base our study on recent policy work that recommends excluding the \emph{top ten diagnosis groups} that contributed the most to higher coding intensity in Medicare Advantage~\citep{kronick2025are}. 
We show that this proposal is suboptimal in strategic $\MSE$, and that our method can preserve both predictive and strategic performance while retaining some of these diagnoses. 
Code to reproduce our experiments is available at \url{https://github.com/jivatneet/strategic-feature-selection}.

\subsection{Experimental setup}

To simulate realistic upcoding, we build on prior work to first generate diagnosis codes in the absence of coding incentives, and then modify them to match the shifts in diagnosis prevalence associated with more intense coding.
We use the \texttt{upcoding} package\footnote{\url{https://github.com/StanfordHPDS/upcoding}}, which simulates realistic co-occurring HCCs for the Medicare population using self-reported health conditions from a large national survey that are unaffected by coding incentives~\citep{enache2026time}. 
Starting from this baseline data, we upcode selected HCCs among the top-ten diagnosis groups identified by \citet{kronick2025are} to match their reported prevalence shifts.
The baseline HCC data provide a realistic empirical covariance structure for diagnoses, while the induced prevalence shifts provide policy-grounded information about coding vulnerability that we use to construct the manipulation-ease matrix $H^{-1}$.
The resulting baseline dataset contains $n=5000$ beneficiaries and $115$ V28 HCC indicators. 
Of these, $30$ HCCs have nonzero baseline prevalence.\footnote{This is partly due to export restrictions in the underlying All of Us data, which exclude co-occurring HCC sets reported by only a small number of survey respondents.} 
Support sizes in the results that follow are therefore reported over the 30 HCCs with nonzero baseline prevalence.

We use the published CMS-HCC V28 coefficients\footnote{\url{https://www.cms.gov/files/document/2024-announcement-pdf.pdf}} as our signal $\theta^*$ and generate outcomes from the baseline HCC covariates.
All methods are fit using only baseline, pre-manipulation data, following the MA setting.
Evaluation follows the strategic response model studied in this paper: for a fitted coefficient vector $\widehat\theta$, features at deployment
respond as $X \mapsto X^{(0)} + \alpha H^{-1}\widehat\theta$, where $\alpha\ge0$ controls manipulation intensity.
We report the resulting post-manipulation $\mathrm{MSE}$.

We compare our support-restricted ridge method to the following baselines:
full ridge, which tunes ordinary ridge on all HCCs with nonzero baseline prevalence;
top-ten exclusion, which removes HCCs with nonzero baseline prevalence that belong to the top-ten
diagnosis groups identified by \citet{kronick2025are} and then tunes ridge on the remaining features;
prediction-only, which selects features using predictive value alone and
then tunes ridge;
cost-only, which selects features using manipulation-ease scores alone
and then tunes ridge; 
and subset-only, which performs support-restricted
least squares without ridge shrinkage. 
We also include the full-support zero-intercept oracle $\widetilde{\OPT}$ (introduced in \Cref{sec:lower-bound}) as a benchmark.
For ease of visualization and comparison across methods, we report post-manipulation $\mathrm{MSE}$ normalized by the full-support ridge value at the same manipulation intensity, so values below one indicate an improvement over full-support ridge.

We refer the reader to Appendix~\ref{sec:app-expts} for further details on the data, semi-synthetic outcome construction, cost matrix, baselines, and evaluation metrics.

\begin{figure}[!t]
    \centering
    \includegraphics[width=0.99\linewidth]{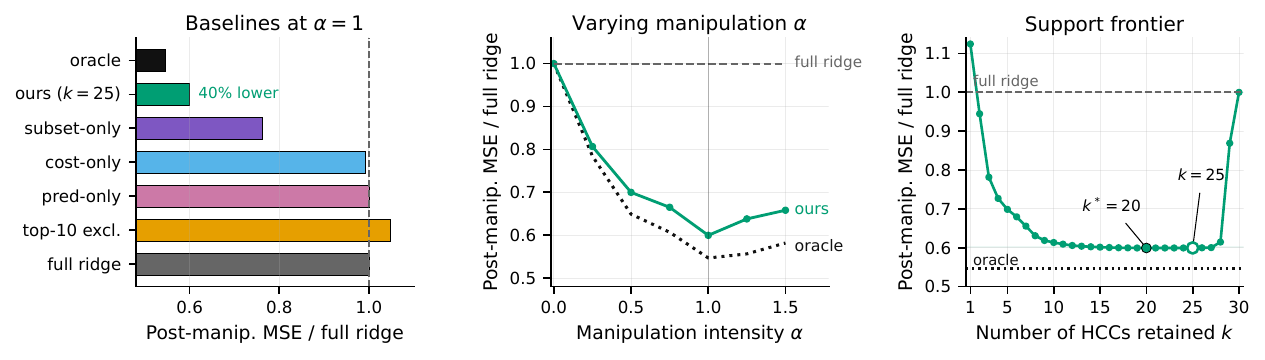}
    \caption{
    \textbf{Feature selection under optimal regularization retains many features while keeping post-manipulation error low.}
    (Left) Post-manipulation $\MSE$ at manipulation intensity $\alpha=1$ for all methods. 
    (Middle) Post-manipulation $\MSE$ under varying $\alpha$ for the oracle $\widetilde{\OPT}$ and the support-restricted ridge design with $k=25$.
    (Right) Cardinality frontier of support-restricted ridge at $\alpha=1$, showing post-manipulation $\MSE$ across support sizes $k$.
    Post-manipulation $\MSE$ is normalized by the full ridge value.
    }
    \label{fig:top10-anyfix-main}
\end{figure}

\subsection{Results}

Our empirical results match the qualitative findings of our theoretical analysis: low post-manipulation error need not require blanket exclusion, and the best policy designs come from combining feature selection with regularization.

\paragraph{Proposed method substantially reduces post-manipulation $\mathrm{MSE}$ relative to baselines and remains competitive with the oracle.} 
We compare the proposed support-restricted ridge estimator to the baselines and the oracle $\widetilde{\OPT}$. 
The left panel of \Cref{fig:top10-anyfix-main} shows that the support-restricted ridge design with $k=25$ retained features reduces post-manipulation $\mathrm{MSE}$ by about 40\% relative to full ridge and outperforms the top-ten exclusion policy, prediction-only selection, cost-only selection, and subset-only selection without ridge shrinkage.
This illustrates the importance of tuning feature selection and regularization jointly.
Further, the middle panel of \Cref{fig:top10-anyfix-main} shows that the support-restricted ridge estimator remains competitive with the oracle across a range of manipulation intensities.

\paragraph{Low strategic error does not require aggressive exclusion.}
The right panel of \Cref{fig:top10-anyfix-main} showcases the support frontier of the proposed estimator, illustrating that low post-manipulation prediction error can be achieved while retaining a larger feature set. 
Specifically, while the unconstrained optimum is attained at $k^*=20$, expanding the model to include $k=25$ HCCs yields nearly identical performance.
For this reason, we use $k = 25$ as a representative operating point in the main results.

\begin{figure}[!t]
    \centering
    \begin{minipage}[t]{0.45\linewidth}
        \centering
        \includegraphics[width=\linewidth]{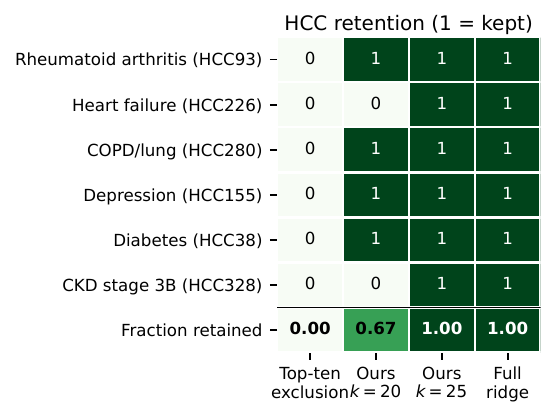}
    \end{minipage}
    \quad
    \begin{minipage}[t]{0.45\linewidth}
        \centering
        \includegraphics[width=\linewidth]{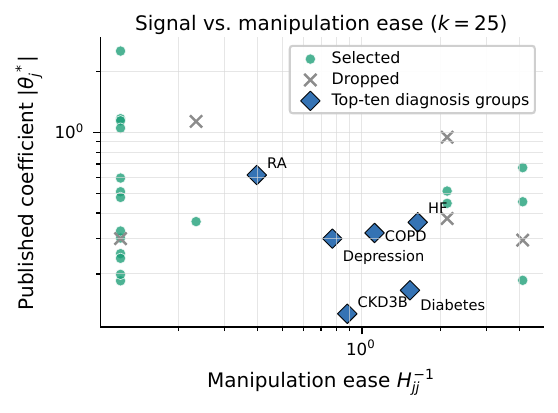}
    \end{minipage}
    \caption{
    \textbf{The selected model reflects the joint predictability-manipulability structure, rather than dropping all intensely coded features.}
    (Left) HCC groups with nonzero prevalence in the data among the top-ten diagnosis groups identified by~\citet{kronick2025are}.
    The top-ten exclusion policy drops all of these groups, whereas the unconstrained optimum at $k=20$ retains some of them and the broader $k=25$ design retains all six such groups. 
    (Right) Each HCC with nonzero prevalence in the data is plotted by its published CMS-HCC coefficient magnitude and manipulation-ease score $H^{-1}_{jj}$; marker type indicates whether the broader $k=25$ support-restricted ridge design retains that HCC.
    }
    \label{fig:anyfix-selection-mechanism}
\end{figure}

As shown in the left panel of \Cref{fig:anyfix-selection-mechanism}, the support-restricted ridge estimator with $k=25$ retains all the HCCs (with nonzero baseline prevalence) that belong to the top-ten diagnosis groups while outperforming both the full ridge estimator and the top-ten exclusion policy.
The support frontier confirms that the best design is often an interior solution rather than either the full feature set or aggressive exclusion.
Importantly, these gains are not driven by feature selection alone: the strongest performance comes from combining support restriction with regularization.

\paragraph{Selected support is not determined by coding vulnerability alone.}
The selected support is not obtained by simply ranking diagnoses by coding vulnerability: the $k=25$ support-restricted ridge design retains several HCC groups identified by \citet{kronick2025are}, while excluding other HCCs (see the right panel of \Cref{fig:anyfix-selection-mechanism}).
The plot further illustrates there is no straightforward pattern based on a feature's direct predictiveness and its manipulation ease that determines inclusion or exclusion.
This matches our theory, which predicts that feature selection should balance the predictive loss from excluding features against the strategic burden and cost geometry of the retained set. 
\Cref{fig:anyfix-selection-mechanism} also contrasts our design with the top-ten exclusion policy, which removes all such groups by construction, whereas support-restricted ridge can retain some of them and control strategic exposure through regularization.

\section{Inherent Robustness under Cost Uncertainty}
\label{sec:cost-uncertainty}

Our main results assume that the cost matrix $H$ is known.
In practice, however, manipulation costs are difficult to estimate precisely. 
We therefore briefly examine the case where cost matrix $H$ is \emph{unknown}.

The goal in this section is to study the \emph{inherent robustness} of various methods, rather than designing a new robust optimization procedure. 
We show that regularization and feature selection provide additional robustness guarantees relative to other strategies. 
For simplicity, we reparameterize and consider the inverse cost matrix $K = H^{-1}$. 

To demonstrate this, we first revisit the case of known costs and consider the benchmark predictor that attains the strategic optimum $\OPT$.
The decision maker can attain $\OPT$ if they are able to \textit{directly} tune the intercept based on the costs according to $b^{\OPT} = -\theta^{*\top}K\theta^*$, achieving the irreducible error $\sigma^2$ (see \Cref{lem:strategic-opt}).
Thus, with exact knowledge of costs and the flexibility to choose an intercept, the unrestricted intercept correction (along with the true signal model $\theta^*$) cannot be improved upon.

However, the situation changes in the practical scenario when the decision maker only has partial information about the inverse-cost matrix $K$.
To make this precise, we introduce some additional notation.
For a set $\mathcal K$ of plausible inverse-cost matrices, define the worst-case excess strategic MSE of a fixed rule $(\theta,b)$ by
\[
\sup_{K\in\mathcal K}
\left\{
\MSE_K(\theta,b)- \OPT
\right\},
\quad
\text{where}
\quad
\MSE_K(\theta,b)
=
(\theta-\theta^*)^\top\Sigma(\theta-\theta^*)
+
(b+\theta^\top K\theta)^2
+
\sigma^2 .
\]
This expression makes the robustness issue transparent. 
An intercept correction can remove the strategic shift exactly when $K$ is known: the strategic optimum uses the (coefficients, intercept) tuple $(\theta^*,-\theta^{*\top}K\theta^*)$. 
But under cost uncertainty, an intercept chosen using a nominal estimate $\widehat K$ leaves the residual shift $\theta^\top(K-\widehat K)\theta$. 
The following simple identity formalizes the worst-case error of a general fixed model $(\theta, b)$, and then specializes it to intercept correction.
\begin{proposition}[Worst-case error in general and for intercept correction]
\label{prop:fixed-rule-worst-case}
Let $\mathcal K$ be a compact set of possible inverse-cost matrices and fix $(\theta,b)$. 
Define $\underline q(\theta):=\inf_{K\in\mathcal K}\theta^\top K\theta$ and $\overline q(\theta):=\sup_{K\in\mathcal K}\theta^\top K\theta$.
Then
\begin{equation}
    \label{eq:worstcaseerror}
\sup_{K\in\mathcal K}
\left\{
\MSE_K(\theta,b)- \OPT
\right\}
=
(\theta-\theta^*)^\top\Sigma(\theta-\theta^*)
+
\max\{(b+\underline q(\theta))^2,\,(b+\overline q(\theta))^2\}.
\end{equation}
In particular, for intercept correction, where $\theta = \theta^*$ and $b=-{\theta^*}^\top\widehat K\theta^*$, we have
\[
\sup_{K\in\mathcal K}
\left\{
\MSE_K(\theta,b)- \OPT
\right\}
=
\sup_{K\in\mathcal K}
\bigl({\theta^*}^\top(K-\widehat K)\theta^*\bigr)^2 .
\]
\end{proposition}

Calculation of the worst-case error in general \eqref{eq:worstcaseerror} shows that a model is inherently more robust when its coefficient vector has small exposure to the uncertain manipulation directions. 
Regularization reduces this exposure by shrinking coefficients, while feature selection can remove highly uncertain or highly manipulable coordinates entirely.
We now give two examples showing how regularization and feature selection can improve inherent robustness compared to intercept correction.

\begin{example}[Regularization alone]
\label{ex:robustness-regularization-alone}
Let $d=1$, $\Sigma=1$, and $\theta^*=1$. 
Suppose the true inverse cost lies in $\mathcal K=[1/2,5/2]$, and the nominal estimate is the midpoint $\widehat K=3/2$. 
The plug-in intercept correction deploys $(\theta,b)=(1,-3/2)$, whose worst-case excess strategic MSE is
\[
\sup_{K\in[1/2,5/2]}(K-3/2)^2
=
1.
\]
By contrast, ordinary ridge with $\lambda=1$ gives $\theta^{\OR}(\lambda)=1/2$ and $b=0$. 
Its worst-case excess strategic MSE is
\[
\left(\frac12-1\right)^2
+
\sup_{K\in[1/2,5/2]}
\left(K\left(\frac12\right)^2\right)^2
=
\frac14+\left(\frac{5}{8}\right)^2
=
\frac{41}{64}
<1.
\]
Thus, shrinkage alone can be inherently more robust than a plug-in intercept correction.
\end{example}

Note that in \Cref{ex:robustness-regularization-alone} we examine a fixed intercept $b = -\widehat K = -3/2$ rather than proposing a robust intercept.
Nevertheless, \Cref{prop:fixed-rule-worst-case} implies that if one were allowed to choose the best fixed intercept for a fixed coefficient model $\theta$, the minimizer would place $-b$ at the midpoint of the interval \([\underline q(\theta),\overline q(\theta)]\), yielding residual squared shift $(\overline q(\theta)-\underline q(\theta))^2/4$.
So in this example, we are in fact already comparing against the best fixed intercept model with unknown $K$.

\begin{example}[Feature selection and regularization together]
\label{ex:robustness-feature-selection-regularization}
Now, consider a two-feature example in which the first feature is directly predictive but has uncertain manipulation cost, while the second feature is a less manipulable correlated proxy for the first feature:
\[
    \theta^*=(1,0),
    \qquad
    \Sigma=
    \begin{pmatrix}
    1 & 49/50\\
    49/50 & 1
    \end{pmatrix},
\quad
\text{and}
\quad
    K(k)=
    \begin{pmatrix}
    k & 0\\
    0 & 1/2
    \end{pmatrix},
    \qquad
    k\in[1/2,5/2].
\]
Again, take the nominal estimate to be the midpoint $\widehat K=K(3/2)$. 
The plug-in intercept correction with slope $\theta^*$ has worst-case excess strategic MSE equal to $1$, as in \Cref{ex:robustness-regularization-alone}.

Full-support ridge reduces this error by shrinking the manipulable coordinate. 
For example, full-support ridge with $\lambda=1$ has worst-case excess strategic MSE approximately
\[
    \sup_{k\in[1/2,5/2]}
    \MSE_{K(k)}(\theta^{\OR}([2],1),0)- \OPT
    \approx 0.240 .
\]
However, support-restricted ridge can do even better by dropping the uncertain coordinate and using the less manipulable proxy. 
On the support $\mathsf S=\{2\}$, ridge with $\lambda=1/4$ gives
\[
    \theta^{\OR}(\{2\},1/4)
    =
    \left(0,\frac{98}{125}\right).
\]
Since this rule places zero weight on the uncertain coordinate, its strategic shift is independent of $k$. 
Its worst-case excess strategic MSE is
\[
\begin{aligned}
1
-
2\left(\frac{49}{50}\right)\left(\frac{98}{125}\right)
+
\left(\frac{98}{125}\right)^2
+
\left[
\frac12
\left(\frac{98}{125}\right)^2
\right]^2
\approx 0.172
<
0.240 .
\end{aligned}
\]
For comparison, support-restricted least squares on $\mathsf S=\{2\}$ gives $\theta^{\mathrm{LS}}(\{2\}) = (0,49/50)$, whose worst-case excess strategic MSE is approximately $0.270$. 
Thus, the improvement comes from combining feature selection with regularization: feature selection removes the uncertain manipulation direction, and regularization controls the residual exposure of the retained proxy.
\end{example}

\Cref{ex:robustness-regularization-alone,ex:robustness-feature-selection-regularization} illustrate that, under cost uncertainty, intercept correction can be sensitive to misspecification, whereas regularization and feature selection can be inherently more robust by reducing exposure to uncertain manipulation directions.
Designing explicitly robust versions of these methods, such as choosing $b$ or $(\mathsf S,\lambda)$ to minimize worst-case strategic MSE over $ \mathcal K$, is a natural direction for future work.

\section{Discussion}
\label{sec:discussion}

We develop a framework for studying strategic classification through the coarse policy levers of feature selection and regularization. 
Our main message is that features should be selected jointly based on their relative manipulability and predictability, rather than based on the manipulability of individual features alone. 
Through theoretical analysis and a real-world case study, we show that the interaction between feature selection and regularization yields nuanced implications for policy design that broaden the prevailing narrative. 
Furthermore, the coarse levers offer greater inherent robustness under cost uncertainty compared to alternative strategies.
More broadly, given the ubiquity of prediction systems in practice, we envision a growing need for principled and practical frameworks that bridge theoretical foundations with the policy and institutional contexts in which these systems are deployed.

Several interesting questions remain. We study a one-shot interaction in this work, and it would be valuable to analyze long-term dynamics in this setting.
An interesting statistical perspective to explore further is to view strategic tuning as tuning under an endogenous deployment shift. 
Related work studies optimal regularization under exogenous covariate and regression shifts (see, e.g., \cite{patil2024optimal} and references therein), whereas here the shift is induced by the deployed coefficients themselves and is coupled to support restriction.
Another extension is to study more expressive coarse levers, such as grouped ridge or per-feature regularization, as intermediate designs between ordinary ridge and the strategically optimal solution. 
Our framework can also be extended to incorporate settings where certain variables \emph{must} be retained for policy or institutional reasons, by modifying the support-selection objective.

At the same time, our model makes some simplifying assumptions that are common in the strategic learning literature, such as known and quadratic manipulation costs, which are homogeneous across organizations. 
In this paper, we take a first step toward understanding misspecification by studying the inherent robustness of a fixed deployed model when the true inverse-cost matrix lies in a plausible uncertainty set. 
Our examples illustrate that intercept corrections can be sensitive to misspecification, while regularization and feature selection can reduce exposure to uncertain manipulation directions. 
However, these examples barely scratch the surface of the broader misspecification question.
We do not design robust versions of any method in this work; developing worst-case or distributionally robust choices of the intercept, support, and regularization level, examining richer cost structures, and studying the impact of cost heterogeneity across organizations are important directions for future work.

\section*{Acknowledgments}

We would like to thank Oana M. Enach, Keegan Harris, Richard Kronick, Lydia T. Liu, Juan C. Perdomo, Sherri Rose, Manolis Zampetakis, Anna Zink, and Tijana Zrnic for insightful conversations on our work. Ida Sim's guest lecture on healthcare payment models introduced the authors to this problem, and we are grateful for several helpful discussions. Computing support is partially provided by the Texas Advanced Computing Center (TACC). This work was supported in part by a Stanford AI Lab postdoctoral fellowship and by the European Union (ERC-2022-SYG-OCEAN-101071601). Views and opinions expressed are however those of the author(s) only and do not necessarily reflect those of the European Union or the European Research Council Executive Agency. Neither the European Union nor the granting authority can~be~held~responsible~for~them.

\bibliography{ref}
\bibliographystyle{plainnat}

\clearpage
\appendix

\clearpage

\section{Theoretical Details and Proofs for \Cref{sec:model}}
\label{sec:model-proofs}

\subsection{Strategic optimum characterization}

\begin{lemma}[Strategic optimum characterization]
\label{lem:strategic-opt}
The strategic optimum \eqref{eq:opt} is attained at $(\theta^{\mathrm{OPT}},b^{\mathrm{OPT}})=(\theta^*,-\theta^{*\top}H^{-1}\theta^*)$ and its value is $\mathrm{OPT}_{\theta^*}(\gP_{X,Y}, H) = \sigma^2$.
\end{lemma}

\begin{proof}
By \Cref{lem:general-mse}, we can decompose the strategic MSE \eqref{eq:mse-strat} into:
\[
\mathrm{MSE}_{\mathrm{strat}}(\theta,b)
=
(\theta-\theta^*)^\top\Sigma(\theta-\theta^*)
+
\bigl(b+\theta^\top H^{-1}\theta\bigr)^2
+
\sigma^2.
\]
For any fixed $\theta$, the optimal intercept is $b(\theta)=-\theta^\top H^{-1}\theta$, which eliminates the strategic shift term.
Therefore
\[
\min_{b\in\RR}
\mathrm{MSE}_{\mathrm{strat}}(\theta,b)
=
(\theta-\theta^*)^\top\Sigma(\theta-\theta^*)+\sigma^2.
\]
Since $\Sigma \succ 0$, this expression is uniquely minimized at $\theta=\theta^*$.
Hence the strategic optimum \eqref{eq:opt} is attained at 
\[
\theta^{\mathrm{OPT}}=\theta^*,
\qquad
b^{\mathrm{OPT}}=-\theta^{*\top}H^{-1}\theta^*.
\]
The optimal value is $\mathrm{OPT}_{\theta^*}(\gP_{X,Y},H) = \sigma^2$.
\end{proof}

\subsection{Centering and projection identity}
\label{sec:centering-projection-identities}

We work under the centering convention $\EE[X]=0$ and $\EE[Y]=0$.
This convention is without loss of generality for the linear models studied in the paper: after replacing $X$ and $Y$ by $X-\EE[X]$ and $Y-\EE[Y]$, any shift in means can be absorbed into the intercept of the deployed linear model.
Moreover, since the organization's action is additive in the features, replacing $X$ by $X-\EE[X]$ does not change the action $a$ or the quadratic manipulation cost.

We next record a projection identity implied by the projection parameter $\theta^*$ as defined in \eqref{eq:projection-parameter}.
This is a standard result (see, e.g., \cite{gyorfi2002distribution}), whose derivation we collect below for completeness.

\begin{lemma}[Projection decomposition]
\label{lem:projection-decomposition}
Suppose $\EE[X]=0$, $\EE[Y]=0$, and $\Sigma:=\EE[XX^\top]\succ 0$.
Let $\theta^*$ be the coefficients of the best linear projection of $Y$ onto $X$ as in \eqref{eq:projection-parameter}.
Defining $\eps:=Y-\theta^{*\top}X$, we have $\EE[\eps]=0$ and $\EE[X\eps]=0$.
Moreover, for every $\theta\in\RR^d$, we can decompose
$
\EE[(Y-\theta^\top X)^2]
=
(\theta-\theta^*)^\top\Sigma(\theta-\theta^*)+\sigma^2
$,
where $\sigma^2 := \EE[\eps^2]$.
\end{lemma}
\begin{proof}
Since $\EE[X]=0$ and $\EE[Y]=0$, $\EE[\eps]=\EE[Y]-\theta^{*\top}\EE[X]=0$.
Moreover, since $\Sigma\theta^*=\EE[XY]$ because $\theta^*=\Sigma^{-1}\EE[XY]$ (or from the first-order condition for the objective in \eqref{eq:projection-parameter}), we have
\[
\EE[X\eps]
=
\EE[XY]-\EE[XX^\top]\theta^*
=
\EE[XY]-\Sigma\theta^*
=
0.
\]
Now, for any $\theta \in \RR^d$, $Y-\theta^\top X = (\theta^*-\theta)^\top X+\eps$.
Therefore
\begin{align*}
\EE[(Y-\theta^\top X)^2]
&=
\EE[((\theta^*-\theta)^\top X)^2]
+
2\EE[((\theta^*-\theta)^\top X)\eps]
+
\EE[\eps^2] \\
&=
(\theta-\theta^*)^\top\Sigma(\theta-\theta^*)+\sigma^2,
\end{align*}
where the cross term vanishes because $\EE[X\eps]=0$.
\end{proof}

\subsection{Best response and strategic MSE}

Below we collect basic identities used in the proofs of results in \Cref{sec:model} and throughout the paper.
All results in this section are standard and are only reproduced for completeness.

\begin{lemma}[Organization's best response]
\label{lem:org-br}
For any linear predictor $f_{\theta,b}(x)=\theta^\top x+b$, the organization's best response in \eqref{eq:agent_br} is unique and equals $a^*(x;\theta)=H^{-1}\theta$.
In particular, under the unconstrained quadratic-cost model, the optimal action does not depend on the realized feature vector $x$.
\end{lemma}
\begin{proof}
For fixed $x$, $\theta$, and $b$, the organization's objective is:
\[
\theta^\top(x+a)+b-\tfrac{1}{2}a^\top H a.
\]
The terms $\theta^\top x+b$ do not depend on $a$, so the optimization problem is equivalent to:
\[
\max_{a\in\RR^d}
\big\{
\theta^\top a-\tfrac{1}{2}a^\top H a
\big\}.
\]
Since $H \succ 0$, this objective is strictly concave in $a$.
Its first-order condition is $\theta-Ha=0$.
Thus, the unique maximizer is $a^*(x;\theta)=H^{-1}\theta$.
Observe that this optimal manipulation depends on the announced coefficient vector $\theta$ but not on the realized feature vector $x$.
\end{proof}

Under this best response per \Cref{lem:org-br}, the deployed prediction is
\[
\theta^\top(X+a^*(X;\theta))+b
=
\theta^\top X+\theta^\top H^{-1}\theta+b.
\]
Thus strategic manipulation induces an endogenous deterministic shift $\theta^\top H^{-1}\theta$ in the prediction.

\begin{lemma}[Strategic MSE identity]
\label{lem:general-mse}
For any $\theta\in\mathbb R^d$ and $b \in \sR$,
\begin{equation}
\label{eq:strategic-mse-decomp}
\mathrm{MSE}_{\mathrm{strat}}(\theta, b)
= (\theta-\theta^*)^\top \Sigma (\theta-\theta^*)
\;+\; (b + \theta^\top H^{-1}\theta)^2
\;+\; \sigma^2.
\end{equation}
\end{lemma}
\begin{proof}
By \Cref{lem:org-br}, $a^*(X;\theta)=H^{-1}\theta$.
Therefore
\[
\theta^\top(X+a^*(X;\theta))+b-Y
=
\theta^\top X+\theta^\top H^{-1}\theta+b-Y.
\]
Using $Y=\theta^{*\top}X+\eps$, this becomes
\[
(\theta-\theta^*)^\top X+\theta^\top H^{-1}\theta+b-\eps.
\]
Let $c(\theta,b):=\theta^\top H^{-1}\theta+b$.
Then
\[
\theta^\top(X+a^*(X;\theta))+b-Y
=
(\theta-\theta^*)^\top X+c(\theta,b)-\eps.
\]
Squaring and taking expectations lets us decompose the strategic MSE \eqref{eq:mse-strat} as:
\begin{align*}
\mathrm{MSE}_{\mathrm{strat}}(\theta,b)
&=
\EE
\big[
(
(\theta-\theta^*)^\top X+c(\theta,b)-\eps
)^2
\big] \\
&=
\EE[((\theta-\theta^*)^\top X)^2]
+
c(\theta,b)^2
+
\EE[\eps^2] \\
&\qquad
-2\EE[((\theta-\theta^*)^\top X)\eps]
+
2c(\theta,b)\EE[(\theta-\theta^*)^\top X]
-
2c(\theta,b)\EE[\eps].
\end{align*}
The cross terms vanish because $\EE[X\eps]=0$, $\EE[X]=0$, and $\EE[\eps]=0$.
Hence
\[
\mathrm{MSE}_{\mathrm{strat}}(\theta,b)
=
(\theta-\theta^*)^\top\Sigma(\theta-\theta^*)
+
c(\theta,b)^2
+
\sigma^2,
\]
as claimed.
\end{proof}

We can interpret the three terms in the strategic MSE decomposition \eqref{eq:strategic-mse-decomp} as follows:
The first term is the loss from using coefficients $\theta$ instead of the best linear projection coefficients $\theta^*$.
The second term is the deterministic shift induced jointly by the intercept and the organization's equilibrium manipulation.
The final term $\sigma^2$ is the irreducible residual variance.

\subsection{Closed forms for the policy levers}
\label{sec:estimator-expressions}

We record here closed-form expressions for the estimators (\eqref{eq:subset-ls-compact}, \eqref{eq:ordinary-ridge-compact}, and \eqref{eq:subset-ridge-compact}) introduced in \Cref{sec:model}.
The results here are again standard and are only reported for completeness, in our notation.

Recall from \Cref{sec:model} that for a subset $\mathsf S\subseteq[d]$, we let $\mathsf R=[d]\setminus\mathsf S$.
We write $\Sigma_{\mathsf S y}:=\EE[X_{\mathsf S}Y]$.
Note that since $Y=\theta^{*\top}X+\eps$ and $\EE[X\eps]=0$, we have
\begin{equation}
\label{eq:Sigma-S-y-identity}
\Sigma_{\mathsf S y}
=
\Sigma_{\mathsf{SS}}\theta^*_{\mathsf S}
+
\Sigma_{\mathsf{SR}}\theta^*_{\mathsf R}.
\end{equation}

\begin{lemma}[Closed forms for support-restricted least squares and support-restricted ridge]
\label{lem:policy-lever-closed-forms}
Under the centering convention $\EE[X]=0$, $\EE[Y]=0$, all three estimators \eqref{eq:subset-ls-compact}, \eqref{eq:ordinary-ridge-compact}, and \eqref{eq:subset-ridge-compact}, introduced in \Cref{sec:model} have intercept $0$.
Moreover:
\begin{enumerate}
    \item For any $\mathsf S\subseteq[d]$, the support-restricted least squares estimator is:
    \[
    \theta^{\mathrm{LS}}(\mathsf S)
    =
    (\theta_{\mathsf S}^{\mathrm{LS}}(\mathsf S),0_{\mathsf R}),
    \]
    where
    \[
    \theta_{\mathsf S}^{\mathrm{LS}}(\mathsf S)
    =
    \Sigma_{\mathsf{SS}}^{-1}\Sigma_{\mathsf S y}
    =
    \theta^*_{\mathsf S}
    +
    \Sigma_{\mathsf{SS}}^{-1}\Sigma_{\mathsf{SR}}\theta^*_{\mathsf R}.
    \]
    \item For any $\lambda \ge 0$, the ordinary ridge estimator is:
    \[
    \theta^{\mathrm{OR}}(\lambda)
    =
    (\Sigma+\lambda I_d)^{-1}\Sigma\theta^*.
    \]
    \item
    For any $\mathsf S\subseteq[d]$ and $\lambda \ge 0$, the support-restricted ridge estimator is:
    \[
    \theta^{\mathrm{OR}}(\mathsf S,\lambda)
    =
    (\theta_{\mathsf S}^{\mathrm{OR}}(\mathsf S,\lambda),0_{\mathsf R}),
    \]
    where
    \[
    \theta_{\mathsf S}^{\mathrm{OR}}(\mathsf S,\lambda)
    =
    (\Sigma_{\mathsf{SS}}+\lambda I_{|\mathsf S|})^{-1}\Sigma_{\mathsf S y}.
    \]
    Equivalently,
    \[
    \theta_{\mathsf S}^{\mathrm{OR}}(\mathsf S,\lambda)
    =
    (\Sigma_{\mathsf{SS}}+\lambda I_{|\mathsf S|})^{-1}
    \left(
    \Sigma_{\mathsf{SS}}\theta^*_{\mathsf S}
    +
    \Sigma_{\mathsf{SR}}\theta^*_{\mathsf R}
    \right).
    \]
\end{enumerate}
\end{lemma}
\begin{proof}
We first verify that the optimal intercept is zero in all three programs.
For any fixed $\theta$, the unregularized squared-error part satisfies
\[
\EE[(Y-\theta^\top X-b)^2]
=
\EE[(Y-\theta^\top X)^2]+b^2.
\]
Thus, the optimal unique intercept is zero.
The ridge penalty does not involve $b$, so the same conclusion holds for the ridge programs.
We next write out the coefficient expressions.

\paragraph{Support-restricted least squares.}
For $\theta\in\Theta(\mathsf S)$, write $\theta=(\eta,0_{\mathsf R})$ with $\eta\in\RR^{|\mathsf S|}$.
The support-restricted least squares solution reduces to
\[
\min_{\eta\in\RR^{|\mathsf S|}}
\EE[(Y-\eta^\top X_{\mathsf S})^2].
\]
The objective is strictly convex because $\Sigma_{\mathsf{SS}} \succ 0$.
Hence, the solution is unique given by $\theta_{\mathsf S}^{\mathrm{LS}}(\mathsf S) = \Sigma_{\mathsf{SS}}^{-1}\Sigma_{\mathsf S y}$, and using \eqref{eq:Sigma-S-y-identity} gives
\[
\theta_{\mathsf S}^{\mathrm{LS}}(\mathsf S)
=
\Sigma_{\mathsf{SS}}^{-1}
\left(
\Sigma_{\mathsf{SS}}\theta^*_{\mathsf S}
+
\Sigma_{\mathsf{SR}}\theta^*_{\mathsf R}
\right)
=
\theta^*_{\mathsf S}
+
\Sigma_{\mathsf{SS}}^{-1}\Sigma_{\mathsf{SR}}\theta^*_{\mathsf R}.
\]
\paragraph{Ridge.}
By \Cref{lem:projection-decomposition}, we can decompose
\[
\EE[(Y-\theta^\top X)^2]
=
(\theta-\theta^*)^\top\Sigma(\theta-\theta^*)+\sigma^2.
\]
Thus ordinary ridge minimizes
\[
\min_{\theta \in \RR^d} (\theta-\theta^*)^\top\Sigma(\theta-\theta^*)+\lambda\|\theta\|_2^2.
\]
Since $\Sigma+\lambda I_d \succ 0$, the solution is unique and equals 
\[
\theta^{\mathrm{OR}}(\lambda)
=
(\Sigma+\lambda I_d)^{-1}\Sigma\theta^*.
\]
\paragraph{Support-restricted ridge.}
Again write $\theta=(\eta,0_{\mathsf R})$.
The support-restricted ridge problem reduces to
\[
\min_{\eta\in\RR^{|\mathsf S|}}
\left\{
\EE[(Y-\eta^\top X_{\mathsf S})^2]
+
\lambda\|\eta\|_2^2
\right\}.
\]
Since $\Sigma_{\mathsf{SS}}+\lambda I_{|\mathsf S|} \succ 0$, the solution is unique again and after substituting \eqref{eq:Sigma-S-y-identity} equals
\[
\theta_{\mathsf S}^{\mathrm{OR}}(\mathsf S,\lambda)
=
(\Sigma_{\mathsf{SS}}+\lambda I_{|\mathsf S|})^{-1}\Sigma_{\mathsf S y}
= 
(\Sigma_{\mathsf{SS}}+\lambda I_{|\mathsf S|})^{-1} (\Sigma_{\mathsf{SS}}\theta^*_{\mathsf S} + \Sigma_{\mathsf{SR}}\theta^*_{\mathsf R}),
\]
as desired.
\end{proof}

The formula $\theta_{\mathsf S}^{\mathrm{LS}}(\mathsf S) = \theta^*_{\mathsf S} + \Sigma_{\mathsf{SS}}^{-1}\Sigma_{\mathsf{SR}}\theta^*_{\mathsf R}$
shows that omitted features can still affect the fitted coefficients on retained features through the covariance structure of $X$.
Thus, the value of dropping a feature depends not only on its own manipulability, but also on how its signal loads onto the features that remain.
This observation becomes important in the support restriction decomposition in \Cref{sec:theory}.

\begin{corollary}[Nesting of the policy levers]
\label{cor:policy-lever-nesting}
The support-restricted ordinary ridge estimator satisfies $\theta^{\mathrm{OR}}(\mathsf S,0)=\theta^{\mathrm{LS}}(\mathsf S)$ and, for $\mathsf S=[d]$, $\theta^{\mathrm{OR}}([d],\lambda)=\theta^{\mathrm{OR}}(\lambda)$.
\end{corollary}
\begin{proof}
Both identities follow immediately from the closed forms in \Cref{lem:policy-lever-closed-forms}.
\end{proof}

\section{Theoretical Details and Proofs for \Cref{sec:theory}}
\label{app:sec:theory}

Throughout this appendix, we use the same convention as in the main text and write $\mathrm{MSE}_{\mathrm{strat}}(\theta)$ as shorthand for $\mathrm{MSE}_{\mathrm{strat}}(\theta,0)$, as defined in \eqref{eq:mse-strat}.
Thus, by \Cref{lem:general-mse},
\[
\mathrm{MSE}_{\mathrm{strat}}(\theta)
=
(\theta-\theta^*)^\top\Sigma(\theta-\theta^*)
+
(\theta^\top H^{-1}\theta)^2
+
\sigma^2.
\]
For a support $\mathsf S\subseteq[d]$, 
\[
\Theta(\mathsf S)
:=
\{\theta\in\RR^d:\supp(\theta)\subseteq\mathsf S\},
\]
and
\[
\widetilde{\OPT}(\mathsf S)
:=
\min_{\theta\in\Theta(\mathsf S)}
\mathrm{MSE}_{\mathrm{strat}}(\theta),
\qquad
\widetilde{\OPT}
:=
\min_{\theta\in\RR^d}
\mathrm{MSE}_{\mathrm{strat}}(\theta).
\]

\subsection{Proof of \Cref{thm:lower-bound}}

Fix any $\mathsf S\subseteq[d]$ and $\lambda \ge 0$.
By construction, $\theta^{\mathrm{OR}}(\mathsf S,\lambda)\in\Theta(\mathsf S)$.
Therefore,
\[
\mathrm{MSE}_{\mathrm{strat}}(\theta^{\mathrm{OR}}(\mathsf S,\lambda))
\ge
\widetilde{\OPT}(\mathsf S).
\]
Moreover, since $\Theta(\mathsf S)\subseteq\RR^d$,
\[
\widetilde{\OPT}(\mathsf S)
\ge
\widetilde{\OPT}.
\]
Combining the two inequalities gives
\[
\mathrm{MSE}_{\mathrm{strat}}(\theta^{\mathrm{OR}}(\mathsf S,\lambda))
\ge
\widetilde{\OPT}.
\]
Subtracting $\OPT$ from both sides and minimizing over $(\mathsf S,\lambda)$ yields
\[
\widetilde{\OPT}-\OPT
\le
\min_{\mathsf S\subseteq[d],\,\lambda\ge0}
\left\{
\mathrm{MSE}_{\mathrm{strat}}(\theta^{\mathrm{OR}}(\mathsf S,\lambda))
-\OPT
\right\}.
\]
This proves the claimed lower bound.

\subsection{Proof of \Cref{prop:irreducible-error-bounds}}

\begin{proposition}[Irreducible error bounds]
\label{prop:irreducible-error-bounds}
We have $0\le \widetilde{\OPT}-\OPT \le (\theta^{*\top}H^{-1}\theta^*)^2$.
\end{proposition}

\begin{proof}
Since $\widetilde{\OPT}$ minimizes over zero-intercept linear rules, while $\OPT$ minimizes over both coefficients and intercepts, we have $\OPT\le\widetilde{\OPT}$.
Hence $0\le\widetilde{\OPT}-\OPT$.
For the upper bound, evaluate the zero-intercept benchmark at $\theta=\theta^*$: $\widetilde{\OPT} \le \mathrm{MSE}_{\mathrm{strat}}(\theta^*)$.
By the strategic MSE identity,
\[
\mathrm{MSE}_{\mathrm{strat}}(\theta^*)
=
(\theta^*-\theta^*)^\top\Sigma(\theta^*-\theta^*)
+
(\theta^{*\top}H^{-1}\theta^*)^2
+
\sigma^2
=
\sigma^2+(\theta^{*\top}H^{-1}\theta^*)^2.
\]
By \Cref{lem:strategic-opt}, $\OPT=\sigma^2$. Therefore $\widetilde{\OPT}-\OPT \le (\theta^{*\top}H^{-1}\theta^*)^2$.
\end{proof}

\subsection{Proof of \Cref{prop:support-restriction}}

\begin{proposition}[Support restriction identity]
\label{prop:support-restriction}
Fix a subset of features $\sset\subseteq[d]$ and let $\mathsf R:=[d]\setminus\sset$.
Define the predictive loss from omitting the coordinates in $\mathsf R$ 
\[
L_{\mathrm{pred}}(\sset)
:=
(\theta^*_{\mathsf R})^\top
\Sigma_{\mathsf R\mid \mathsf S}
\theta^*_{\mathsf R},
\qquad
\Sigma_{\mathsf R\mid \mathsf S}
:=
\Sigma_{\mathsf{RR}}
-
\Sigma_{\mathsf{RS}}
\Sigma_{\mathsf{SS}}^{-1}
\Sigma_{\mathsf{SR}},
\]
and the support-restricted strategic burden
\begin{equation}
\label{eq:Gamma-S}
\Gamma(\mathsf S)
:=
\min_{\eta\in\mathbb R^{|\mathsf S|}}
\left\{
(\eta-\theta_{\mathsf S}^{\mathrm{LS}}(\mathsf S))^\top
\Sigma_{\mathsf{SS}}
(\eta-\theta_{\mathsf S}^{\mathrm{LS}}(\mathsf S))
+
(\eta^\top H^{-1}_{\mathsf{SS}}\eta)^2
\right\}.
\end{equation}
Then
\begin{equation}
\label{eq:support-restriction-decomp}
\widetilde{\OPT}(\sset)-\widetilde{\OPT}
=
L_{\mathrm{pred}}(\sset)\;-\;\big(\Gamma([d])-\Gamma(\sset)\big).
\end{equation}
\end{proposition}

We first record a useful support-restricted decomposition.

\begin{lemma}[Support-restricted Pythagorean identity]
\label{lem:support-pythagorean}
Fix $\mathsf S\subseteq[d]$ and let $\mathsf R=[d]\setminus\mathsf S$.
For any $\theta\in\Theta(\mathsf S)$, write $\theta=(\eta,0_{\mathsf R})$, where $\eta\in\RR^{|\mathsf S|}$.
Let
\[
\theta_{\mathsf S}^{\mathrm{LS}}(\mathsf S)
=
\Sigma_{\mathsf{SS}}^{-1}\Sigma_{\mathsf S y}
=
\theta^*_{\mathsf S}
+
\Sigma_{\mathsf{SS}}^{-1}\Sigma_{\mathsf{SR}}\theta^*_{\mathsf R}.
\]
Then
\[
(\theta-\theta^*)^\top\Sigma(\theta-\theta^*)
=
L_{\mathrm{pred}}(\mathsf S)
+
(\eta-\theta_{\mathsf S}^{\mathrm{LS}}(\mathsf S))^\top
\Sigma_{\mathsf{SS}}
(\eta-\theta_{\mathsf S}^{\mathrm{LS}}(\mathsf S)),
\]
where
\[
L_{\mathrm{pred}}(\mathsf S)
=
(\theta^*_{\mathsf R})^\top
\Sigma_{\mathsf R\mid\mathsf S}
\theta^*_{\mathsf R},
\qquad
\Sigma_{\mathsf R\mid\mathsf S}
=
\Sigma_{\mathsf{RR}}
-
\Sigma_{\mathsf{RS}}\Sigma_{\mathsf{SS}}^{-1}\Sigma_{\mathsf{SR}}.
\]
\end{lemma}

\begin{proof}
Let $z:=\eta-\theta^*_{\mathsf S}$.
Since $\theta=(\eta,0_{\mathsf R})$, we have $\theta-\theta^* = (z,-\theta^*_{\mathsf R})$.
Expanding the quadratic form blockwise gives
\[
(\theta-\theta^*)^\top\Sigma(\theta-\theta^*)
=
z^\top\Sigma_{\mathsf{SS}}z
-
2z^\top\Sigma_{\mathsf{SR}}\theta^*_{\mathsf R}
+
(\theta^*_{\mathsf R})^\top\Sigma_{\mathsf{RR}}\theta^*_{\mathsf R}.
\]
Now define
$
z_0
:=
\eta-\theta_{\mathsf S}^{\mathrm{LS}}(\mathsf S)
=
z-\Sigma_{\mathsf{SS}}^{-1}\Sigma_{\mathsf{SR}}\theta^*_{\mathsf R}
$.
Completing the square yields
\[
z^\top\Sigma_{\mathsf{SS}}z
-
2z^\top\Sigma_{\mathsf{SR}}\theta^*_{\mathsf R}
=
z_0^\top\Sigma_{\mathsf{SS}}z_0
-
(\theta^*_{\mathsf R})^\top
\Sigma_{\mathsf{RS}}\Sigma_{\mathsf{SS}}^{-1}\Sigma_{\mathsf{SR}}
\theta^*_{\mathsf R}.
\]
Substituting this into the previous display gives
\[
(\theta-\theta^*)^\top\Sigma(\theta-\theta^*)
=
z_0^\top\Sigma_{\mathsf{SS}}z_0
+
(\theta^*_{\mathsf R})^\top
\left(
\Sigma_{\mathsf{RR}}
-
\Sigma_{\mathsf{RS}}\Sigma_{\mathsf{SS}}^{-1}\Sigma_{\mathsf{SR}}
\right)
\theta^*_{\mathsf R},
\]
which is the claimed identity.
\end{proof}

\begin{proof}[Proof of \Cref{prop:support-restriction}]
Fix $\mathsf S\subseteq[d]$ and write $\mathsf R=[d]\setminus\mathsf S$.
For any $\theta=(\eta,0_{\mathsf R})\in\Theta(\mathsf S)$, \Cref{lem:support-pythagorean} gives
\[
(\theta-\theta^*)^\top\Sigma(\theta-\theta^*)
=
L_{\mathrm{pred}}(\mathsf S)
+
(\eta-\theta_{\mathsf S}^{\mathrm{LS}}(\mathsf S))^\top
\Sigma_{\mathsf{SS}}
(\eta-\theta_{\mathsf S}^{\mathrm{LS}}(\mathsf S)).
\]
Moreover,
$
\theta^\top H^{-1}\theta
=
\eta^\top H^{-1}_{\mathsf{SS}}\eta
$.
Therefore
\[
\mathrm{MSE}_{\mathrm{strat}}(\theta)
=
\sigma^2
+
L_{\mathrm{pred}}(\mathsf S)
+
(\eta-\theta_{\mathsf S}^{\mathrm{LS}}(\mathsf S))^\top
\Sigma_{\mathsf{SS}}
(\eta-\theta_{\mathsf S}^{\mathrm{LS}}(\mathsf S))
+
(\eta^\top H^{-1}_{\mathsf{SS}}\eta)^2.
\]
Minimizing over $\eta\in\RR^{|\mathsf S|}$ yields
\[
\widetilde{\OPT}(\mathsf S)
=
\sigma^2
+
L_{\mathrm{pred}}(\mathsf S)
+
\Gamma(\mathsf S),
\]
where
\[
\Gamma(\mathsf S)
=
\min_{\eta\in\RR^{|\mathsf S|}}
\left\{
(\eta-\theta_{\mathsf S}^{\mathrm{LS}}(\mathsf S))^\top
\Sigma_{\mathsf{SS}}
(\eta-\theta_{\mathsf S}^{\mathrm{LS}}(\mathsf S))
+
(\eta^\top H^{-1}_{\mathsf{SS}}\eta)^2
\right\}.
\]
For the full support $[d]$, we have $L_{\mathrm{pred}}([d])=0$ and $\theta^{\mathrm{LS}}([d])=\theta^*$, so
\[
\widetilde{\OPT}
=
\widetilde{\OPT}([d])
=
\sigma^2+\Gamma([d]).
\]
Subtracting gives
\[
\widetilde{\OPT}(\mathsf S)-\widetilde{\OPT}
=
L_{\mathrm{pred}}(\mathsf S)
+
\Gamma(\mathsf S)-\Gamma([d])
=
L_{\mathrm{pred}}(\mathsf S)
-
\bigl(\Gamma([d])-\Gamma(\mathsf S)\bigr),
\]
as claimed.
\end{proof}

\subsection{Proof of \Cref{thm:fixed-support}}

\begin{theorem}[Near-optimality of ridge on a fixed support]
\label{thm:fixed-support}
For every $\mathsf S \subseteq [d]$,
\[
0
\le
\min_{\lambda\ge 0}
\left\{
\mathrm{MSE}_{\mathrm{strat}}(\theta^{\mathrm{OR}}(\mathsf S,\lambda))
-
\widetilde{\OPT}(\mathsf S)
\right\}
\le
C_H(\mathsf S)\,\delta_H(\mathsf S)^2,
\]
where
$C_H(\mathsf S):=4\,L_H(\mathsf S)\,\|\Sigma_{\mathsf{SS}}^{-1/2} H^{-1}_{\mathsf{SS}} \Sigma_{\mathsf{SS}}^{-1/2}\|_{\op}^2\,\|\Sigma_{\mathsf{SS}}^{1/2}\theta_{\mathsf S}^{\mathrm{LS}}(\mathsf S)\|_2^6$, with 
$L_H(\mathsf S):=1+6\|\Sigma_{\mathsf{SS}}^{-1/2} H^{-1}_{\mathsf{SS}} \Sigma_{\mathsf{SS}}^{-1/2}\|_{\op}^2\|\Sigma_{\mathsf{SS}}^{1/2}\theta_{\mathsf S}^{\mathrm{LS}}(\mathsf S)\|_2^2$.
\end{theorem}

We now prove the fixed-support ridge approximation theorem.
Fix a support $\mathsf S\subseteq[d]$.
Let $\eta_{\mathsf S}^\star$ denote the unique minimizer in the definition of $\Gamma(\mathsf S)$:
\[
\eta_{\mathsf S}^\star
\in
\argmin_{\eta\in\RR^{|\mathsf S|}}
\left\{
(\eta-\theta_{\mathsf S}^{\mathrm{LS}}(\mathsf S))^\top
\Sigma_{\mathsf{SS}}
(\eta-\theta_{\mathsf S}^{\mathrm{LS}}(\mathsf S))
+
(\eta^\top H^{-1}_{\mathsf{SS}}\eta)^2
\right\}.
\]
For $\lambda\ge0$, let $\theta_{\mathsf S}^{\mathrm{OR}}(\lambda)$ denote the retained-coordinate vector of $\theta^{\mathrm{OR}}(\mathsf S,\lambda)$:
\[
\theta_{\mathsf S}^{\mathrm{OR}}(\lambda)
=
(\Sigma_{\mathsf{SS}}+\lambda I)^{-1}
\Sigma_{\mathsf{SS}}
\theta_{\mathsf S}^{\mathrm{LS}}(\mathsf S).
\]

\begin{lemma}[Support-restricted oracle is generalized ridge]
\label{lem:fixed-support-oracle-gr}
The minimizer $\eta_{\mathsf S}^\star$ is unique. Moreover, defining
$
\alpha_{\mathsf S}
:=
2\,\eta_{\mathsf S}^{\star\top}
H^{-1}_{\mathsf{SS}}
\eta_{\mathsf S}^\star
$,
we have
\[
(\Sigma_{\mathsf{SS}}+\alpha_{\mathsf S}H^{-1}_{\mathsf{SS}})
\eta_{\mathsf S}^\star
=
\Sigma_{\mathsf{SS}}
\theta_{\mathsf S}^{\mathrm{LS}}(\mathsf S).
\]
Equivalently,
\[
\eta_{\mathsf S}^\star
=
(\Sigma_{\mathsf{SS}}+\alpha_{\mathsf S}H^{-1}_{\mathsf{SS}})^{-1}
\Sigma_{\mathsf{SS}}
\theta_{\mathsf S}^{\mathrm{LS}}(\mathsf S).
\]
\end{lemma}

\begin{proof}
The Hessian of the support-restricted objective in $\eta$ is
\[
2\Sigma_{\mathsf{SS}}
+
8H^{-1}_{\mathsf{SS}}\eta\eta^\top H^{-1}_{\mathsf{SS}}
+
4(\eta^\top H^{-1}_{\mathsf{SS}}\eta)H^{-1}_{\mathsf{SS}}
\succeq
2\Sigma_{\mathsf{SS}}
\succ0.
\]
Thus the objective is strongly convex and has a unique minimizer.
The first-order condition at $\eta_{\mathsf S}^\star$ is
\[
2\Sigma_{\mathsf{SS}}
(\eta_{\mathsf S}^\star-\theta_{\mathsf S}^{\mathrm{LS}}(\mathsf S))
+
4(\eta_{\mathsf S}^{\star\top}H^{-1}_{\mathsf{SS}}\eta_{\mathsf S}^\star)
H^{-1}_{\mathsf{SS}}\eta_{\mathsf S}^\star
=
0.
\]
Dividing by $2$ and using the definition of $\alpha_{\mathsf S}$ gives
\[
\Sigma_{\mathsf{SS}}
(\eta_{\mathsf S}^\star-\theta_{\mathsf S}^{\mathrm{LS}}(\mathsf S))
+
\alpha_{\mathsf S}H^{-1}_{\mathsf{SS}}\eta_{\mathsf S}^\star
=
0.
\]
Equivalently,
\[
(\Sigma_{\mathsf{SS}}+\alpha_{\mathsf S}H^{-1}_{\mathsf{SS}})
\eta_{\mathsf S}^\star
=
\Sigma_{\mathsf{SS}}
\theta_{\mathsf S}^{\mathrm{LS}}(\mathsf S).
\]
\end{proof}

We use whitened coordinates on the support.
Define
\[
u_{\mathsf S}^{\star}
:=
\Sigma_{\mathsf{SS}}^{1/2}\eta_{\mathsf S}^\star,
\qquad
u_{\mathsf S}^{\mathrm{LS}}
:=
\Sigma_{\mathsf{SS}}^{1/2}\theta_{\mathsf S}^{\mathrm{LS}}(\mathsf S),
\qquad
u_{\lambda,\mathsf S}
:=
\Sigma_{\mathsf{SS}}^{1/2}\theta_{\mathsf S}^{\mathrm{OR}}(\lambda).
\]
The support-restricted oracle satisfies
\[
\left(
I
+
\alpha_{\mathsf S}
\Sigma_{\mathsf{SS}}^{-1/2}
H^{-1}_{\mathsf{SS}}
\Sigma_{\mathsf{SS}}^{-1/2}
\right)
u_{\mathsf S}^{\star}
=
u_{\mathsf S}^{\mathrm{LS}},
\]
while the ordinary ridge path satisfies
\[
(I+\lambda\Sigma_{\mathsf{SS}}^{-1})
u_{\lambda,\mathsf S}
=
u_{\mathsf S}^{\mathrm{LS}}.
\]

\begin{lemma}[Distance from the oracle to the ordinary ridge path]
\label{lem:ridge-path-distance}
For every $\gamma\ge0$, if $\lambda=\alpha_{\mathsf S}\gamma$, then
\[
\|u_{\lambda,\mathsf S}-u_{\mathsf S}^{\star}\|_2
\le
\alpha_{\mathsf S}
\left\|
\Sigma_{\mathsf{SS}}^{-1/2}
(H^{-1}_{\mathsf{SS}}-\gamma I)
\Sigma_{\mathsf{SS}}^{-1/2}
u_{\mathsf S}^{\star}
\right\|_2.
\]
Consequently,
\[
\inf_{\lambda\ge0}
\|u_{\lambda,\mathsf S}-u_{\mathsf S}^{\star}\|_2
\le
\alpha_{\mathsf S}\,
\delta_H(\mathsf S)\,
\|u_{\mathsf S}^{\star}\|_2.
\]
\end{lemma}

\begin{proof}
Subtracting the whitened normal equations gives
\[
(I+\lambda\Sigma_{\mathsf{SS}}^{-1})
(u_{\lambda,\mathsf S}-u_{\mathsf S}^{\star})
=
\left(
\alpha_{\mathsf S}
\Sigma_{\mathsf{SS}}^{-1/2}
H^{-1}_{\mathsf{SS}}
\Sigma_{\mathsf{SS}}^{-1/2}
-
\lambda\Sigma_{\mathsf{SS}}^{-1}
\right)
u_{\mathsf S}^{\star}.
\]
Therefore
\[
u_{\lambda,\mathsf S}-u_{\mathsf S}^{\star}
=
(I+\lambda\Sigma_{\mathsf{SS}}^{-1})^{-1}
\left(
\alpha_{\mathsf S}
\Sigma_{\mathsf{SS}}^{-1/2}
H^{-1}_{\mathsf{SS}}
\Sigma_{\mathsf{SS}}^{-1/2}
-
\lambda\Sigma_{\mathsf{SS}}^{-1}
\right)
u_{\mathsf S}^{\star}.
\]
Choosing $\lambda=\alpha_{\mathsf S}\gamma$ gives
\[
u_{\alpha_{\mathsf S}\gamma,\mathsf S}-u_{\mathsf S}^{\star}
=
\alpha_{\mathsf S}
(I+\alpha_{\mathsf S}\gamma\Sigma_{\mathsf{SS}}^{-1})^{-1}
\Sigma_{\mathsf{SS}}^{-1/2}
(H^{-1}_{\mathsf{SS}}-\gamma I)
\Sigma_{\mathsf{SS}}^{-1/2}
u_{\mathsf S}^{\star}.
\]
Since $\Sigma_{\mathsf{SS}}^{-1}\succ0$,
\[
\|(I+\alpha_{\mathsf S}\gamma\Sigma_{\mathsf{SS}}^{-1})^{-1}\|_{\op}\le1.
\]
Thus
\[
\|u_{\alpha_{\mathsf S}\gamma,\mathsf S}-u_{\mathsf S}^{\star}\|_2
\le
\alpha_{\mathsf S}
\left\|
\Sigma_{\mathsf{SS}}^{-1/2}
(H^{-1}_{\mathsf{SS}}-\gamma I)
\Sigma_{\mathsf{SS}}^{-1/2}
u_{\mathsf S}^{\star}
\right\|_2.
\]
Taking the infimum over $\gamma\ge0$ and using the definition of $\delta_H(\mathsf S)$ gives
\[
\inf_{\lambda\ge0}
\|u_{\lambda,\mathsf S}-u_{\mathsf S}^{\star}\|_2
\le
\alpha_{\mathsf S}
\delta_H(\mathsf S)
\|u_{\mathsf S}^{\star}\|_2.
\]
\end{proof}

\begin{lemma}[Curvature bound on the fixed support]
\label{lem:fixed-support-curvature}
Let
\[
L_H(\mathsf S)
:=
1+
6
\left\|
\Sigma_{\mathsf{SS}}^{-1/2}
H^{-1}_{\mathsf{SS}}
\Sigma_{\mathsf{SS}}^{-1/2}
\right\|_{\op}^2
\left\|
\Sigma_{\mathsf{SS}}^{1/2}
\theta_{\mathsf S}^{\mathrm{LS}}(\mathsf S)
\right\|_2^2.
\]
Then
\[
\mathrm{MSE}_{\mathrm{strat}}(\theta^{\mathrm{OR}}(\mathsf S,\lambda))
-
\widetilde{\OPT}(\mathsf S)
\le
L_H(\mathsf S)
\|u_{\lambda,\mathsf S}-u_{\mathsf S}^{\star}\|_2^2
\]
for every $\lambda\ge0$.
\end{lemma}

\begin{proof}
After subtracting the support-dependent constants
$\sigma^2+L_{\mathrm{pred}}(\mathsf S)$, the support-restricted zero-intercept objective in whitened coordinates is
\[
u
\mapsto
\|u-u_{\mathsf S}^{\mathrm{LS}}\|_2^2
+
\left(
u^\top
\Sigma_{\mathsf{SS}}^{-1/2}
H^{-1}_{\mathsf{SS}}
\Sigma_{\mathsf{SS}}^{-1/2}
u
\right)^2.
\]
The Hessian of this function is
\[
2I
+
8
\Sigma_{\mathsf{SS}}^{-1/2}H^{-1}_{\mathsf{SS}}\Sigma_{\mathsf{SS}}^{-1/2}
uu^\top
\Sigma_{\mathsf{SS}}^{-1/2}H^{-1}_{\mathsf{SS}}\Sigma_{\mathsf{SS}}^{-1/2}
+
4
\left(
u^\top
\Sigma_{\mathsf{SS}}^{-1/2}H^{-1}_{\mathsf{SS}}\Sigma_{\mathsf{SS}}^{-1/2}
u
\right)
\Sigma_{\mathsf{SS}}^{-1/2}H^{-1}_{\mathsf{SS}}\Sigma_{\mathsf{SS}}^{-1/2}.
\]
Hence its operator norm is at most
\[
2
+
12
\left\|
\Sigma_{\mathsf{SS}}^{-1/2}
H^{-1}_{\mathsf{SS}}
\Sigma_{\mathsf{SS}}^{-1/2}
\right\|_{\op}^2
\|u\|_2^2.
\]
Both $u_{\mathsf S}^{\star}$ and $u_{\lambda,\mathsf S}$ lie in the ball $
\|u\|_2
\le
\|u_{\mathsf S}^{\mathrm{LS}}\|_2.$
Indeed,
\[
u_{\mathsf S}^{\star}
=
\left(
I+
\alpha_{\mathsf S}
\Sigma_{\mathsf{SS}}^{-1/2}
H^{-1}_{\mathsf{SS}}
\Sigma_{\mathsf{SS}}^{-1/2}
\right)^{-1}
u_{\mathsf S}^{\mathrm{LS}},
\]
and
\[
u_{\lambda,\mathsf S}
=
(I+\lambda\Sigma_{\mathsf{SS}}^{-1})^{-1}
u_{\mathsf S}^{\mathrm{LS}},
\]
and both resolvents have operator norm at most $1$.
Therefore, along the line segment between $u_{\mathsf S}^{\star}$ and $u_{\lambda,\mathsf S}$, the Hessian operator norm is at most $2L_H(\mathsf S)$.
Since $u_{\mathsf S}^{\star}$ is the minimizer, Taylor's theorem gives
\[
\mathrm{MSE}_{\mathrm{strat}}(\theta^{\mathrm{OR}}(\mathsf S,\lambda))
-
\widetilde{\OPT}(\mathsf S)
\le
L_H(\mathsf S)
\|u_{\lambda,\mathsf S}-u_{\mathsf S}^{\star}\|_2^2.
\]
\end{proof}

\begin{proof}[Proof of \Cref{thm:fixed-support}]
By \Cref{lem:fixed-support-curvature} and \Cref{lem:ridge-path-distance},
\[
\min_{\lambda\ge0}
\left\{
\mathrm{MSE}_{\mathrm{strat}}(\theta^{\mathrm{OR}}(\mathsf S,\lambda))
-
\widetilde{\OPT}(\mathsf S)
\right\}
\le
\alpha_{\mathsf S}^2
L_H(\mathsf S)
\delta_H(\mathsf S)^2
\|u_{\mathsf S}^{\star}\|_2^2.
\]
Now
\[
\|u_{\mathsf S}^{\star}\|_2
\le
\|u_{\mathsf S}^{\mathrm{LS}}\|_2
=
\left\|
\Sigma_{\mathsf{SS}}^{1/2}
\theta_{\mathsf S}^{\mathrm{LS}}(\mathsf S)
\right\|_2.
\]
Also,
\[
\alpha_{\mathsf S}
=
2\eta_{\mathsf S}^{\star\top}
H^{-1}_{\mathsf{SS}}
\eta_{\mathsf S}^{\star}
=
2u_{\mathsf S}^{\star\top}
\Sigma_{\mathsf{SS}}^{-1/2}
H^{-1}_{\mathsf{SS}}
\Sigma_{\mathsf{SS}}^{-1/2}
u_{\mathsf S}^{\star},
\]
so
\[
\alpha_{\mathsf S}
\le
2
\left\|
\Sigma_{\mathsf{SS}}^{-1/2}
H^{-1}_{\mathsf{SS}}
\Sigma_{\mathsf{SS}}^{-1/2}
\right\|_{\op}
\|u_{\mathsf S}^{\star}\|_2^2
\le
2
\left\|
\Sigma_{\mathsf{SS}}^{-1/2}
H^{-1}_{\mathsf{SS}}
\Sigma_{\mathsf{SS}}^{-1/2}
\right\|_{\op}
\|u_{\mathsf S}^{\mathrm{LS}}\|_2^2.
\]
Therefore
\[
\alpha_{\mathsf S}^2
\|u_{\mathsf S}^{\star}\|_2^2
\le
4
\left\|
\Sigma_{\mathsf{SS}}^{-1/2}
H^{-1}_{\mathsf{SS}}
\Sigma_{\mathsf{SS}}^{-1/2}
\right\|_{\op}^2
\|u_{\mathsf S}^{\mathrm{LS}}\|_2^6.
\]
Combining the displays gives
\[
\min_{\lambda\ge0}
\left\{
\mathrm{MSE}_{\mathrm{strat}}(\theta^{\mathrm{OR}}(\mathsf S,\lambda))
-
\widetilde{\OPT}(\mathsf S)
\right\}
\le
C_H(\mathsf S)\delta_H(\mathsf S)^2,
\]
with the valid choice
\[
C_H(\mathsf S)
:=
4\,L_H(\mathsf S)\,
\left\|
\Sigma_{\mathsf{SS}}^{-1/2}
H^{-1}_{\mathsf{SS}}
\Sigma_{\mathsf{SS}}^{-1/2}
\right\|_{\op}^2
\left\|
\Sigma_{\mathsf{SS}}^{1/2}
\theta_{\mathsf S}^{\mathrm{LS}}(\mathsf S)
\right\|_2^6.
\]
\end{proof}

\subsection{Proof of \Cref{thm:global-upper}}
\label{sec:proof-upper}

For every support $\mathsf S$ and every $\lambda\ge0$, we have the exact decomposition
\[
\mathrm{MSE}_{\mathrm{strat}}(\theta^{\mathrm{OR}}(\mathsf S,\lambda))
-
\widetilde{\OPT}
=
\left(
\widetilde{\OPT}(\mathsf S)-\widetilde{\OPT}
\right)
+
\left(
\mathrm{MSE}_{\mathrm{strat}}(\theta^{\mathrm{OR}}(\mathsf S,\lambda))
-
\widetilde{\OPT}(\mathsf S)
\right).
\]
By \Cref{prop:support-restriction},
\[
\widetilde{\OPT}(\mathsf S)-\widetilde{\OPT}
=
L_{\mathrm{pred}}(\mathsf S)
-
\bigl(\Gamma([d])-\Gamma(\mathsf S)\bigr).
\]
By \Cref{thm:fixed-support},
\[
\min_{\lambda\ge0}
\left\{
\mathrm{MSE}_{\mathrm{strat}}(\theta^{\mathrm{OR}}(\mathsf S,\lambda))
-
\widetilde{\OPT}(\mathsf S)
\right\}
\le
C_H(\mathsf S)\delta_H(\mathsf S)^2.
\]
Therefore, for every $\mathsf S$,
\[
\min_{\lambda\ge0}
\left\{
\mathrm{MSE}_{\mathrm{strat}}(\theta^{\mathrm{OR}}(\mathsf S,\lambda))
-
\widetilde{\OPT}
\right\}
\le
L_{\mathrm{pred}}(\mathsf S)
-
\bigl(\Gamma([d])-\Gamma(\mathsf S)\bigr)
+
C_H(\mathsf S)\delta_H(\mathsf S)^2.
\]

\subsection{Structure of the zero-intercept oracle}
\label{sec:zero-intercept-oracle}

The next proposition records the generalized-ridge structure of the zero-intercept oracle over the full feature set.
It is the full-support analogue of \Cref{lem:fixed-support-oracle-gr}.

\begin{proposition}[Zero-intercept oracle over $\RR^d$]
\label{prop:gr-optimum}
The objective $\mathrm{MSE}_{\mathrm{strat}}(\theta)$ is strongly convex and therefore admits a unique minimizer
$
\theta^{\widetilde{\OPT}}
\in\RR^d
$.
Moreover, $\theta^{\widetilde{\OPT}}$ is characterized by
\[
\Sigma(\theta^{\widetilde{\OPT}}-\theta^*)
+
2(\theta^{\widetilde{\OPT}\top}H^{-1}\theta^{\widetilde{\OPT}})
H^{-1}\theta^{\widetilde{\OPT}}
=
0.
\]
Equivalently, defining
\[
\lambda^{\widetilde{\OPT}}
:=
2\,\theta^{\widetilde{\OPT}\top}H^{-1}\theta^{\widetilde{\OPT}},
\]
we have
\[
(\Sigma+\lambda^{\widetilde{\OPT}}H^{-1})
\theta^{\widetilde{\OPT}}
=
\Sigma\theta^*,
\]
and hence
\[
\theta^{\widetilde{\OPT}}
=
(\Sigma+\lambda^{\widetilde{\OPT}}H^{-1})^{-1}
\Sigma\theta^*.
\]
Thus the zero-intercept oracle is a generalized-ridge shrinkage of $\theta^*$, with the regularization level determined endogenously by
\[
\lambda^{\widetilde{\OPT}}
=
2\,\theta^{\widetilde{\OPT}\top}H^{-1}\theta^{\widetilde{\OPT}}.
\]
\end{proposition}

\begin{proof}
Ignoring the additive constant $\sigma^2$, the zero-intercept objective is
\[
\mathrm{MSE}_{\mathrm{strat}}(\theta)-\sigma^2
=
(\theta-\theta^*)^\top\Sigma(\theta-\theta^*)
+
(\theta^\top H^{-1}\theta)^2.
\]
Its gradient is
\[
\nabla_\theta\mathrm{MSE}_{\mathrm{strat}}(\theta)
=
2\Sigma(\theta-\theta^*)
+
4(\theta^\top H^{-1}\theta)H^{-1}\theta,
\]
and its Hessian is
\[
\nabla_\theta^2\mathrm{MSE}_{\mathrm{strat}}(\theta)
=
2\Sigma
+
8H^{-1}\theta\theta^\top H^{-1}
+
4(\theta^\top H^{-1}\theta)H^{-1}.
\]
Since $\Sigma\succ0$ and $H^{-1}\succ0$,
\[
\nabla_\theta^2\mathrm{MSE}_{\mathrm{strat}}(\theta)
\succeq
2\Sigma
\succ0.
\]
Thus $\mathrm{MSE}_{\mathrm{strat}}(\theta)$ is strongly convex and has a unique minimizer.

At the minimizer $\theta^{\widetilde{\OPT}}$, the first-order condition gives
\[
2\Sigma(\theta^{\widetilde{\OPT}}-\theta^*)
+
4(\theta^{\widetilde{\OPT}\top}H^{-1}\theta^{\widetilde{\OPT}})
H^{-1}\theta^{\widetilde{\OPT}}
=
0.
\]
Dividing by $2$ and defining
\[
\lambda^{\widetilde{\OPT}}
=
2\,\theta^{\widetilde{\OPT}\top}H^{-1}\theta^{\widetilde{\OPT}}
\]
yields
\[
(\Sigma+\lambda^{\widetilde{\OPT}}H^{-1})
\theta^{\widetilde{\OPT}}
=
\Sigma\theta^*.
\]
Since $\Sigma+\lambda^{\widetilde{\OPT}}H^{-1}\succ0$, this is equivalent to
\[
\theta^{\widetilde{\OPT}}
=
(\Sigma+\lambda^{\widetilde{\OPT}}H^{-1})^{-1}
\Sigma\theta^*.
\]
\end{proof}

\subsection{Ridge optimality under isotropic costs}
\label{sec:zero-ridge-gap-isotropy}

The next corollary formalizes the case in which scalar ridge is expressive enough to attain the support-restricted zero-intercept oracle exactly.

\begin{corollary}[Zero heterogeneity gap under isotropy]
\label{cor:zero-ridge-gap}
Fix a support $\mathsf S\subseteq[d]$. If $H^{-1}_{\mathsf{SS}}=\gamma I$ for some $\gamma \ge 0$, then 
\[
\min_{\lambda\ge 0} \mathrm{MSE}_{\mathrm{strat}}(\theta^{\mathrm{OR}}(\mathsf S,\lambda)) = \widetilde{\OPT}(\mathsf S). 
\]
\end{corollary}

\begin{proof}[Proof of \Cref{cor:zero-ridge-gap}]
Let $\eta_{\mathsf S}^{\star}$ denote the support-restricted zero-intercept oracle on $\mathsf S$.
By \Cref{lem:fixed-support-oracle-gr}, there exists
$
\alpha_{\mathsf S}
:=
2\,\eta_{\mathsf S}^{\star\top}
H^{-1}_{\mathsf{SS}}
\eta_{\mathsf S}^{\star}
$
such that
\[
(\Sigma_{\mathsf{SS}}+\alpha_{\mathsf S}H^{-1}_{\mathsf{SS}})
\eta_{\mathsf S}^{\star}
=
\Sigma_{\mathsf{SS}}
\theta_{\mathsf S}^{\mathrm{LS}}(\mathsf S).
\]
If $H^{-1}_{\mathsf{SS}}=\gamma I$, then
\[
(\Sigma_{\mathsf{SS}}+\alpha_{\mathsf S}\gamma I)
\eta_{\mathsf S}^{\star}
=
\Sigma_{\mathsf{SS}}
\theta_{\mathsf S}^{\mathrm{LS}}(\mathsf S).
\]
Thus $\eta_{\mathsf S}^{\star}$ is exactly the ordinary ridge estimator on support $\mathsf S$ with ridge parameter
$
\lambda=\alpha_{\mathsf S}\gamma.
$
Therefore the ordinary ridge path contains the support-restricted zero-intercept oracle, and so
\[
\min_{\lambda\ge0}
\mathrm{MSE}_{\mathrm{strat}}
\bigl(\theta^{\mathrm{OR}}(\mathsf S,\lambda)\bigr)
=
\widetilde{\OPT}(\mathsf S).
\]
\end{proof}

\Cref{cor:zero-ridge-gap} shows that scalar ridge is exact when the retained manipulation costs are isotropic.
This condition is sufficient but not necessary.
\Cref{prop:necessary-sufficient-ridge-optimality} below gives the exact necessary-and-sufficient condition: the fixed-support ridge gap is zero if and only if the support-restricted oracle lies in an eigenspace of the retained inverse-cost matrix.

\subsection{Necessary and sufficient condition for ridge optimality}
\label{sec:necessary-sufficient-ridge-optimality}
The isotropy condition in \Cref{cor:zero-ridge-gap} is sufficient but not necessary.
The exact necessary-and-sufficient condition is that the support-restricted oracle lies in an eigenspace of the inverse-cost matrix on the retained support.

\begin{proposition}[Necessary and sufficient condition for exact fixed-support ridge optimality]
\label{prop:necessary-sufficient-ridge-optimality}
Fix a support $\mathsf S\subseteq[d]$.
Let $\eta_{\mathsf S}^{\star}$ denote the support-restricted zero-intercept oracle on $\mathsf S$.
Then
\[
\min_{\lambda\ge0}
\mathrm{MSE}_{\mathrm{strat}}(\theta^{\mathrm{OR}}(\mathsf S,\lambda))
=
\widetilde{\OPT}(\mathsf S)
\]
if and only if there exists $\gamma\ge0$ such that
$
(H^{-1}_{\mathsf{SS}}-\gamma I)\eta_{\mathsf S}^{\star}=0
$.
\end{proposition}

\begin{proof}
The support-restricted oracle $\eta_{\mathsf S}^{\star}$ satisfies
\[
(\Sigma_{\mathsf{SS}}+\alpha_{\mathsf S}H^{-1}_{\mathsf{SS}})
\eta_{\mathsf S}^{\star}
=
\Sigma_{\mathsf{SS}}
\theta_{\mathsf S}^{\mathrm{LS}}(\mathsf S),
\]
where
$
\alpha_{\mathsf S}
=
2\eta_{\mathsf S}^{\star\top}
H^{-1}_{\mathsf{SS}}
\eta_{\mathsf S}^{\star}
$.
The ordinary ridge estimator at level $\lambda$ satisfies
\[
(\Sigma_{\mathsf{SS}}+\lambda I)
\theta_{\mathsf S}^{\mathrm{OR}}(\lambda)
=
\Sigma_{\mathsf{SS}}
\theta_{\mathsf S}^{\mathrm{LS}}(\mathsf S).
\]

Suppose first that the fixed-support ridge gap is zero.
Since the support-restricted objective is strongly convex, its minimizer is unique. Hence there exists $\lambda\ge0$ such that
$
\theta_{\mathsf S}^{\mathrm{OR}}(\lambda)
=
\eta_{\mathsf S}^{\star}
$.
Subtracting the two normal equations gives $(\alpha_{\mathsf S}H^{-1}_{\mathsf{SS}}-\lambda I)\eta_{\mathsf S}^{\star}=0$.
If $\alpha_{\mathsf S}>0$, set $\gamma=\lambda/\alpha_{\mathsf S}$. 
Then $(H^{-1}_{\mathsf{SS}}-\gamma I)\eta_{\mathsf S}^{\star}=0$.
If $\alpha_{\mathsf S}=0$, then $\eta_{\mathsf S}^{\star}=0$ because $H^{-1}_{\mathsf{SS}}\succ0$. 
In this case, $(H^{-1}_{\mathsf{SS}}-\gamma I)\eta_{\mathsf S}^{\star}=0$
holds for every $\gamma\ge0$.

Conversely, suppose there exists $\gamma\ge0$ such that $(H^{-1}_{\mathsf{SS}}-\gamma I)\eta_{\mathsf S}^{\star}=0$.
If $\eta_{\mathsf S}^{\star}=0$, then the first-order condition implies
$
\Sigma_{\mathsf{SS}}\theta_{\mathsf S}^{\mathrm{LS}}(\mathsf S)=0
$,
and hence $\theta_{\mathsf S}^{\mathrm{LS}}(\mathsf S)=0$. Ordinary ridge also returns $0$ for every $\lambda\ge0$, so the ridge gap is zero.

If $\eta_{\mathsf S}^{\star}\neq0$, then $\alpha_{\mathsf S}>0$. Set $\lambda=\alpha_{\mathsf S}\gamma$. Then
\[
(\Sigma_{\mathsf{SS}}+\lambda I)\eta_{\mathsf S}^{\star}
=
(\Sigma_{\mathsf{SS}}+\alpha_{\mathsf S}\gamma I)\eta_{\mathsf S}^{\star}
=
(\Sigma_{\mathsf{SS}}+\alpha_{\mathsf S}H^{-1}_{\mathsf{SS}})
\eta_{\mathsf S}^{\star}
=
\Sigma_{\mathsf{SS}}
\theta_{\mathsf S}^{\mathrm{LS}}(\mathsf S).
\]
Therefore $\eta_{\mathsf S}^{\star}$ is the ordinary ridge estimator at level $\lambda$, so ordinary ridge attains $\widetilde{\OPT}(\mathsf S)$.
\end{proof}

\subsection{Two-level cost model}
\label{app:two-level-inverse-cost}

This subsection formalizes the intuition that high manipulability alone does not force exclusion.
What matters for scalar ridge is whether the retained features have manipulation costs that are homogeneous enough to be controlled by a single regularization parameter.

Partition the features as $[d]=\mathsf S_m\cup\mathsf S_b$ with $\mathsf S_m\cap\mathsf S_b=\emptyset$, where $\mathsf S_m$ denotes a more manipulable feature group and $\mathsf S_b$ denotes a baseline feature group.
For $\kappa>\mu>0$, consider the two-level inverse-cost model
\[
H^{-1}_{\kappa,\mu}
=
\begin{pmatrix}
\kappa I_{\mathsf S_m} & 0\\
0 & \mu I_{\mathsf S_b}
\end{pmatrix},
\]
where the rows and columns are ordered with the features in $\mathsf S_m$ first.
The special case $\mu=1$ corresponds to a highly manipulable group and a baseline group; the limiting regime $\mu\to0$ corresponds to a baseline group that is essentially non-manipulable.

\begin{proposition}[Homogeneity and heterogeneity in the two-level cost model]
\label{prop:two-level-inverse-cost}
Assume $\Sigma = I$ and $H^{-1}=H^{-1}_{\kappa,\mu}$ with $\kappa>\mu>0$.
For any support $\mathsf T\subseteq[d]$,
  \[
\delta_H(\mathsf T)
=
\begin{cases}
0,
& \mathsf T\subseteq \mathsf S_m
\text{ or }
\mathsf T\subseteq \mathsf S_b,\\[0.35em]
(\kappa-\mu)/2,
& \mathsf T\cap\mathsf S_m\neq\emptyset
\text{ and }
\mathsf T\cap\mathsf S_b\neq\emptyset.
\end{cases}
\]
Consequently, if \(\mathsf T\subseteq \mathsf S_m\) or
\(\mathsf T\subseteq \mathsf S_b\), then
\[
\min_{\lambda\ge 0}
\left\{
\mathrm{MSE}_{\mathrm{strat}}
\bigl(\theta^{\mathrm{OR}}(\mathsf T,\lambda)\bigr)
-
\widetilde{\OPT}(\mathsf T)
\right\}
=0.
\]
\end{proposition}
\begin{proof}
Under the isotropic covariance $\Sigma = I$, we have
\[
\delta_H(\mathsf T)
=
\inf_{\gamma\ge 0}
\left\|
H^{-1}_{\mathsf{TT}}-\gamma I
\right\|_{\op}.
\]
If \(\mathsf T\subseteq \mathsf S_m\), then
\(H^{-1}_{\mathsf{TT}}=\kappa I_{\mathsf T}\), so choosing
\(\gamma=\kappa\) gives \(\delta_H(\mathsf T)=0\).
Similarly, if \(\mathsf T\subseteq \mathsf S_b\), then
\(H^{-1}_{\mathsf{TT}}=\mu I_{\mathsf T}\), so choosing
\(\gamma=\mu\) gives \(\delta_H(\mathsf T)=0\).

Now suppose that $\mathsf T\cap\mathsf S_m\neq\emptyset$ and $\mathsf T\cap\mathsf S_b\neq\emptyset$.
Then \(H^{-1}_{\mathsf{TT}}\) has eigenvalue \(\kappa\) on
\(\mathsf T\cap\mathsf S_m\) and eigenvalue \(\mu\) on
\(\mathsf T\cap\mathsf S_b\).
Therefore
\[
\left\|
H^{-1}_{\mathsf{TT}}-\gamma I
\right\|_{\op}
=
\max\{|\kappa-\gamma|,\ |\mu-\gamma|\}.
\]
The minimizer is \(\gamma=(\kappa+\mu)/2\), and the minimum value is
\((\kappa-\mu)/2\).
This proves the expression for \(\delta_H(\mathsf T)\).

The zero scalar-ridge gap for supports contained in a single group follows from
\(\delta_H(\mathsf T)=0\) together with \Cref{thm:fixed-support}. 
It also follows directly from \Cref{prop:two-level-inverse-cost}, since
\(H^{-1}_{\mathsf{TT}}\) is a scalar multiple of the identity on such supports.
\end{proof}

\Cref{prop:two-level-inverse-cost} leads to a near-optimality result for the optimal support-restricted ridge estimator:

\begin{corollary}[Near-optimality of support-restricted ridge in the two-level cost model]
Assume the setup of \Cref{prop:two-level-inverse-cost}.
If there exists a support $\mathsf T\subseteq \mathsf S_m$ or $\mathsf T\subseteq \mathsf S_b$ such that
\[
L_{\mathrm{pred}}(\mathsf T)
-
\bigl(\Gamma([d])-\Gamma(\mathsf T)\bigr)
\le
\varepsilon,
\]
then
\[
\min_{\mathsf S\subseteq[d],\,\lambda\ge 0}
\mathrm{MSE}_{\mathrm{strat}}
\bigl(\theta^{\mathrm{OR}}(\mathsf S,\lambda)\bigr)
\le
\widetilde{\OPT}+\varepsilon.
\]
\end{corollary}
\begin{proof}
By \Cref{prop:two-level-inverse-cost}, any support contained entirely in
\(\mathsf S_m\) or entirely in \(\mathsf S_b\) has zero scalar-ridge gap:
\[
\min_{\lambda\ge 0}
\mathrm{MSE}_{\mathrm{strat}}
\bigl(\theta^{\mathrm{OR}}(\mathsf T,\lambda)\bigr)
=
\widetilde{\OPT}(\mathsf T).
\]
Therefore
\[
\min_{\lambda\ge 0}
\mathrm{MSE}_{\mathrm{strat}}
\bigl(\theta^{\mathrm{OR}}(\mathsf T,\lambda)\bigr)
-
\widetilde{\OPT}
=
\widetilde{\OPT}(\mathsf T)-\widetilde{\OPT}.
\]
Using the support restriction identity from
\Cref{prop:support-restriction},
\[
\widetilde{\OPT}(\mathsf T)-\widetilde{\OPT}
=
L_{\mathrm{pred}}(\mathsf T)
-
\bigl(\Gamma([d])-\Gamma(\mathsf T)\bigr).
\]
By assumption, this quantity is at most \(\varepsilon\).
Taking the minimum over all supports and all ridge parameters can only improve
the value, which proves the claim.
\end{proof}

In this two-level cost model, scalar ridge is exact on any support drawn from a single cost class, even when that class is highly manipulable.
The obstacle is therefore not high manipulability by itself, but combining features with very different manipulation costs.
On a mixed support, the oracle would shrink the two cost classes at different rates, whereas ordinary ridge has only one scalar shrinkage parameter.
As a result, the heterogeneity gap grows with the separation $\kappa-\mu$.
This formalizes the policy insight in \Cref{fig:homogeneity-phase-main}: predictive manipulable groups can be retained when their manipulation costs are internally homogeneous, whereas feature selection becomes more important when retained costs are heterogeneous.

\subsection{Additional properties of ridge tuning}
\label{app:ridge-tuning-properties}

This subsection records two elementary facts about ridge regularization in the strategic setting.
They are not needed for our main approximation theorem, but they support the qualitative message that regularization directly reduces strategic $\MSE$.

Fix a support $\mathsf S$.
Let
\[
\theta_{\mathsf S}^{\mathrm{LS}}
:=
\theta_{\mathsf S}^{\mathrm{LS}}(\mathsf S),
\qquad
\theta_{\mathsf S}^{\mathrm{OR}}(\lambda)
=
(\Sigma_{\mathsf{SS}}+\lambda I)^{-1}
\Sigma_{\mathsf{SS}}\theta_{\mathsf S}^{\mathrm{LS}}.
\]
Ignoring constants independent of $\lambda$, the strategic $\MSE$ along the support-restricted ridge path is
\[
R_{\mathsf S}(\lambda)
=
(\theta_{\mathsf S}^{\mathrm{OR}}(\lambda)-\theta_{\mathsf S}^{\mathrm{LS}})^\top
\Sigma_{\mathsf{SS}}
(\theta_{\mathsf S}^{\mathrm{OR}}(\lambda)-\theta_{\mathsf S}^{\mathrm{LS}})
+
\left(
\theta_{\mathsf S}^{\mathrm{OR}}(\lambda)^\top
H^{-1}_{\mathsf{SS}}
\theta_{\mathsf S}^{\mathrm{OR}}(\lambda)
\right)^2.
\]

\begin{proposition}[Infinitesimal ridge improves strategic $\MSE$]
\label{prop:infinitesimal-ridge-improves}
If
$
\theta_{\mathsf S}^{\mathrm{LS}\top}
H^{-1}_{\mathsf{SS}}
\Sigma_{\mathsf{SS}}^{-1}
\theta_{\mathsf S}^{\mathrm{LS}}
>
0
$,
then there exists $\lambda_0>0$ such that for all $\lambda\in(0,\lambda_0)$,
\[
\mathrm{MSE}_{\mathrm{strat}}(\theta^{\mathrm{OR}}(\mathsf S,\lambda))
<
\mathrm{MSE}_{\mathrm{strat}}(\theta^{\mathrm{OR}}(\mathsf S,0)).
\]
\end{proposition}

\begin{proof}
The derivative of the ridge path is
\[
\frac{d}{d\lambda}\theta_{\mathsf S}^{\mathrm{OR}}(\lambda)
=
-(\Sigma_{\mathsf{SS}}+\lambda I)^{-1}
\theta_{\mathsf S}^{\mathrm{OR}}(\lambda).
\]
Thus
\[
\frac{d}{d\lambda}\theta_{\mathsf S}^{\mathrm{OR}}(0)
=
-\Sigma_{\mathsf{SS}}^{-1}
\theta_{\mathsf S}^{\mathrm{LS}}.
\]
Write
$
G(\lambda)
:=
(\theta_{\mathsf S}^{\mathrm{OR}}(\lambda)-\theta_{\mathsf S}^{\mathrm{LS}})^\top
\Sigma_{\mathsf{SS}}
(\theta_{\mathsf S}^{\mathrm{OR}}(\lambda)-\theta_{\mathsf S}^{\mathrm{LS}})
$
and
$
Q(\lambda)
:=
\left(
\theta_{\mathsf S}^{\mathrm{OR}}(\lambda)^\top
H^{-1}_{\mathsf{SS}}
\theta_{\mathsf S}^{\mathrm{OR}}(\lambda)
\right)^2
$.
Since $\theta_{\mathsf S}^{\mathrm{OR}}(0)=\theta_{\mathsf S}^{\mathrm{LS}}$, we have
$
G'(0)=0.
$
Next, define
$
q(\lambda)
:=
\theta_{\mathsf S}^{\mathrm{OR}}(\lambda)^\top
H^{-1}_{\mathsf{SS}}
\theta_{\mathsf S}^{\mathrm{OR}}(\lambda)
$.
Then
$
Q'(\lambda)=2q(\lambda)q'(\lambda)
$,
and
$
q'(0)
=
-2\theta_{\mathsf S}^{\mathrm{LS}\top}
H^{-1}_{\mathsf{SS}}
\Sigma_{\mathsf{SS}}^{-1}
\theta_{\mathsf S}^{\mathrm{LS}}
$.
Therefore
$
Q'(0)
=
-4
(
\theta_{\mathsf S}^{\mathrm{LS}\top}
H^{-1}_{\mathsf{SS}}
\theta_{\mathsf S}^{\mathrm{LS}}
)
(
\theta_{\mathsf S}^{\mathrm{LS}\top}
H^{-1}_{\mathsf{SS}}
\Sigma_{\mathsf{SS}}^{-1}
\theta_{\mathsf S}^{\mathrm{LS}}
)
$.

Since $H^{-1}_{\mathsf{SS}}\succ0$, the assumed condition implies $\theta_{\mathsf S}^{\mathrm{LS}}\neq0$, and hence
$
\theta_{\mathsf S}^{\mathrm{LS}\top}
H^{-1}_{\mathsf{SS}}
\theta_{\mathsf S}^{\mathrm{LS}}
>
0
$.
Thus
$
R_{\mathsf S}'(0)=G'(0)+Q'(0)<0
$.
By continuity, there exists $\lambda_0>0$ such that
$
R_{\mathsf S}(\lambda)<R_{\mathsf S}(0)
$
for every $\lambda\in(0,\lambda_0)$.
Since $R_{\mathsf S}(\lambda)$ differs from
$\mathrm{MSE}_{\mathrm{strat}}(\theta^{\mathrm{OR}}(\mathsf S,\lambda))$
only by constants independent of $\lambda$, the claim follows.
\end{proof}

\begin{proposition}[Scaling inverse-cost intensity increases the locally optimal ridge level]
\label{prop:scaling-inverse-cost-increases-lambda}
Fix a support $\mathsf S$.
For $\tau>0$, consider scaling the inverse-cost matrix from $H^{-1}_{\mathsf{SS}}$ to $\tau H^{-1}_{\mathsf{SS}}$.
Ignoring constants independent of $\lambda$, define the strategic $\MSE$ along the support-restricted ridge path:
\[
R_{\mathsf S,\tau}(\lambda)
=
G(\lambda)+\tau^2 Q(\lambda),
\]
where $G(\lambda)$ and $Q(\lambda)$ are defined in the proof of \Cref{prop:infinitesimal-ridge-improves}.
Suppose that for each $\tau$ in an open interval $I\subset(0,\infty)$, $R_{\mathsf S,\tau}$ admits a finite interior minimizer
$
\lambda^*(\tau)\in(0,\infty)
$,
and that this minimizer is strict, i.e.,
\[
\frac{\partial R_{\mathsf S,\tau}}{\partial\lambda}(\lambda^*(\tau))=0,
\qquad
\frac{\partial^2 R_{\mathsf S,\tau}}{\partial\lambda^2}(\lambda^*(\tau))>0.
\]
Then $\lambda^*(\tau)$ is differentiable on $I$, and
\[
\frac{d}{d\tau}\lambda^*(\tau)>0.
\]
\end{proposition}
\begin{proof}
Let
\[
\mathcal F(\lambda,\tau)
:=
\frac{\partial R_{\mathsf S,\tau}}{\partial\lambda}(\lambda)
=
G'(\lambda)+\tau^2 Q'(\lambda).
\]
At $\lambda=\lambda^*(\tau)$,
$
\mathcal F(\lambda^*(\tau),\tau)=0
$.
Moreover,
\[
\frac{\partial \mathcal F}{\partial\lambda}(\lambda^*(\tau),\tau)
=
\frac{\partial^2 R_{\mathsf S,\tau}}{\partial\lambda^2}(\lambda^*(\tau))
>
0.
\]
By the implicit function theorem, $\lambda^*(\tau)$ is differentiable and
\[
\frac{d}{d\tau}\lambda^*(\tau)
=
-
\frac{\partial \mathcal F/\partial\tau}{\partial \mathcal F/\partial\lambda}
=
-
\frac{2\tau Q'(\lambda^*(\tau))}
{R_{\mathsf S,\tau}''(\lambda^*(\tau))}.
\]
It remains to show that
$
Q'(\lambda^*(\tau))<0
$.

At the optimum,
$
G'(\lambda^*(\tau))+\tau^2Q'(\lambda^*(\tau))=0
$.
We now show that $G'(\lambda)>0$ for every $\lambda>0$ in the nontrivial case.
Using
\[
(\Sigma_{\mathsf{SS}}+\lambda I)
\theta_{\mathsf S}^{\mathrm{OR}}(\lambda)
=
\Sigma_{\mathsf{SS}}
\theta_{\mathsf S}^{\mathrm{LS}},
\]
we have
\[
\theta_{\mathsf S}^{\mathrm{OR}}(\lambda)-\theta_{\mathsf S}^{\mathrm{LS}}
=
-\lambda\Sigma_{\mathsf{SS}}^{-1}
\theta_{\mathsf S}^{\mathrm{OR}}(\lambda).
\]
Also,
\[
\frac{d}{d\lambda}\theta_{\mathsf S}^{\mathrm{OR}}(\lambda)
=
-(\Sigma_{\mathsf{SS}}+\lambda I)^{-1}
\theta_{\mathsf S}^{\mathrm{OR}}(\lambda).
\]
Hence
\begin{align*}
G'(\lambda)
=
2
(\theta_{\mathsf S}^{\mathrm{OR}}(\lambda)-\theta_{\mathsf S}^{\mathrm{LS}})^\top
\Sigma_{\mathsf{SS}}
\frac{d}{d\lambda}\theta_{\mathsf S}^{\mathrm{OR}}(\lambda)
=
2\lambda\,
\theta_{\mathsf S}^{\mathrm{OR}}(\lambda)^\top
(\Sigma_{\mathsf{SS}}+\lambda I)^{-1}
\theta_{\mathsf S}^{\mathrm{OR}}(\lambda).
\end{align*}
For $\lambda>0$ and $\theta_{\mathsf S}^{\mathrm{OR}}(\lambda)\neq0$, this is strictly positive.
Under the stated strict interior-minimizer assumption, the nontrivial case has
$\theta_{\mathsf S}^{\mathrm{OR}}(\lambda^*(\tau))\neq0$.
Thus
$
G'(\lambda^*(\tau))>0
$.
The first-order condition then implies
\[
Q'(\lambda^*(\tau))
=
-\frac{G'(\lambda^*(\tau))}{\tau^2}
<
0.
\]
Therefore
\[
\frac{d}{d\tau}\lambda^*(\tau)
=
-
\frac{2\tau Q'(\lambda^*(\tau))}
{R_{\mathsf S,\tau}''(\lambda^*(\tau))}
>
0,
\]
because $\tau>0$, $Q'(\lambda^*(\tau))<0$, and
$R_{\mathsf S,\tau}''(\lambda^*(\tau))>0$.
\end{proof}

\section{Simulation Details and Illustrations for \Cref{sec:implications}}

\subsection{Details for \Cref{fig:decomposition-tradeoff}}
\label{sec:fig:decomposition-tradeoff-details}

We consider the following $d = 4$ example:
\[
\Sigma=
\begin{pmatrix}
1 & 0.9 & 0.2 & 0.1\\
0.9 & 1 & 0.2 & 0.1\\
0.2 & 0.2 & 1 & 0.5\\
0.1 & 0.1 & 0.5 & 1
\end{pmatrix},
\qquad
\theta^*=(1.6,1.2,1.0,0.7), 
\qquad H^{-1}=\diag(8.0,1.5,0.8,0.5).
\]
Feature $1$ is highly predictive but also highly manipulable.
Feature $2$ is strongly correlated with feature $1$, and therefore acts as a less manipulable proxy.
Features $3$ and $4$ are less manipulable, moderately predictive features.

We enumerate all $2^4=16$ supports $\mathsf S\subseteq[4]$.
For each support, we compute the support-restricted zero-intercept oracle value
\[
\widetilde{\OPT}(\mathsf S)
=
\min_{\theta\in\Theta(\mathsf S)}
\mathrm{MSE}_{\mathrm{strat}}(\theta,0),
\]
and the best support-restricted ridge value
\[
\min_{\lambda\ge0}\,
\mathrm{MSE}_{\mathrm{strat}}
\bigl(\theta^{\mathrm{OR}}(\mathsf S,\lambda),0\bigr)
\]
by one-dimensional optimization over $\lambda$.
The plotted quantities are predictive loss $L_{\mathrm{pred}}(\mathsf S)$, negative manipulability gain
$-(\Gamma([d]) - \Gamma(\mathsf S))$, the heterogeneity gap $C_H(\mathsf S)\delta_H(\mathsf S)^2$, and total excess, given by the sum of these three quantities.
The support restriction identity in \Cref{prop:support-restriction} gives the support restriction cost $\widetilde{\OPT}(\mathsf S)-\widetilde{\OPT} = L_{\mathrm{pred}}(\mathsf S) - (\Gamma([d])-\Gamma(\mathsf S))$.

The left panel illustrates five representative supports.
The full support $\{1,2,3,4\}$ has zero support restriction cost but a large heterogeneity gap because it combines features with very different manipulation costs.
The support $\{2,3,4\}$ drops the most manipulable feature while retaining its less manipulable proxy and the remaining less manipulable predictive features; it has a larger support restriction cost than the full support but a much smaller heterogeneity gap, and is best overall.
The support $\{2,3\}$ retains the proxy and one less manipulable feature but loses additional predictive signal.
The singleton support $\{2\}$ retains only the proxy and therefore has zero heterogeneity gap but incurs substantial predictive loss.
The pair $\{1, 2\}$ retains the manipulable feature and its proxy, and hence incurs smaller predictive loss than the singleton support, but has a large heterogeneity gap.

The right panel plots every support by support restriction cost and heterogeneity gap.
The light labels show all supports, while the five representative supports are highlighted.
Because $\Theta(\mathsf S)\subseteq\Theta(\mathsf T)$ whenever $\mathsf S\subseteq\mathsf T$, support restriction cost is monotone with respect to set inclusion: enlarging the support cannot increase support restriction cost.
This monotonicity is visible in the placement of the full support near the left edge and smaller supports farther to the right.
By contrast, the heterogeneity gap is not monotone in support size: adding features can either improve or worsen the heterogeneity gap, depending on whether the added features make the retained inverse-cost geometry more heterogeneous.

\subsection{Details for \Cref{fig:homogeneity-phase-main}}
\label{sec:fig:homogeneity-phase-main-details}
We consider the following $d=4$ example:
\[
\Sigma = I_4,
\qquad
\theta^*=(a,a,1.0,0.9),
\qquad
H^{-1}=\diag\bigl(\kappa(1+\delta),\,\kappa(1-\delta),\,0.6,\,0.6\bigr),
\]
where we fix $a=1.2$.
Thus features $1$ and $2$ form a manipulable pair with average manipulability $\kappa$ and relative cost heterogeneity $\delta$, while features $3$ and $4$ are less manipulable baseline features.
We vary $(\kappa,\delta)$ on a $51\times 51$ grid, with $\kappa\in[0.6,7.5]$ and $\delta\in[0,0.95]$.

For each grid point $(\kappa,\delta)$, we enumerate all $2^4$ supports $\mathsf{S} \subseteq[4]$.
For each support, we compute the exact support-restricted ridge optimum
\[
\min_{\lambda\ge 0}\,
\mathrm{MSE}_{\mathrm{strat}}\bigl(\theta^{\mathrm{OR}}(\mathsf S,\lambda), 0\bigr),
\]
and the exact support-restricted zero-intercept oracle
\[
\widetilde{\mathrm{OPT}}(\mathsf S)
:=
\min_{\theta \in \Theta(\mathsf S)}
\mathrm{MSE}_{\mathrm{strat}}(\theta,0).
\]
We also compute the full-support zero-intercept oracle
\[
\widetilde{\mathrm{OPT}}
:=
\min_{\theta\in\mathbb{R}^4}\,
\mathrm{MSE}_{\mathrm{strat}}(\theta,0).
\]
These are computed using the generalized-ridge fixed-point characterization (see \Cref{prop:gr-optimum} for more details):
\begin{equation}
\label{eq:gr-fixed-point}
\theta(q)=(\Sigma+2qH^{-1})^{-1}\Sigma\theta^*,
\qquad
q=\theta(q)^\top H^{-1}\theta(q).
\end{equation}

The support-restricted ridge estimator is then chosen by joint optimization over the support and ridge level:
\[
(\mathsf{S}^{\star},\lambda^{\star})
\in
\argmin_{\mathsf{S}\subseteq[4],\,\lambda\ge 0}
\mathrm{MSE}_{\mathrm{strat}}\bigl(\theta^{\mathrm{OR}}(\mathsf{S},\lambda), 0\bigr).
\]

The left, middle, and right panels show, respectively, the optimal support $\mathsf{S}^{\star}(\kappa,\delta)$, the corresponding optimal ridge level $\lambda^{\star}(\kappa,\delta)$, and the excess strategic $\MSE$ relative to $\widetilde{\OPT}$: $\mathrm{MSE}_{\mathrm{strat}}\bigl(\theta^{\mathrm{OR}}(\mathsf{S}^{\star},\lambda^{\star}), 0\bigr)-\widetilde{\mathrm{OPT}}.$

Only three supports appear on the grid: $\{1,2,3,4\}$, $\{2,3,4\}$, and $\{3,4\}$.
The transition from $\{1,2,3,4\}$ to $\{2,3,4\}$ shows that once the manipulable pair becomes sufficiently heterogeneous, the optimizer drops the more manipulable coordinate while retaining the less manipulable one.
The transition from $\{2,3,4\}$ to $\{3,4\}$ shows that when average manipulability becomes large enough, the entire manipulable pair is excluded and only the less manipulable baseline features are retained.

\subsection{Details for \Cref{fig:supp-correlated-proxy}}
\label{sec:supp-correlated-proxy}

We consider the following $d=4$ example:
\[
\Sigma(\rho)=
\begin{pmatrix}
1 & \rho & 0.1 & 0\\
\rho & 1 & 0 & 0.1\\
0.1 & 0 & 1 & 0.2\\
0 & 0.1 & 0.2 & 1
\end{pmatrix},\qquad \theta^*=(1.6,0,1.0,0.9),
\qquad
H^{-1}=\mathrm{diag}(k_1,0.6,0.5,0.5).
\]
Feature $1$ is manipulable and directly predictive, feature $2$ is a correlated less manipulable proxy whose coefficient in $\theta^*$ is zero, and features $3$ and $4$ are less manipulable baseline features.
We vary $\rho\in[0,0.95]$ and $k_1\in[0.5,4.0]$ on a $61\times61$ grid.

For each grid point $(\rho,k_1)$, we enumerate all $2^4$ supports $\mathsf{S} \subseteq[4]$. 
For each support, we compute the exact support-restricted ridge optimum
\[
\min_{\lambda\ge 0}\,
\mathrm{MSE}_{\mathrm{strat}}\bigl(\theta^{\mathrm{OR}}(\mathsf S,\lambda), 0\bigr).
\]
We then record the minimizing support $\mathsf{S}^*(\rho,k_1)$ and ridge level $\lambda^*(\rho,k_1)$. 
The full-support zero-intercept oracle
\[
\widetilde{\mathrm{OPT}}
=
\min_{\theta\in\mathbb{R}^4}
\mathrm{MSE}_{\mathrm{strat}}(\theta,0)
\]
is computed using the generalized-ridge fixed-point equation \eqref{eq:gr-fixed-point} described above.

The left, middle, and right panels show, respectively, the optimal support $\mathsf{S}^{\star}(\rho,k_1)$, the corresponding optimal ridge level $\lambda^{\star}(\rho,k_1)$, and the excess strategic $\MSE$ relative to $\widetilde{\OPT}$: $\mathrm{MSE}_{\mathrm{strat}}\bigl(\theta^{\mathrm{OR}}(\mathsf{S}^{\star},\lambda^{\star}), 0\bigr)-\widetilde{\mathrm{OPT}}.$

The support map in the left panel shows that only two supports appear on the grid: $\{1,2,3,4\}$ and $\{2,3,4\}$.
The phase transition is therefore especially transparent: as the proxy becomes more correlated with the manipulable feature, and as that feature becomes more manipulable, the optimal design drops feature $1$ while retaining feature $2$.
The middle and right panels show that this support transition is accompanied by a systematic change in the preferred regularization level, while the resulting support-restricted ridge estimator remains close to the oracle benchmark over a broad region of the parameter space.

\subsection{Predictiveness versus manipulability with homogeneous costs}
\label{sec:supp-signal-manip}

\begin{figure}[!t]
    \centering
    \includegraphics[width=0.99\linewidth]{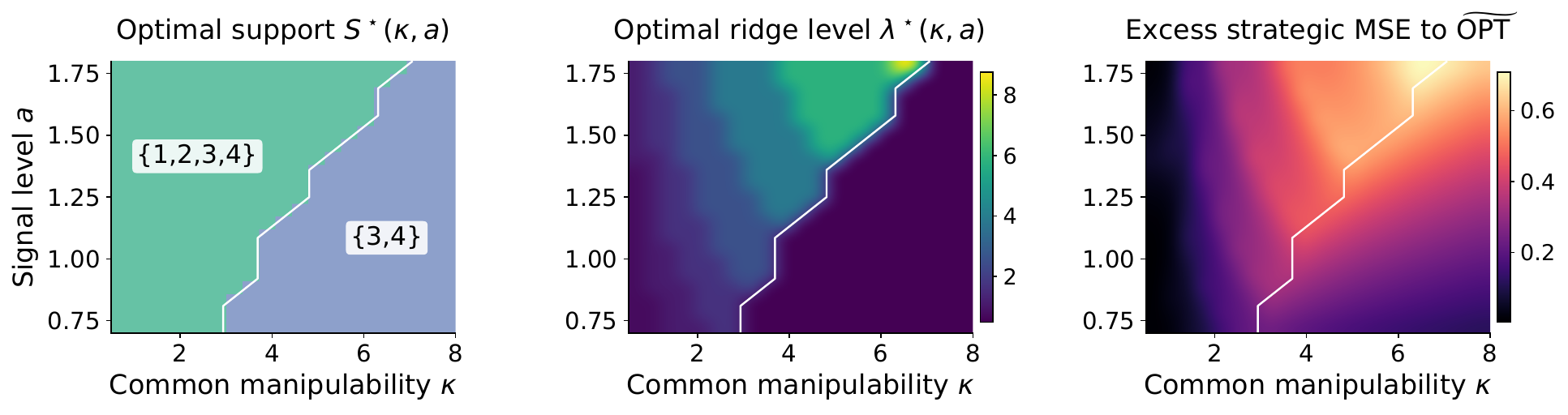}
    \caption{
    \textbf{A homogeneous manipulable pair can be retained and regularized when sufficiently predictive.}
    We compute the optimal support and regularization level for $d=4$ features with diagonal costs.
    Features $\{1,2\}$ are a homogeneous manipulable pair with common signal level $a$ and common manipulability $\kappa$;  features $3$ and $4$ have low manipulability.
    (Left) The optimal support $\mathsf{S}^*(\kappa,a)$ after joint optimization over $\mathsf S$ and $\lambda$.
    (Middle) The corresponding optimal ridge level $\lambda^*(\kappa,a)$, where white contour indicates the support boundary.
    (Right) Excess strategic $\MSE$ relative to $\widetilde{\mathrm{OPT}}$.
    As the manipulable pair becomes more predictive, it is retained rather than excluded; as manipulability increases, the optimal response is to increase regularization.
    The resulting support-restricted ridge estimator remains close to $\widetilde{\mathrm{OPT}}$ over a broad range of the parameter space.
    }
    \label{fig:supp-signal-manip}
\end{figure}

We consider the following $d=4$ example in \Cref{fig:supp-signal-manip}:
\[
\Sigma = I_4,
\qquad
\theta^*=(a,a,1.0,0.9),
\qquad
H^{-1}=\diag(\kappa,\kappa,0.6,0.6).
\]
Features $1$ and $2$ form a homogeneous manipulable pair with common signal level $a$ and common manipulability $\kappa$, while features $3$ and $4$ are less manipulable baseline features.
We vary $a\in[0.7,1.8]$ and $\kappa\in[0.5,8]$ on a $21\times 21$ grid.

For each grid point ($a, \kappa$), we compute the optimal support-restricted ridge choice $(\mathsf{S}^*,\lambda^*)$ and evaluate the resulting strategic $\MSE$.
To assess near-optimality, we compute the full-support zero-intercept oracle
\[
\widetilde{\mathrm{OPT}}=\min_{\theta\in\mathbb{R}^4} \mathrm{MSE}_{\mathrm{strat}}(\theta, 0).
\]
Because $\Sigma$ and $H^{-1}$ are diagonal here, the oracle can be computed from the scalar fixed-point equation \eqref{eq:gr-fixed-point}:
\[
q = \sum_{j=1}^4 \frac{H^{-1}_{jj} (\theta_j^*)^2}{(1+2q H^{-1}_{jj})^2},
\qquad
\theta_j^{\star,\mathrm{oracle}} = \frac{\theta_j^*}{1+2q H^{-1}_{jj}}.
\]
The left, middle, and right panels show, respectively, the optimal support $\mathsf{S}^*(\kappa,a)$, the corresponding optimal ridge level $\lambda^*(\kappa,a)$, and the excess strategic $\MSE$ relative to $\widetilde{\OPT}$: $\mathrm{MSE}_{\mathrm{strat}}\bigl(\theta^{\mathrm{OR}}(\mathsf{S}^*,\lambda^*), 0\bigr)-\widetilde{\mathrm{OPT}}.$

Only two supports appear on the grid: the full support $\{1,2,3,4\}$ and the baseline support $\{3,4\}$.
The support switch therefore cleanly illustrates the message: a homogeneous manipulable pair is retained when it is sufficiently predictive and otherwise excluded; conditional on retaining it, the preferred regularization strength increases with manipulability. The resulting support-restricted ridge estimator remains close to $\widetilde{\mathrm{OPT}}$ over a broad region of the parameter space, showing that coarse levers are often near-optimal in this setting.

\section{Theoretical Details and Proofs for \Cref{sec:algorithm}}
\label{sec:app-algorithm}

As Appendix~\ref{app:sec:theory}, throughout this appendix, we use the shorthand 
\[
\mathrm{MSE}_{\mathrm{strat}}(\theta)
=
\mathrm{MSE}_{\mathrm{strat}}(\theta,0)
=
(\theta-\theta^*)^\top\Sigma(\theta-\theta^*)
+
(\theta^\top H^{-1}\theta)^2
+
\sigma^2.
\]
For fixed $\lambda\ge0$, we also write
\[
\Phi_\lambda(\mathsf S)
:=
\Phi(\mathsf S,\lambda)
=
\mathrm{MSE}_{\mathrm{strat}}
\bigl(\theta^{\mathrm{OR}}(\mathsf S,\lambda)\bigr).
\]
Recall that for a support budget $s$, the exact fixed-$\lambda$ support-restricted ridge optimum is
\[
\Phi_{\lambda,s}^{\star}
:=
\min_{\substack{\mathsf S\subseteq[d]\\|\mathsf S|\le s}}
\Phi(\mathsf S,\lambda).
\]

\subsection{Proof of \Cref{prop:binary-equivalence}}

Let $w\in\{0,1\}^d$ and define $\mathsf S(w) = \{i\in[d]:w_i=1\}$.
By the definition of the perspective function $\varphi$, if $w_i=0$, then $\varphi(\theta_i,w_i) = +\infty$ unless $\theta_i=0$.
Thus every finite-valued feasible vector in the weighted ridge problem must satisfy $\supp(\theta)\subseteq\mathsf S(w)$.
On coordinates $i\in\mathsf S(w)$, we have $w_i=1$, so $\varphi(\theta_i,w_i)=\theta_i^2$.
Therefore the weighted ridge problem becomes:
\[
\argmin_{\theta\in\RR^d:\supp(\theta)\subseteq\mathsf S(w)}
\left\{
\E[(Y-\theta^\top X)^2]
+
\lambda\|\theta\|_2^2
\right\}.
\]
This is exactly the support-restricted ordinary ridge problem defining $\theta^{\mathrm{OR}}(\mathsf S(w),\lambda)$.
Hence, $\theta_\lambda(w)=\theta^{\mathrm{OR}}(\mathsf S(w),\lambda)$.
Evaluating the same exact strategic MSE on both sides gives
$\mathrm{MSE}_{\mathrm{strat}}(\theta_\lambda(w)) = \Phi(\mathsf S(w),\lambda)$.
Finally, the constraint $\mathbf 1^\top w\le s$ is equivalent to $|\mathsf S(w)|\le s$ for binary $w$, so
\[
\Phi_{\lambda,s}^\star
=
\min_{\substack{\mathsf S\subseteq[d]\\|\mathsf S|\le s}}
\Phi(\mathsf S,\lambda)
=
\min_{\substack{w\in\{0,1\}^d\\\mathbf 1^\top w\le s}}
\mathrm{MSE}_{\mathrm{strat}}(\theta_\lambda(w)).
\]

\subsection{Weighted relaxation as a lower bound}

For $\lambda>0$, define the weight polytope
$
\mathcal W_s
:=
\{w\in[0,1]^d:\mathbf 1^\top w\le s\}.
$
The weighted relaxation at fixed $\lambda$ is $L^{\mathrm{wt}}_{\lambda,s}:=\min_{w\in\mathcal W_s}
\MSEstrat(\theta_\lambda(w))$.
Because the binary feasible set is contained in $\mathcal W_s$, this relaxation gives a lower bound on the exact fixed-$\lambda$ support-restricted ridge optimum.

\begin{proposition}[Weighted relaxation lower bound]
\label{prop:weighted-relaxation-lower-bound}
For every $\lambda>0$ and support budget $s$, $L_{\lambda,s}^{\mathrm{wt}} \le \Phi_{\lambda,s}^\star$.
\end{proposition}
\begin{proof}
By \Cref{prop:binary-equivalence}, $\Phi_{\lambda,s}^\star = \min_{\substack{w\in\{0,1\}^d\\\mathbf 1^\top w\le s}} \mathrm{MSE}_{\mathrm{strat}}(\theta_\lambda(w))$.
Since $\{w\in\{0,1\}^d:\mathbf 1^\top w\le s\} \subseteq \mathcal W_s$, minimizing the same objective over $\mathcal W_s$ can only decrease the value:
\[
L_{\lambda,s}^{\mathrm{wt}}
=
\min_{w\in\mathcal W_s}
\mathrm{MSE}_{\mathrm{strat}}(\theta_\lambda(w))
\le
\Phi_{\lambda,s}^\star.
\]
\end{proof}

This also yields an a posteriori certificate whenever the weighted relaxation is solved globally, or whenever a certified lower bound on it is available.

\begin{corollary}[A posteriori certificate at fixed $\lambda$]
\label{cor:fixed-lambda-certificate}
Fix $\lambda>0$ and suppose $\widehat{\mathsf S}_\lambda$ is any support with $|\widehat{\mathsf S}_\lambda|\le s$.
Then
\[
0
\le
\Phi(\widehat{\mathsf S}_\lambda,\lambda)-\Phi_{\lambda,s}^\star
\le
\Phi(\widehat{\mathsf S}_\lambda,\lambda)-L_{\lambda,s}^{\mathrm{wt}}.
\]
More generally, the same bound holds if $L_{\lambda,s}^{\mathrm{wt}}$ is replaced by any certified lower bound on $\Phi_{\lambda,s}^\star$.
\end{corollary}
\begin{proof}
Since $\Phi_{\lambda,s}^\star$ is the exact optimum over all supports of size at most $s$,
\[
0
\le
\Phi(\widehat{\mathsf S}_\lambda,\lambda)-\Phi_{\lambda,s}^\star.
\]
By \Cref{prop:weighted-relaxation-lower-bound},
$
L_{\lambda,s}^{\mathrm{wt}}\le \Phi_{\lambda,s}^\star
$.
Therefore
\[
\Phi(\widehat{\mathsf S}_\lambda,\lambda)-\Phi_{\lambda,s}^\star
\le
\Phi(\widehat{\mathsf S}_\lambda,\lambda)-L_{\lambda,s}^{\mathrm{wt}}.
\]
The same argument applies to any lower bound on $\Phi_{\lambda,s}^\star$.
\end{proof}

\subsection{Local support refinement}

The weighted relaxation produces a fractional vector, which must be converted into a discrete support.
For $w\in[0,1]^d$, let $\operatorname{Top}_s(w)$
denote the indices of the $s$ largest coordinates of $w$, with ties broken arbitrarily.
After rounding, we refine the support using the exact fixed-$\lambda$ objective $\Phi(\mathsf S,\lambda)$.

For a current support $\mathsf S$ with $|\mathsf S|\le s$, define the add/drop/swap neighborhood
\[
\mathcal N_s(\mathsf S)
:=
\bigl\{
\mathsf S\cup\{j\}: j\notin\mathsf S,\ |\mathsf S|<s
\bigr\}
\cup
\bigl\{
\mathsf S\setminus\{i\}: i\in\mathsf S
\bigr\}
\cup
\bigl\{
(\mathsf S\setminus\{i\})\cup\{j\}: i\in\mathsf S,\ j\notin\mathsf S
\bigr\}.
\]
Starting from an initial support $\mathsf S^{(0)}$, the refinement step repeatedly replaces the current support by the best strictly improving support in
$\mathcal N_s(\mathsf S)\cup\{\mathsf S\}$.

\begin{proposition}[Termination and monotonicity of local refinement]
\label{prop:local-refinement-termination}
Fix $\lambda\ge0$ and support budget $s$.
Starting from any support $\mathsf S^{(0)}$ with $|\mathsf S^{(0)}|\le s$, the add/drop/swap local refinement procedure terminates after finitely many accepted moves.
If $\widehat{\mathsf S}_\lambda$ is its final output, then
$
\Phi(\widehat{\mathsf S}_\lambda,\lambda)
\le
\Phi(\mathsf S^{(0)},\lambda)
$.
Moreover, $\widehat{\mathsf S}_\lambda$ is locally optimal with respect to $\mathcal N_s(\cdot)$:
$
\Phi(\widehat{\mathsf S}_\lambda,\lambda)
\le
\Phi(\mathsf T,\lambda)
$
for every
$\mathsf T\in\mathcal N_s(\widehat{\mathsf S}_\lambda)$.
\end{proposition}

\begin{proof}
Each accepted move replaces the current support $\mathsf S$ by a support $\mathsf T$ satisfying
$
\Phi(\mathsf T,\lambda)<\Phi(\mathsf S,\lambda)
$.
Thus the exact fixed-$\lambda$ objective strictly decreases at every accepted move.
Because there are only finitely many supports of size at most $s$, the procedure must terminate after finitely many accepted moves.

The monotonicity statement follows by chaining the strict improvements along the accepted moves:
$
\Phi(\widehat{\mathsf S}_\lambda,\lambda)
\le
\Phi(\mathsf S^{(0)},\lambda)
$.
At termination, there is no neighbor
$\mathsf T\in\mathcal N_s(\widehat{\mathsf S}_\lambda)$
with
$
\Phi(\mathsf T,\lambda)<\Phi(\widehat{\mathsf S}_\lambda,\lambda)
$,
which is precisely the stated local optimality condition.
\end{proof}

\subsection{A convex reformulation under diagonal covariance}

The weighted relaxation is generally nonconvex.
However, it becomes convex after a change of variables when the feature covariance is diagonal.
Assume throughout this section that the feature covariance is diagonal: $\Sigma = \diag(v_1,\dots,v_d)$ with $v_i > 0$ for all $i \in [d]$.
Write $\theta^* = (\theta_1^*,\dots,\theta_d^*)$.
Fix $\lambda > 0$ and recall that $\mathcal W_s = \{w\in[0,1]^d:\ \mathbf 1^\top w\le s\}$.

Under diagonal covariance, the weighted ridge coefficients decouple coordinatewise.
Indeed, for any $w\in\mathcal W_s$,
\begin{equation}
\label{eq:diag-weighted-theta}
\theta_{\lambda,i}(w)
=
\frac{v_i w_i}{v_i w_i+\lambda}\,\theta_i^*,
\qquad
i=1,\dots,d.
\end{equation}
This motivates the reparameterization
\begin{equation}
\label{eq:r-reparam}
r_i
:=
\frac{v_i w_i}{v_i w_i+\lambda},
\qquad
i=1,\dots,d.
\end{equation}
Let $r = (r_1,\dots,r_d)$.
The variables $r_i$ have a simple interpretation.
They are exactly the fraction of the original signal coefficients that survive after weighted shrinkage, since \eqref{eq:diag-weighted-theta} becomes $\theta_{\lambda,i}(w)=\theta_i^*\,r_i$.

The change of variables in \eqref{eq:r-reparam} is one-to-one for $\lambda>0$, with inverse
\begin{equation}
\label{eq:w-from-r}
w_i
=
\frac{\lambda r_i}{v_i(1-r_i)},
\qquad
0\le r_i\le \frac{v_i}{v_i+\lambda}.
\end{equation}
Under this reparameterization, the support-budget constraint $\mathbf 1^\top w\le s$ becomes
\begin{equation}
\label{eq:r-budget}
\sum_{i=1}^d
\frac{\lambda r_i}{v_i(1-r_i)}
\le s.
\end{equation}

The next proposition gives the exact weighted relaxation in these variables:
\begin{proposition}[Convex reformulation under diagonal covariance]
\label{prop:diag-cov-convex}
Under the diagonal covariance setup, the weighted relaxation is exactly equivalent to the optimization problem:
\begin{equation}
\label{eq:diag-cov-weighted-problem}
\begin{aligned}
\min_{r\in\mathbb R^d}\quad
&
\sum_{i=1}^d
v_i(\theta_i^*)^2(1-r_i)^2
+
\Bigl(
(\theta^*\odot r)^\top H^{-1}(\theta^*\odot r)
\Bigr)^2\\
\textnormal{subject to}\quad
&
0\le r_i\le \frac{v_i}{v_i+\lambda},
\qquad
i=1,\dots,d,\\
&
\sum_{i=1}^d
\frac{\lambda r_i}{v_i(1-r_i)}
\le s,
\end{aligned}
\end{equation}
where $\odot$ denotes coordinatewise multiplication.
Moreover, \eqref{eq:diag-cov-weighted-problem} is a convex optimization problem.
\end{proposition}

\Cref{prop:diag-cov-convex} is especially useful because it turns the inner weighted problem into a globally solvable convex program.
The first term in the objective is the prediction loss from shrinking the retained fraction of each signal coordinate away from its full value.
The second term is the strategic burden induced by the retained signal.
The budget constraint in \eqref{eq:r-budget} plays the role of a soft support-size constraint.

When $H$ is also diagonal, say $H = \diag(h_1,\dots,h_d)$ with $h_i > 0$ for $i \in [d]$, the strategic term simplifies further to $(\sum_{i=1}^d
h_i^{-1}(\theta_i^*)^2 r_i^2)^2$.
This is the form that appears in our exact diagonal examples.
It makes especially transparent how weighted shrinkage trades off predictive signal against exposure to manipulation.

There is also an exact discrete simplification in the diagonal case.
When both $\Sigma$ and $H^{-1}$ are diagonal, the fixed-$\lambda$ binary support-selection objective can be written as a modular predictive benefit minus a quadratic total-exposure penalty.
The resulting gain function is submodular, though not necessarily monotone; see \Cref{prop:diag-fixed-lambda-set-function,prop:diag-gain-submodular}.
This explains why greedy local search is a natural baseline, while also motivating the exact local-refinement step in \Cref{alg:weighted-relaxation-local-refinement}.

\begin{proof}
Before we begin, note that the statement in the main text omits the additive constant $\sigma^2$, which does not affect the minimizer.

Assume $\Sigma=\diag(v_1,\dots,v_d)$ with $v_i>0$.
For $w\in\mathcal W_s$, the weighted ridge problem decouples coordinatewise:
\[
\theta_{\lambda,i}(w)
=
\frac{v_i w_i}{v_iw_i+\lambda}\theta_i^*.
\]
Define
\[
r_i
=
\frac{v_iw_i}{v_iw_i+\lambda}.
\]
Then $\theta_{\lambda,i}(w)=\theta_i^*r_i$.
Since $w_i\in[0,1]$, we have
\[
0\le r_i\le \frac{v_i}{v_i+\lambda}.
\]
For $\lambda>0$, the map $w_i\mapsto r_i$ is one-to-one, with inverse
\[
w_i
=
\frac{\lambda r_i}{v_i(1-r_i)}.
\]
Thus the budget constraint $\mathbf 1^\top w\le s$ is equivalent to
\[
\sum_{i=1}^d
\frac{\lambda r_i}{v_i(1-r_i)}
\le s.
\]
Now, we compute the strategic MSE.
Since $\theta_\lambda(w)=\theta^*\odot r$, we have
\[
(\theta_\lambda(w)-\theta^*)^\top\Sigma(\theta_\lambda(w)-\theta^*)
=
\sum_{i=1}^d
v_i(\theta_i^*)^2(1-r_i)^2.
\]
Also,
\[
\theta_\lambda(w)^\top H^{-1}\theta_\lambda(w)
=
(\theta^*\odot r)^\top H^{-1}(\theta^*\odot r).
\]
Therefore
\[
\mathrm{MSE}_{\mathrm{strat}}(\theta_\lambda(w))
=
\sigma^2
+
\sum_{i=1}^d
v_i(\theta_i^*)^2(1-r_i)^2
+
\left(
(\theta^*\odot r)^\top H^{-1}(\theta^*\odot r)
\right)^2.
\]
Dropping the additive constant $\sigma^2$ gives the stated optimization problem.

It remains to prove convexity.
The first term
$
\sum_{i=1}^d
v_i(\theta_i^*)^2(1-r_i)^2
$
is convex in $r$.
The map
$
r
\mapsto
(\theta^*\odot r)^\top H^{-1}(\theta^*\odot r)
$
is a nonnegative convex quadratic function because $H^{-1} \succ 0$.
Hence the composition
$
r
\mapsto
\left(
(\theta^*\odot r)^\top H^{-1}(\theta^*\odot r)
\right)^2
$
is convex.

The box constraints are convex.
Finally, for each $i$, the function
\[
r_i
\mapsto
\frac{\lambda r_i}{v_i(1-r_i)}
\]
is convex on $0\le r_i<1$, since
\[
\frac{d^2}{dr_i^2}
\frac{\lambda r_i}{v_i(1-r_i)}
=
\frac{2\lambda}{v_i(1-r_i)^3}
>0.
\]
Because $r_i\le v_i/(v_i+\lambda)<1$, the budget constraint is convex.
Thus the entire problem is a convex optimization problem.
\end{proof}

\subsection{Exact strategic MSE of support-restricted ridge}

We record here an explicit expression for the exact strategic MSE of a support-restricted ridge estimator.
This expression is useful for implementation and for exact local search.

\begin{lemma}[Exact strategic MSE on a support]
\label{lem:subset-ridge-exact-risk}
Fix $\mathsf S\subseteq[d]$ and let $\mathsf R=[d]\setminus\mathsf S$.
Let
\[
\theta_{\mathsf S}^{\mathrm{LS}}(\mathsf S)
=
\Sigma_{\mathsf{SS}}^{-1}\Sigma_{\mathsf S y},
\qquad
\theta_{\mathsf S}^{\mathrm{OR}}(\lambda)
=
(\Sigma_{\mathsf{SS}}+\lambda I)^{-1}\Sigma_{\mathsf S y}.
\]
Then
\[
\Phi(\mathsf S,\lambda)
=
\sigma^2
+
L_{\mathrm{pred}}(\mathsf S)
+
\lambda^2
\theta_{\mathsf S}^{\mathrm{OR}}(\lambda)^\top
\Sigma_{\mathsf{SS}}^{-1}
\theta_{\mathsf S}^{\mathrm{OR}}(\lambda)
+
\left(
\theta_{\mathsf S}^{\mathrm{OR}}(\lambda)^\top
H^{-1}_{\mathsf{SS}}
\theta_{\mathsf S}^{\mathrm{OR}}(\lambda)
\right)^2.
\]
\end{lemma}
\begin{proof}
By \Cref{lem:support-pythagorean}, for every $\theta=(\eta,0_{\mathsf R})\in\Theta(\mathsf S)$,
\[
(\theta-\theta^*)^\top\Sigma(\theta-\theta^*)
=
L_{\mathrm{pred}}(\mathsf S)
+
(\eta-\theta_{\mathsf S}^{\mathrm{LS}}(\mathsf S))^\top
\Sigma_{\mathsf{SS}}
(\eta-\theta_{\mathsf S}^{\mathrm{LS}}(\mathsf S)).
\]
Substituting $\eta=\theta_{\mathsf S}^{\mathrm{OR}}(\lambda)$ gives
\[
\Phi(\mathsf S,\lambda)
=
\sigma^2
+
L_{\mathrm{pred}}(\mathsf S)
+
(\theta_{\mathsf S}^{\mathrm{OR}}(\lambda)-\theta_{\mathsf S}^{\mathrm{LS}}(\mathsf S))^\top
\Sigma_{\mathsf{SS}}
(\theta_{\mathsf S}^{\mathrm{OR}}(\lambda)-\theta_{\mathsf S}^{\mathrm{LS}}(\mathsf S))
+
\left(
\theta_{\mathsf S}^{\mathrm{OR}}(\lambda)^\top
H^{-1}_{\mathsf{SS}}
\theta_{\mathsf S}^{\mathrm{OR}}(\lambda)
\right)^2.
\]
The ridge normal equation is
\[
(\Sigma_{\mathsf{SS}}+\lambda I)
\theta_{\mathsf S}^{\mathrm{OR}}(\lambda)
=
\Sigma_{\mathsf{SS}}
\theta_{\mathsf S}^{\mathrm{LS}}(\mathsf S).
\]
Equivalently,
\[
\theta_{\mathsf S}^{\mathrm{OR}}(\lambda)-\theta_{\mathsf S}^{\mathrm{LS}}(\mathsf S)
=
-\lambda\Sigma_{\mathsf{SS}}^{-1}
\theta_{\mathsf S}^{\mathrm{OR}}(\lambda).
\]
Therefore
\[
(\theta_{\mathsf S}^{\mathrm{OR}}(\lambda)-\theta_{\mathsf S}^{\mathrm{LS}}(\mathsf S))^\top
\Sigma_{\mathsf{SS}}
(\theta_{\mathsf S}^{\mathrm{OR}}(\lambda)-\theta_{\mathsf S}^{\mathrm{LS}}(\mathsf S))
=
\lambda^2
\theta_{\mathsf S}^{\mathrm{OR}}(\lambda)^\top
\Sigma_{\mathsf{SS}}^{-1}
\theta_{\mathsf S}^{\mathrm{OR}}(\lambda).
\]
This proves the claim.
\end{proof}

\subsection{Closed form for the weighted ridge}

\begin{lemma}[Closed form for the weighted ridge estimator]
\label{lem:weighted-ridge-closed-form}
Suppose $w_i>0$ for all $i\in[d]$ and let $W=\diag(w)$.
Then $\theta_\lambda(w) = (\Sigma+\lambda W^{-1})^{-1}\Sigma\theta^*$.
\end{lemma}
\begin{proof}
The weighted ridge objective is $\E[(Y-\theta^\top X)^2] + \lambda\theta^\top W^{-1}\theta$.
Using $Y=\theta^{*\top}X+\eps$ and $\E[X\eps]=0$, the first-order condition is $2\Sigma(\theta-\theta^*)+2\lambda W^{-1}\theta=0$.
Thus $(\Sigma+\lambda W^{-1})\theta=\Sigma\theta^*$.
Since $\Sigma\succ0$ and $\lambda W^{-1}\succ0$, the matrix $\Sigma+\lambda W^{-1}$ is positive definite, and hence $\theta_\lambda(w) = (\Sigma+\lambda W^{-1})^{-1}\Sigma\theta^*$.
\end{proof}

\subsection{Exact set function in the diagonal case}

Assume
\[
\Sigma=\diag(v_1,\dots,v_d),
\qquad
H^{-1}=\diag(\omega_1,\dots,\omega_d),
\qquad
\theta^*=(\theta_1^*,\dots,\theta_d^*),
\]
where $v_i>0$ and $\omega_i>0$.
Fix $\lambda\ge0$.
For a subset $\mathsf S\subseteq[d]$,
\[
\theta_i^{\mathrm{OR}}(\mathsf S,\lambda)
=
\begin{cases}
\dfrac{v_i}{v_i+\lambda}\theta_i^*, & i\in\mathsf S,\\[0.8em]
0, & i\notin\mathsf S.
\end{cases}
\]

Define
\[
b_i(\lambda)
:=
\frac{v_i^2(v_i+2\lambda)}{(v_i+\lambda)^2}
(\theta_i^*)^2,
\qquad
c_i(\lambda)
:=
\omega_i
\left(\frac{v_i}{v_i+\lambda}\right)^2
(\theta_i^*)^2.
\]
Then define
\[
G_\lambda(\mathsf S)
:=
\sum_{i\in\mathsf S}b_i(\lambda)
-
\left(
\sum_{i\in\mathsf S}c_i(\lambda)
\right)^2.
\]

\begin{proposition}[Exact diagonal fixed-$\lambda$ objective]
\label{prop:diag-fixed-lambda-set-function}
Under the diagonal setup above,
\[
\Phi_\lambda(\mathsf S)
=
\Phi_\lambda(\varnothing)-G_\lambda(\mathsf S),
\]
where
\[
\Phi_\lambda(\varnothing)
=
\sigma^2+\sum_{i=1}^d v_i(\theta_i^*)^2.
\]
Consequently,
\[
\argmin_{\substack{\mathsf S\subseteq[d]\\|\mathsf S|\le s}}
\Phi_\lambda(\mathsf S)
=
\argmax_{\substack{\mathsf S\subseteq[d]\\|\mathsf S|\le s}}
G_\lambda(\mathsf S).
\]
\end{proposition}
\begin{proof}
Let
\[
r_i(\lambda)=\frac{v_i}{v_i+\lambda}.
\]
For $i\in\mathsf S$,
$
\theta_i^{\mathrm{OR}}(\mathsf S,\lambda)=r_i(\lambda)\theta_i^*
$,
and for $i\notin\mathsf S$,
$
\theta_i^{\mathrm{OR}}(\mathsf S,\lambda)=0
$.
The prediction component of the strategic MSE is
\[
\sum_{i\notin\mathsf S}v_i(\theta_i^*)^2
+
\sum_{i\in\mathsf S}v_i(\theta_i^*)^2(1-r_i(\lambda))^2.
\]
Equivalently, relative to the empty support, retaining coordinate $i$ gives prediction gain
\[
v_i(\theta_i^*)^2
-
v_i(\theta_i^*)^2(1-r_i(\lambda))^2
=
v_i(\theta_i^*)^2
\left(
2r_i(\lambda)-r_i(\lambda)^2
\right).
\]
Since
\[
2r_i(\lambda)-r_i(\lambda)^2
=
\frac{v_i(v_i+2\lambda)}{(v_i+\lambda)^2},
\]
this gain is exactly $b_i(\lambda)$.

The strategic term is
\[
\left(
\sum_{i\in\mathsf S}
\omega_i r_i(\lambda)^2(\theta_i^*)^2
\right)^2
=
\left(
\sum_{i\in\mathsf S}c_i(\lambda)
\right)^2.
\]
Thus
\[
\Phi_\lambda(\mathsf S)
=
\Phi_\lambda(\varnothing)
-
\sum_{i\in\mathsf S}b_i(\lambda)
+
\left(
\sum_{i\in\mathsf S}c_i(\lambda)
\right)^2,
\]
which is equivalent to the stated identity.
\end{proof}

\begin{proposition}[Submodularity of the diagonal gain function]
\label{prop:diag-gain-submodular}
For fixed $\lambda\ge0$, the set function $G_\lambda$ is submodular.
\end{proposition}
\begin{proof}
For $j\notin\mathsf S$, the marginal gain from adding $j$ is
\begin{align*}
\Delta_j^\lambda(\mathsf S)
&:=
G_\lambda(\mathsf S\cup\{j\})-G_\lambda(\mathsf S)\\
&=
b_j(\lambda)
-
\left(
\sum_{i\in\mathsf S}c_i(\lambda)+c_j(\lambda)
\right)^2
+
\left(
\sum_{i\in\mathsf S}c_i(\lambda)
\right)^2\\
&=
b_j(\lambda)
-
2c_j(\lambda)\sum_{i\in\mathsf S}c_i(\lambda)
-
c_j(\lambda)^2.
\end{align*}
Now let $\mathsf S\subseteq\mathsf T$ and $j\notin\mathsf T$.
Since $c_i(\lambda)\ge0$ for all $i$,
\[
\sum_{i\in\mathsf S}c_i(\lambda)
\le
\sum_{i\in\mathsf T}c_i(\lambda).
\]
Therefore
$
\Delta_j^\lambda(\mathsf S)
\ge
\Delta_j^\lambda(\mathsf T)
$.
This is exactly the diminishing-returns characterization of submodularity.
\end{proof}

\subsection{Exact bilinear formulation and McCormick formulation}

The weighted problem can also be written as an exact nonlinear constrained program.
This formulation is useful for deriving convex lower bounds.

Fix $\lambda>0$ and define $r:=\Sigma(\theta^*-\theta)$.
For interior weights, the weighted ridge normal equation is $(\Sigma+\lambda W^{-1})\theta=\Sigma\theta^*$.
Equivalently,
\[
r_i=\lambda\frac{\theta_i}{w_i},
\qquad i=1,\dots,d,
\]
or
\[
\lambda\theta_i=w_i r_i,
\qquad i=1,\dots,d.
\]

\begin{proposition}[Exact bilinear formulation]
\label{prop:exact-bilinear-weighted-formulation}
The weighted relaxation is equivalent to
\begin{equation}
\label{eq:exact-weighted-bilinear-app}
\begin{aligned}
\min_{\theta,r,w}\quad
&
(\theta-\theta^*)^\top\Sigma(\theta-\theta^*)
+
(\theta^\top H^{-1}\theta)^2
+
\sigma^2\\
\textnormal{s.t.}\quad
&
r=\Sigma(\theta^*-\theta),\\
&
\lambda\theta_i=w_i r_i,
\qquad i=1,\dots,d,\\
&
0\le w_i\le1,
\qquad i=1,\dots,d,\\
&
\mathbf 1^\top w\le s.
\end{aligned}
\end{equation}
\end{proposition}
\begin{proof}
For any $w$ with strictly positive coordinates, the constraints $r=\Sigma(\theta^*-\theta)$ and $\lambda\theta_i=w_i r_i$ for $i=1,\dots,d$, are exactly the weighted ridge normal equations. Therefore the feasible $\theta$ associated with $w$ is precisely $\theta_\lambda(w)$.
The objective is exactly $\mathrm{MSE}_{\mathrm{strat}}(\theta)$.
At boundary points with some $w_i=0$, the equality $\lambda\theta_i=w_i r_i$
forces $\theta_i=0$, which matches the perspective convention in the weighted ridge definition.
Thus the formulation remains exact on the boundary as well.
\end{proof}

We next describe a convex lower-bound relaxation.
Let $B_0:=\theta^{*\top}\Sigma\theta^*$.
Since $w=0$ is feasible and gives $\theta=0$, any minimizer of the weighted relaxation has prediction error at most $B_0$:
\[
(\theta-\theta^*)^\top\Sigma(\theta-\theta^*)
\le
B_0.
\]
Therefore, for each coordinate,
\[
|r_i|
=
|e_i^\top\Sigma(\theta^*-\theta)|
\le
\sqrt{\Sigma_{ii}B_0}
=:
R_i.
\]
The bound follows from Cauchy's inequality:
\[
|e_i^\top\Sigma(\theta^*-\theta)|
=
|(\Sigma^{1/2}e_i)^\top\Sigma^{1/2}(\theta^*-\theta)|
\le
\sqrt{\Sigma_{ii}}
\sqrt{(\theta-\theta^*)^\top\Sigma(\theta-\theta^*)}.
\]

For $0\le w_i\le1$ and $-R_i\le r_i\le R_i$, the convex hull relaxation of $\lambda\theta_i=w_i r_i$
is given by the McCormick inequalities:
\begin{align}
\label{eq:mcc1-app}
\lambda\theta_i &\ge -R_i w_i,\\
\label{eq:mcc2-app}
\lambda\theta_i &\le R_i w_i,\\
\label{eq:mcc3-app}
\lambda\theta_i &\ge r_i-R_i(1-w_i),\\
\label{eq:mcc4-app}
\lambda\theta_i &\le r_i+R_i(1-w_i).
\end{align}

\section{Algorithmic Details and Illustrations for \Cref{sec:algorithm}}
\label{sec:app-algorithm-illustrations}

The synthetic benchmarks in this section validate the algorithmic pipeline from \Cref{sec:algorithm} in settings where a support-constrained oracle can be computed exactly.
The main question behind these benchmarks is whether,
given population objects $(\Sigma,\theta^*,H)$ and a support budget $s$, weighted relaxation followed by rounding and exact local refinement recovers the best support-restricted ridge rule.
Because the benchmark is low-dimensional, we can enumerate all supports and directly compare every method to the exact grid oracle.

\paragraph{Problem size and oracle.}
All instances use
\[
    d=18,
    \qquad
    s=5,
    \qquad
    \Lambda=\{0.05,\ 0.20,\ 0.50,\ 1.00,\ 2.00\}.
\]
For each instance, the exact grid oracle is
\[
(\mathsf S^\star,\lambda^\star)
\in
\arg\min_{\substack{\mathsf S\subseteq[d]\\ |\mathsf S|\le s}}
\ \min_{\lambda\in\Lambda}\,
\Phi(\mathsf S,\lambda),
\qquad
\Phi(\mathsf S,\lambda)
=
\mathrm{MSE}_{\mathrm{strat}}
\bigl(\theta^{\mathrm{OR}}(\mathsf S,\lambda)\bigr).
\]
The oracle is computed by exhaustive search over all supports of size at most $s$ and all ridge levels over the grid $\Lambda$. All reported values are population strategic $\MSE$ under the synthetic $(\Sigma,\theta^*,H)$.
We suppress the additive noise variance $\sigma^2$, since it is constant across methods for a fixed instance.

\paragraph{Synthetic regimes.}
We consider three covariance regimes and two signal--cost regimes.
The covariance regimes are:
identity covariance where $\Sigma=I_d$;
AR(1) covariance where $\Sigma_{ij}=\rho^{|i-j|}$ with $\rho=0.65$;
and block-correlated covariance where the $18$ coordinates are split into three blocks of size $6$, with pairwise correlation $0.80$ within each block and $0.05$ across blocks, and diagonal entries normalized to one.
Throughout, the cost matrix $H$ is diagonal, and we generate inverse costs $H^{-1}_{jj}$ directly. 

We consider two regimes for the joint distribution of signal and manipulation costs.
In the independent regime, signal and inverse costs are drawn independently as:
$\theta_j^* = |\mathcal N(0.9,0.45^2)|+0.05$ with $H^{-1}_{jj}=\exp(\mathcal N(0,0.8^2))$.
In the tradeoff regime, higher-signal coordinates also tend to be more manipulable. 
Specifically, we draw $z_j,\varepsilon_j,\eta_j\stackrel{\mathrm{iid}}{\sim}\mathcal N(0,1)$ and set $\theta_j^*=|0.95+0.60z_j+0.20\varepsilon_j|+0.05$, with  $H^{-1}_{jj}=\exp(0.95z_j+0.35\eta_j)$.
For each of the $3\times2$ regime combinations, we generate $5$ independent replications, giving $30$ total benchmark instances.

\begin{figure}[!t]
    \centering
    \includegraphics[width=0.99\linewidth]{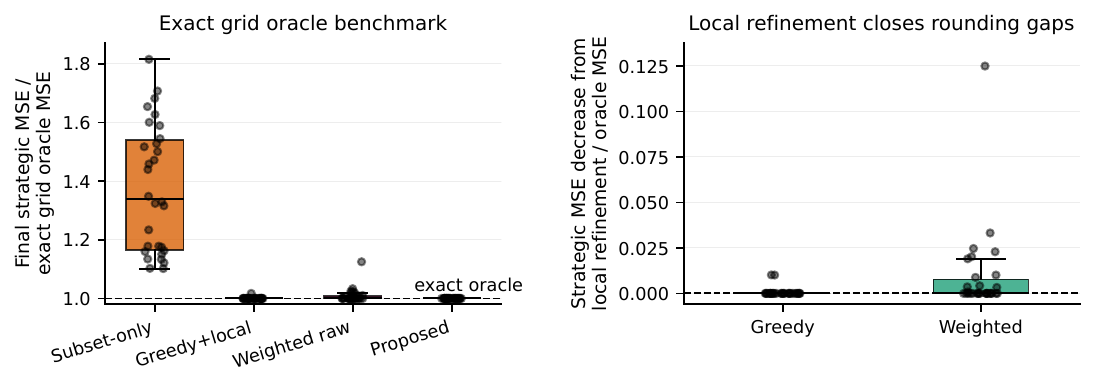}
    \caption{
    \textbf{Low-dimensional synthetic benchmarks validate our proposed weighted screen-and-refit method.}
    (Left) Final strategic $\MSE$ normalized by the exact grid oracle $\MSE$ on the same instance, $\mathrm{MSE}_{\mathrm{strat}}(\widehat\theta)/\mathrm{MSE}_{\mathrm{strat}}(\theta^{\mathrm{OR}}(\mathsf S^\star,\lambda^\star))$. A value of $1$ means that the method attains the exact grid oracle.
    The proposed weighted screen-and-refit method attains the oracle on all $30$ instances; greedy with local refinement attains it on $29$ of $30$; weighted rounding alone attains it on $18$ of $30$ and can have a small rounding gap; subset selection without regularization remains consistently farther from the oracle.
    (Right) Decrease in strategic $\MSE$ produced by local refinement, normalized by the exact grid oracle $\MSE$. Local refinement rarely changes the greedy solution, but it repairs several rounded weighted supports.
    This motivates its inclusion in our default algorithm.
    }
    \label{fig:lowdim-algorithm-design}
\end{figure}

\begin{figure}[!t]
    \centering
    \includegraphics[width=0.8\linewidth]{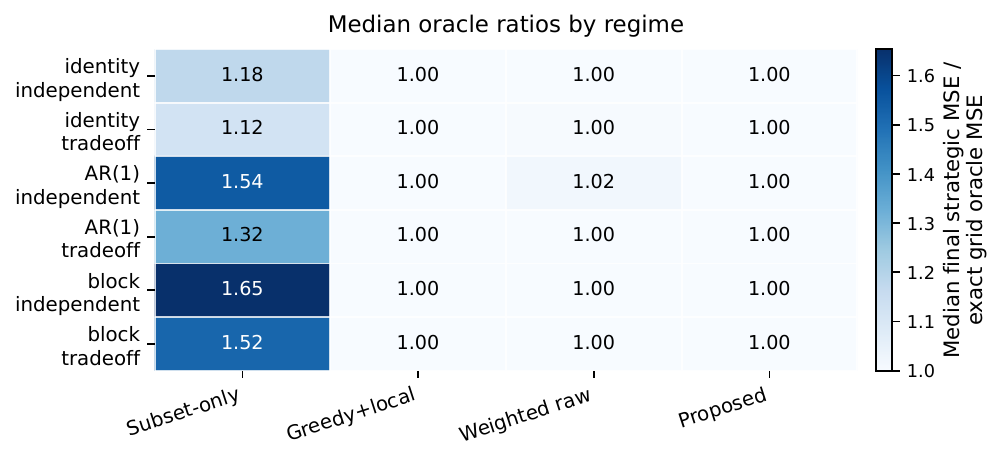}
    \caption{
    \textbf{Median oracle ratios by synthetic regime.}
    Each entry reports the median final strategic $\MSE$ normalized by the exact grid oracle $\MSE$, $\mathrm{MSE}_{\mathrm{strat}}(\widehat\theta)/\mathrm{MSE}_{\mathrm{strat}}(\theta^{\mathrm{OR}}(\mathsf S^\star,\lambda^\star))$,
    across the five replications in the corresponding regime.
    The proposed weighted screen-and-refit method attains median ratio $1.00$ in every regime.
    The subset-only baseline is farthest from the oracle, especially in the correlated regimes.
    This illustrates the value of combining support selection with regularization stressed throughout the paper.
    }
    \label{fig:lowdim-algorithm-heatmap}
\end{figure}

\paragraph{Methods.}
We compare four support-budget-feasible methods: 
(1) Best subset without regularization, which exhaustively searches over supports of size at most $s$ at $\lambda=0$.
This baseline isolates the value of support selection alone, without ridge shrinkage.
(2) Greedy plus local refinement, in which for each fixed $\lambda \in \Lambda$, we start from the empty support and repeatedly add the coordinate whose inclusion gives the largest decrease in the exact fixed-$\lambda$ objective $\Phi(\mathsf S,\lambda)$, until the support budget is reached.
The resulting greedy support is then refined by the same add/drop/swap local search used in our method (see \Cref{alg:weighted-relaxation-local-refinement} for more details), and the best refined candidate is chosen over $\lambda \in \Lambda$.
(3) Weighted rounding only, in which we solve the weighted relaxation over $w \in[0,1]^d$ with $\mathbf 1^\top w\le s$, round to the top-$s$ coordinates of the resulting weight vector, refit exact support-restricted ridge on that rounded support, and then choose the best $\lambda$.
(4) Weighted screen-and-refit, which starts from the same weighted rounded support but then applies exact add/drop/swap local refinement before the final support-restricted ridge refit.

The local-refinement neighborhood contains all one-step additions, deletions, and swaps that respect the support budget.  
The refinement procedure repeatedly accepts the best improving move under the exact fixed-$\lambda$ objective $\Phi(\mathsf S,\lambda)$ until no improving move remains.
This step can only improve the rounded support and terminates after finitely many moves by \Cref{prop:local-refinement-termination}.
We do not include full-support ridge in the oracle-ratio figures because it violates the support budget and therefore is not a feasible method for this support-constrained benchmark.

\paragraph{Results.}
\Cref{fig:lowdim-algorithm-design} shows that support selection and regularization must be tuned jointly: the exact subset-only baseline at $\lambda=0$ has median normalized $\MSE$ $1.34$ and maximum normalized $\MSE$ $1.82$, relative to the exact grid oracle. \Cref{fig:lowdim-algorithm-heatmap} further shows that the subset-only baseline remains farthest from the oracle across covariance and signal–cost regimes.
Weighted relaxation is useful as a screening device, but the rounded support should be refined.
Weighted rounding alone attains the exact grid oracle on $18$ out of $30$ instances and has maximum normalized $\MSE$ $1.125$, whereas the weighted screen-and-refit pipeline after exact local refinement attains the exact grid oracle on all $30$ instances.
Local refinement improves the rounded weighted solution on $12$ of $30$ instances, with maximum normalized $\MSE$ decrease $0.125$.

These synthetic benchmarks motivate the final algorithmic steps used in the experiments: weighted screening to obtain a strong support initializer, exact local refinement to remove rounding artifacts, and exact support-restricted ridge refitting on the refined support.
Greedy local search is a strong discrete baseline in these low-dimensional problems, but the weighted relaxation provides the more general methodology: it directly implements the continuous relaxation introduced in \Cref{sec:algorithm} and admits the convex reformulation for the diagonal-covariance case in \Cref{prop:diag-cov-convex}.

\section{Experimental Details and Illustrations for \Cref{sec:experiments}}
\label{sec:app-expts}

\subsection{Data description}
\label{sec:data-description}

The \texttt{upcoding} package\footnote{\url{https://github.com/StanfordHPDS/upcoding}} simulates realistic co-occurring HCCs for American adults aged 65 years and older (i.e., Medicare eligible) based on self-reported health conditions from All of Us survey data~\citep{all2019all}. Because these self-reported conditions are unaffected by coding incentives, the resulting baseline data provide an unmanipulated representation of patients’ underlying health conditions.
The package allows users to modify the baseline data by specifying either upcoding or undercoding of existing baseline diagnoses.
We use the package to construct realistic baseline and manipulated CMS-HCC Version 28 (V28) feature vectors. 
The resulting baseline data contain $n=5000$ beneficiaries and 115 binary V28 HCC indicators. 
The HCC features are denoted \texttt{HCC1}, \texttt{HCC2}, \dots and correspond to CMS-HCC condition categories.
For example, HCC238 corresponds to specified heart arrhythmias, while HCCs 125--127 correspond to severe, moderate, and mild or unspecified dementia, respectively.

The upcoding targets are the ten V28 diagnosis groups identified in recent work~\citep{kronick2025are} as the groups that contributed the most to differential coding between traditional Medicare (TM) and Medicare Advantage (MA): chronic obstructive pulmonary disease (COPD), acute or chronic heart failure, diabetes, morbid obesity, drug use disorder, chronic kidney disease, major depression, rheumatoid arthritis, hemiplegia/hemiparesis, and alcohol use disorder.
These groups are reported to account for almost all of the V28 MA--TM measured-risk difference in 2021, while the remaining 100 HCCs contribute little aggregate difference.
Related work~\citep{kronick2025insurer} also shows that differential persistence and new incidence in the top ten diagnosis groups explain most of the MA--TM risk-score gap, with little contribution from the remaining HCCs.

The ten diagnosis groups identified by~\citep{kronick2025are} as most intensely coded in MA (which we refer to as the top-ten diagnosis groups throughout) are used to calibrate the direction and magnitude of the simulated
coding shifts.
For each targeted group $g$, the desired post-upcoding prevalence is
constructed as
\[
    p^{\mathrm{target}}_{g,m}
    =
    p^{(0)}_g + m\,(p^{\mathrm{MA}}_g-p^{\mathrm{TM}}_g),
\]
where $p^{(0)}_g$ is the baseline simulator prevalence, $p^{\mathrm{MA}}_g$ and $p^{\mathrm{TM}}_g$ are the Medicare Advantage and traditional Medicare prevalences reported in Exhibit 2 of \citet{kronick2025are}, and $m\in\{0.5,0.75,1.0\}$ indexes the low, main, and high upcoding scenarios.
The simulator converts these target prevalences into upcoding proportions, clips them to $[0,1]$, and uses one of two upcoding approaches, any-available or severity-based, depending on whether the targeted group has lower-severity conditions within the same clinical family. 
Under any-available upcoding, any beneficiary not already coded for a target HCC is eligible to be upcoded into that HCC; under severity-based upcoding, only beneficiaries already coded for a lower-severity related HCC are eligible. 
Applying the CMS-HCC V28 hierarchy then sets lower-severity HCC indicators to zero whenever a higher-severity HCC is present for a beneficiary.
The three scenarios are generated independently.

A target-mapping table is constructed as part of this upcoding calibration step.
Each row corresponds to a diagnosis group from the top-ten diagnosis groups and records: (i) the target HCC whose prevalence is targeted in the simulator, (ii) the V28 HCCs belonging to the same diagnosis group, and (iii) the group-level calibration score $r_g$, which is the group-level contribution to the MA--TM risk-score gap.
The simulator then uses this table, together with the baseline HCC matrix, to generate the low, main, and high upcoding scenarios.

We treat the baseline covariates as the pre-deployment feature distribution.
The manipulated covariates are used to calibrate and validate the coding shifts.
We use the full baseline cohort of $n=5000$ beneficiaries to compute the empirical mean and covariance of the HCCs with nonzero baseline prevalence, construct the semi-synthetic signal vector and manipulation-ease matrix, and evaluate each method using the closed-form post-manipulation MSE described in \Cref{sec:supp-semisynth-metrics}.

\subsection{Exploratory data analysis and calibration diagnostics}

The baseline is sparse.
In the baseline data, the mean number of HCCs per beneficiary is $1.52$, the median is $1$, and the maximum is $5$.
In the high-upcoding scenario, the corresponding mean is $1.79$. 
\Cref{fig:row_sum_distributions} reports the beneficiary-level distribution of these HCC counts across the baseline and calibrated upcoding scenarios.

\begin{figure}[!t]
    \centering
    \includegraphics[width=0.8\linewidth]{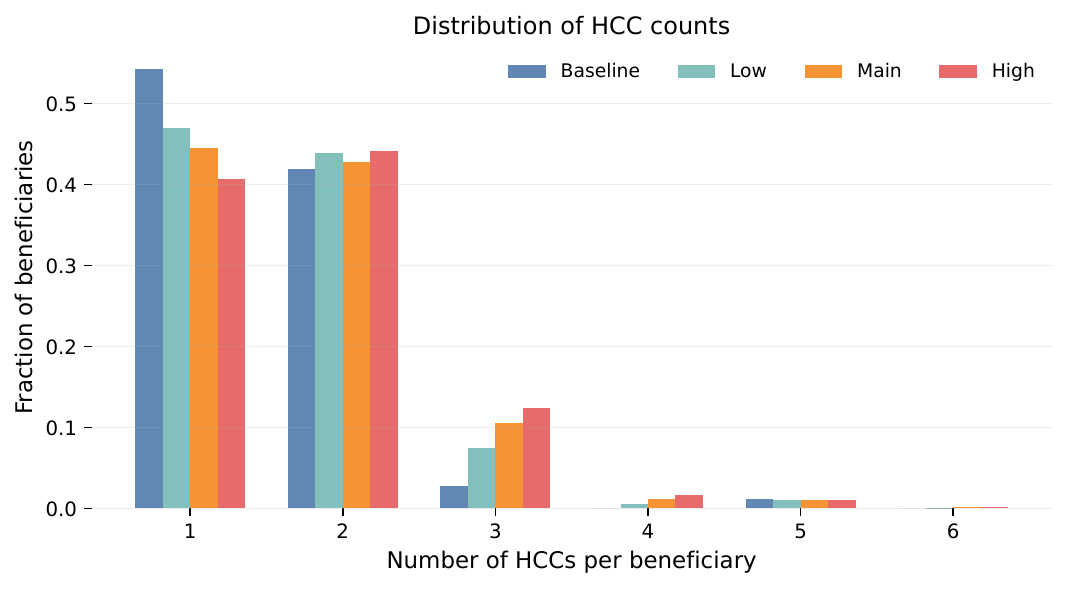}
    \caption{
    \textbf{Distribution of HCC counts per beneficiary.}
    The figure shows the number of V28 HCCs per beneficiary in the baseline data and in the low, main, and high upcoding scenarios.
    Only HCCs with nonzero baseline prevalence are included. 
    The baseline is sparse, with mean $1.52$ HCCs, median $1$, and maximum $5$.
    The upcoding scenarios increase coding intensity modestly.
    }
    \label{fig:row_sum_distributions}
\end{figure}

For each target HCC, the achieved prevalence after applying the V28 hierarchy closely matches the
intended prevalence.  
Across the low, main, and high scenarios, the mean absolute target-prevalence errors are $0.0002$, $0.0001$, and $0.0003$, respectively, and the largest absolute error over all targets and scenarios is $0.0010$.
The left panel of \Cref{fig:target_prevalence_combined} compares baseline, target, and achieved prevalences, while the right panel reports the realized prevalence shifts across scenarios.

\begin{figure}[t]
    \centering
    \begin{minipage}[c]{0.48\textwidth}
        \centering
        \includegraphics[width=\linewidth]{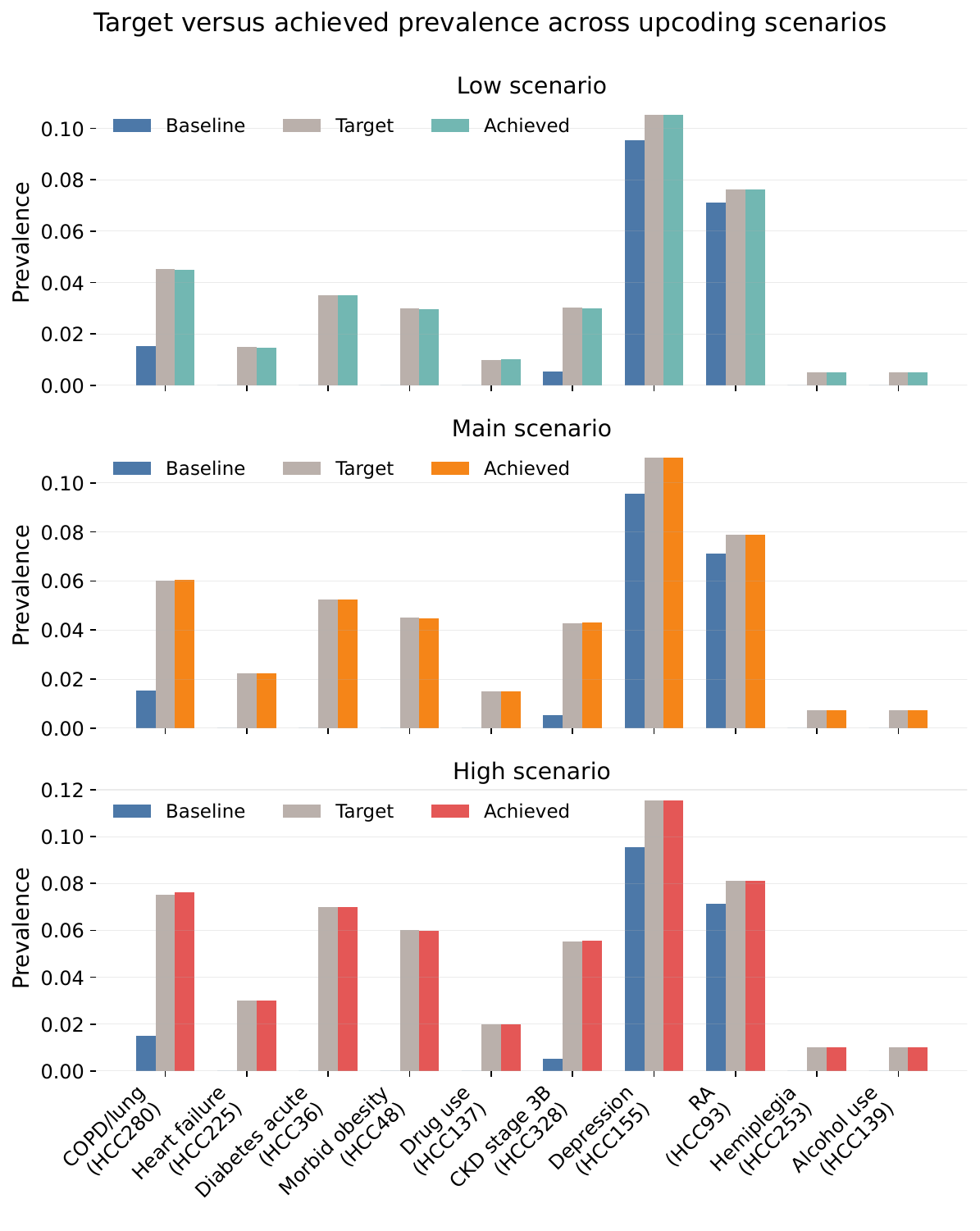}
    \end{minipage}\hfill
    \begin{minipage}[c]{0.52\textwidth}
        \centering
        \includegraphics[width=\linewidth]{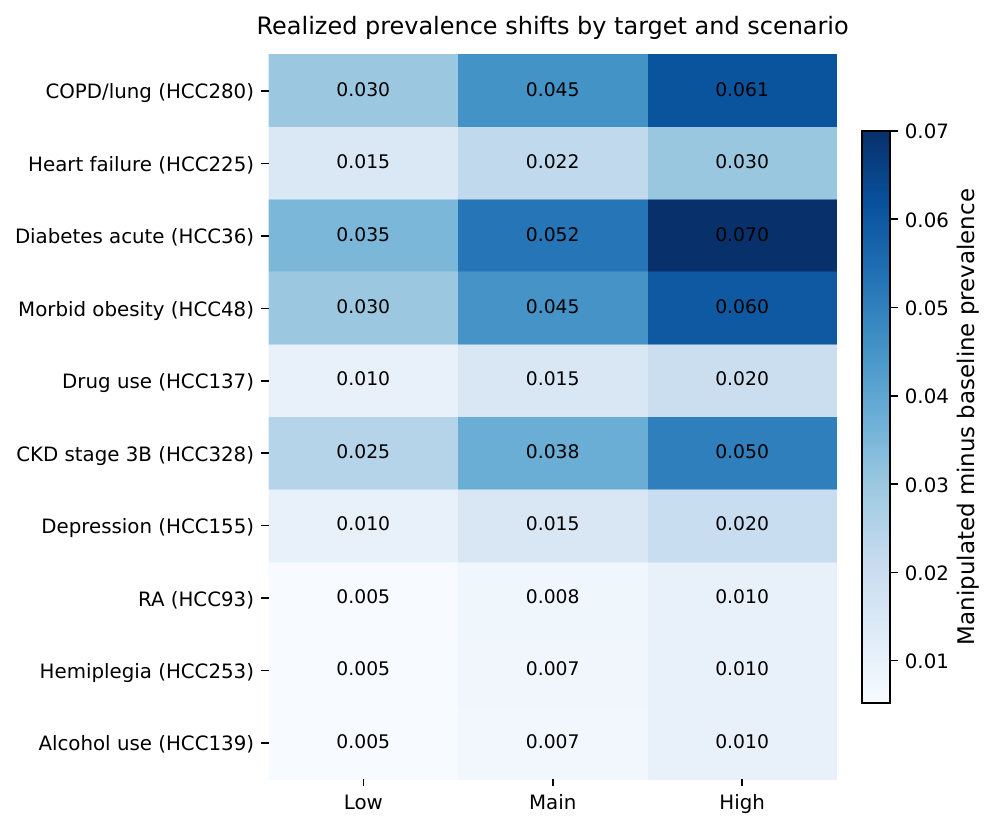}
    \end{minipage}
    \caption{
    \textbf{Calibration of simulated upcoding across target HCCs and scenarios.}
    (Left) Baseline, target, and achieved prevalence after applying the V28 hierarchy  for each target HCC under the low, main, and high scenarios.
    The achieved prevalences closely track the calibration targets.
    (Right) Realized manipulated-minus-baseline prevalence shifts after applying the V28 hierarchy.
    The low, main, and high scenarios scale the same diagnosis-level calibration targets, and the realized shifts are concentrated in the targeted top-ten diagnosis groups.
    }
    \label{fig:target_prevalence_combined}
\end{figure}

\subsection{Semi-synthetic outcome, strategic response, and metrics}
\label{sec:supp-semisynth-metrics}

Let $X_i^{(0)}\in\{0,1\}^{115}$ denote the baseline HCC vector for beneficiary
$i$.  We restrict the benchmark optimization to the $30$ HCCs with nonzero baseline prevalence and center
these features using the empirical baseline mean.  Writing the resulting centered
feature vector as $\widetilde X_i^{(0)}$, we construct the semi-synthetic outcome
as
\[
    Y_i = \widetilde X_i^{(0)\top}\theta^* + \varepsilon_i,
    \qquad
    \varepsilon_i\sim (0,\sigma^2),
\]
where $\theta^*$ is the published CMS-HCC V28 Community, NonDual, Aged
coefficient vector restricted to these 30 HCCs.  We set
$\sigma^2=0.25\,\widehat{\mathrm{Var}}(\widetilde X^{(0)\top}\theta^*)$.
The same $\theta^*$ is used both to generate outcomes and to compute the plug-in strategic $\MSE$.

For a fitted coefficient vector $\widehat\theta$, evaluation is based on the
endogenous strategic response model studied in the paper:
\[
    \widetilde X_i^{(\alpha,\widehat\theta)}
    =
    \widetilde X_i^{(0)} + \alpha H^{-1}\widehat\theta,
\]
where $\alpha\ge0$ is the manipulation-intensity parameter. The main figures
report the plug-in post-manipulation MSE
\begin{equation}
\label{eq:post-manip-mse}\mathrm{MSE}_{\alpha}(\widehat\theta)
    =
    (\widehat\theta-\theta^*)^\top\widehat\Sigma
    (\widehat\theta-\theta^*)
    +
    \alpha^2(\widehat\theta^\top H^{-1}\widehat\theta)^2
    +
    \sigma^2.
\end{equation}
This is the exact empirical analogue of the theoretical strategic MSE studied in the paper. 
We normalize most reported quantities by the full ridge value at the
same manipulation intensity, so values below one indicate an improvement over
full-support ridge.

\subsection{Cost matrix construction}
\label{sec:supp-semisynth-cost}

The manipulation cost matrix $H$ is not directly observed.
We therefore construct a proxy using available evidence from health policy on differential coding.
Since the strategic best response $a^*(\theta)=H^{-1}\theta$ depends on $H^{-1}$, it is convenient to build the inverse-cost matrix first.
We refer to $H^{-1}$ as the manipulation-ease matrix.
Larger entries of $H^{-1}$ mean that a fixed coefficient vector induces a larger strategic response, or equivalently, that manipulation is less costly in the corresponding direction.

The cost construction is fixed before fitting any model or choosing any support.
It uses three inputs: (i) the CMS-HCC V28 feature list; (ii) the baseline HCC matrix; and (iii) the target-mapping table, which links each calibrated top-ten diagnosis group to its V28 HCC members and records a group-level calibration score.
(Recall from \Cref{sec:data-description} that the realized upcoded matrices are used to check that these calibration targets are achieved.)

The final manipulation-ease matrix has the form:
\begin{equation}
\label{eq:Hinv-model}
    H^{-1}
    =
    \tau\{D+\rho B\}+\eta I.
\end{equation}
Here $D$ is diagonal and controls the marginal ease of manipulating each HCC. 
The matrix $B$ is a block matrix that introduces correlated manipulation within pre-specified HCC blocks based on the CMS-HCC V28 hierarchy. We refer to $D$ as the diagonal component and $B$ as the block component of the inverse-cost matrix.
The scalar $\tau$ controls the overall manipulation-intensity scale, $\rho$ controls the strength of the block component, and $\eta I$ is a small ridge term that ensures strict positive definiteness.
These scalar parameters are fixed across all experiments and summarized in \Cref{tab:case-study-cost-params}.
We next describe the construction of the matrices $D$ and $B$.

\begin{table}[!t]
\centering
\caption{Cost-construction parameters for the experiments in \Cref{sec:experiments}.}
\label{tab:case-study-cost-params}
\begin{tabular}{ccc}
\toprule
Parameter & Description & Value \\
\midrule
$\tau$ & global manipulation-intensity scale & $0.5$ \\
$\rho$ & correlated-cost scale & $0.5$ \\
$\eta$ & positive-definiteness ridge & $10^{-6}$ \\
$d_{\mathrm{floor}}$ & baseline diagonal manipulation ease for every HCC & $0.03$ \\
$\xi$ & moderation factor for HCCs in the top-ten diagnosis groups & 0.25 \\
\bottomrule
\end{tabular}
\end{table}

\paragraph{Construction of matrix $D$.}

We first construct diagonal scores over all V28 HCC columns and then restrict to the HCC columns with nonzero prevalence. We initialize every HCC at the floor: $d_j^{\mathrm{raw}} = d_{\mathrm{floor}}$.
For HCCs belonging to a diagnosis group in the target-mapping table, we increase its diagonal manipulation-ease score according to the normalized group-level calibration score.
Specifically, if HCC $j$ belongs to one of the top-ten diagnosis groups $g$, we update $d_j^{\mathrm{raw}} \leftarrow \max\{d_j^{\mathrm{raw}},\, d_{\mathrm{floor}}+s_g\}$, where $s_g := r_g/{\max_{g'} r_{g'}} \in [0,1]$ is the normalized (group-level) calibration score for that group $g$.
Thus the diagonal manipulation-ease scores are tied to the same diagnosis groups used to generate the upcoded scenarios.

Finally, we apply a moderation factor to the HCCs belonging to the top-ten diagnosis groups (i.e., $\texttt{HCC155},\texttt{HCC280},\texttt{HCC38},\texttt{HCC93},\texttt{HCC328},\texttt{HCC226}$, which are also highlighted in \Cref{fig:anyfix-selection-mechanism}).
Specifically, for HCC $j$ in this group, we set $d_j = d_{\mathrm{floor}} + \xi (d_j^{\mathrm{raw}}-d_{\mathrm{floor}})$, where $\xi$ is the moderation factor applied to the excess manipulation score.
The purpose of this moderation factor is to avoid making the conclusions degenerate.
The top-ten group membership is coarse group-level evidence of coding sensitivity, but not a direct estimate of an infinite coordinate-level 
manipulation ease.
If every HCC belonging to the top-ten diagnosis groups were assigned an extreme unmoderated ease score, the cost matrix would mechanically favor exclusion of these HCCs.
The moderation factor instead creates a non-degenerate selection problem: these HCCs remain easier to manipulate than baseline HCCs, but their predictive value and covariance structure can still justify retention under ridge shrinkage.
For all other HCCs, we set $d_j = d_j^{\mathrm{raw}}$.

\paragraph{Construction of matrix $B$.}
Upcoding practices for a single diagnosis are likely to affect closely related diagnoses as well. This intuition is captured in the block matrix $B$, which captures correlated manipulation within pre-specified HCC blocks.
These blocks are based on CMS-HCC V28 hierarchy branches and represent groups of HCCs that correspond to the same underlying condition or severity levels within the same condition.
\Cref{tab:hcc-family-blocks} lists the blocks used in our construction.
For each block, we add a positive-semidefinite
outer-product term and then normalize the resulting block matrix over the HCCs with nonzero prevalence. More precisely, for each listed HCC block $\mathcal G_\ell$ for $\ell = 1, \dots, L$, define a vector $v_\ell$ over all HCC columns by
\[
    (v_\ell)_j
    =
    \begin{cases}
    \sqrt{\max\{d_j,d_{\mathrm{floor}}\}},
    & j\in\mathcal G_\ell,\\
    0,
    & j\notin\mathcal G_\ell.
    \end{cases}
\]
We then set $B_0 = \sum_{\ell=1}^L v_\ell v_\ell^\top$.
This construction makes $B_0$ positive semidefinite.
Whenever $B_0$ is nonzero, we normalize it by its largest diagonal entry $B = \frac{B_0}{\max_j (B_0)_{jj}}$. 
Finally, we restrict to the HCC columns with nonzero baseline prevalence to get the manipulation-ease matrix.
Since $D, B \succeq 0$ and $\eta>0$, the final matrix is positive definite.
\Cref{fig:supp-cost-diagnostics} shows the resulting diagonal scores in the left panel and the full manipulation-ease matrix in the right panel.

\begin{table}[!t]
\centering
\caption{Pre-specified diagnosis group blocks used in the block component $B$.}
\label{tab:hcc-family-blocks}
\begin{tabular}{ll}
\toprule
Diagnosis group blocks & HCCs \\
\midrule
Diabetes & HCC36, HCC37, HCC38 \\
Acute or chronic heart failure, except end stage & HCC224, HCC225, HCC226 \\
Chronic kidney disease, moderate, stage 3 & HCC328, HCC329 \\
Substance-use disorder & HCC137, HCC138, HCC139 \\
Mental health & HCC151, HCC152, HCC153, HCC154, HCC155 \\
Vascular disease & HCC263, HCC264, HCC267 \\
Cancer & HCC20, HCC21, HCC22, HCC23 \\
\bottomrule
\end{tabular}
\end{table}

\begin{figure}[!t]
    \centering
    \includegraphics[width=0.99\linewidth]{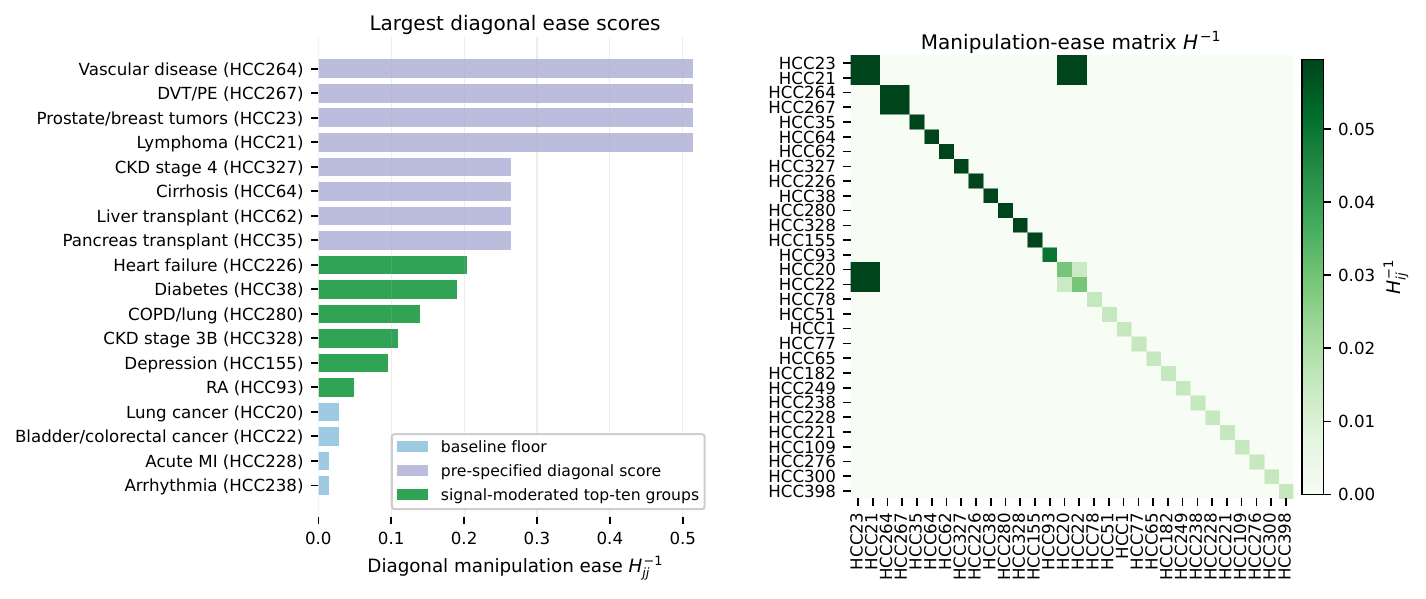}    
    \caption{
    \textbf{Visualization of the manipulation-ease matrix.}
    (Left) Largest HCC diagonal entries of the final $H^{-1}$, grouped by the source of the diagonal score.
    (Right) The final inverse-cost matrix over HCCs with nonzero baseline prevalence, including the block component.
    }
    \label{fig:supp-cost-diagnostics}
\end{figure}

It is worth remarking that the construction of the manipulation-ease matrix in \eqref{eq:Hinv-model} should be interpreted as a proxy, not as an estimated structural inverse cost matrix.
Its purpose in this paper is to encode three qualitative features of the Medicare Advantage coding setting: coding sensitivity is concentrated in a small number of diagnosis groups; coding opportunities can be correlated within certain diagnosis group blocks; and coding-sensitive diagnoses can still carry crucial predictive signal.
The various choices above are made with respect to these considerations.

\subsection{Baselines and tuning}
\label{sec:supp-semisynth-baselines}

All methods are evaluated using the same plug-in post-manipulation $\MSE_\alpha$ objective \eqref{eq:post-manip-mse}.
The main comparisons are summarized below.

\begin{enumerate}[itemsep=2pt,leftmargin=2em]
    \item Full ridge: ordinary ridge on all HCCs with nonzero baseline prevalence, with the ridge level chosen to minimize $\mathrm{MSE}_{\alpha}$.
    \item Top-ten exclusion: ridge after removing all HCCs with nonzero baseline prevalence that belong to the calibrated top-ten differential-coding groups.
    \item Prediction-only: retain the $k$ HCCs with nonzero baseline prevalence that have the largest absolute covariance with the outcome, $|\widehat\Sigma\theta^*|$, and then tune ridge on that support.
    \item Cost-only: retain the $k$ HCCs with nonzero baseline prevalence  that have the smallest diagonal manipulation-ease scores $H^{-1}_{jj}$, and then tune ridge on that support.
    \item Subset-only: use the same support as the proposed support-restricted ridge design but remove ridge shrinkage by setting $\lambda\approx0$.
    \item Ours: jointly choose a support of size $k$ and a ridge level using the $\MSE_\alpha$ objective, and then refit support-restricted ridge on the chosen support.
    \item Oracle $\widetilde{\OPT}$: the full-support zero-intercept oracle over all HCCs with nonzero baseline prevalence. It is not a deployable  method, since it lies outside the restricted estimator class, but serves as a benchmark for that class (see discussion in \Cref{sec:lower-bound}).
\end{enumerate}

For the oracle at manipulation intensity $\alpha$, the first-order condition has the generalized-ridge form \eqref{eq:gr-fixed-point},
which we solve by a scalar iterative fixed-point procedure.
All ridge-based methods are tuned over the common grid $\{10^{-6},10^{-3},0.003,0.01,0.03,0.1,0.3,1,3,10\}$.
Because our case study is a deterministic plug-in benchmark, we tune ridge levels by minimizing the plug-in strategic objective over this grid.
In finite-sample deployments, this tuning could instead be implemented by cross-validation.
In particular, for ridge regression, variants of cross-validation have consistency guarantees in non-strategic settings, even when these settings are high-dimensional and model-free; see, e.g., \citet{patil2021uniform,patil2022estimating}.
Extending these guarantees to strategic settings, along with joint support selection, is an interesting direction for future work.

\begin{figure}[!t]
    \centering
    \includegraphics[width=0.65\linewidth]{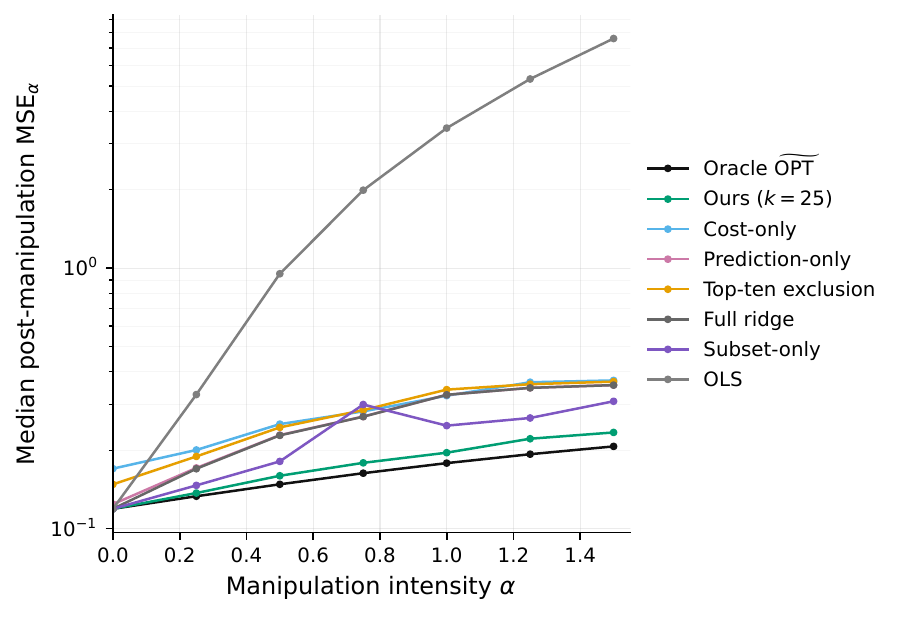}
    \caption{ 
    \textbf{Post-manipulation $\MSE$ under varying manipulation intensity $\alpha$ for all baselines.}
    Each curve reports median post-manipulation $\MSE$ as the manipulation intensity $\alpha$ varies, shown on a log scale.
    The proposed support-restricted ridge estimator with $k=25$ remains close to the zero-intercept oracle $\widetilde{\OPT}$ and improves substantially over full ridge, top-ten exclusion, prediction-only selection, cost-only selection, and subset-only selection.
    OLS is highly sensitive to manipulation and becomes unstable as $\alpha$ increases. 
    }
    \label{fig:supp-alpha-curves}
\end{figure}

\subsection{Supporting illustrations}
\label{sec:additional-results-expts}

\Cref{fig:supp-alpha-curves} extends the comparison under varying manipulation intensity in \Cref{fig:top10-anyfix-main} to all baselines.
The plot illustrates that 
unregularized subset selection and OLS-style estimators become unstable once manipulation is nontrivial, while prediction-only and cost-only rules fail to achieve the same tradeoff as joint support selection and ridge tuning.

\Cref{fig:anyfix-selection-mechanism} shows that, while the top-ten exclusion policy removes all intensely coded HCCs by construction, the proposed method retains some of these HCCs at the unconstrained optimum $k=20$ and all of them at the broader $k=25$ operating point.
This does not mean that all intensely coded variables should always be retained. Rather, it demonstrates that the optimal decision is joint:
an HCC with higher coding intensity can be retained when its predictive value and covariance structure make it sufficiently valuable, and when ridge shrinkage sufficiently controls the induced strategic exposure. \Cref{fig:supp-anyfix-policy-retention-path} shows the retention path as the support budget is varied.
As can be observed, the retention pattern is not a discontinuous artifact of a particular budget choice: several intensely coded HCCs enter before $k=25$, and all six are retained by the broader design.
\Cref{fig:supp-anyfix-signal-manipulability}
shows the same in the signal--manipulability plane.

\begin{figure}[!t]
    \centering
    \includegraphics[width=0.72\linewidth]{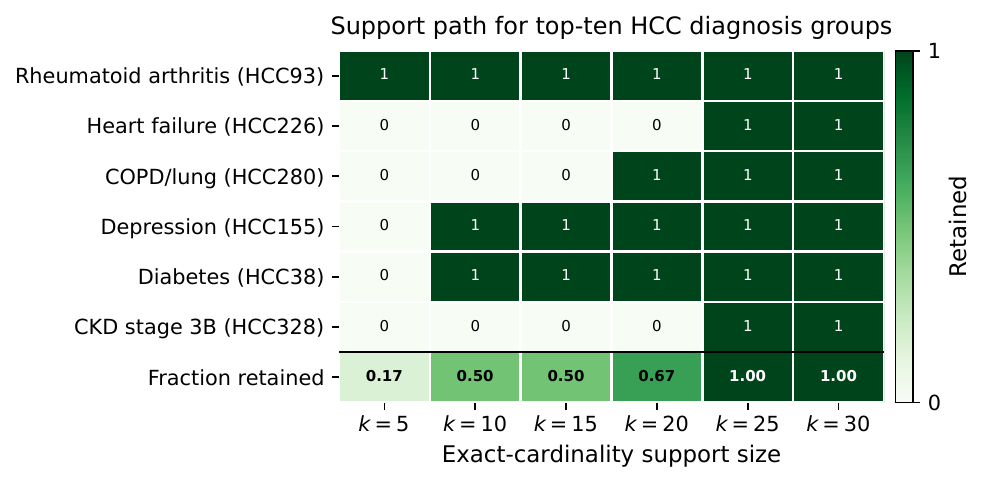}
    \caption{
    \textbf{Support path for top-ten diagnosis groups with nonzero prevalence.}
    Rows show the HCC groups that belong to the top-ten diagnosis groups identified by \citet{kronick2025are}, and columns show exact-cardinality support-restricted ridge solutions at different support sizes.
    The final row reports the fraction of these groups retained.
    The unconstrained optimum at $k=20$ already retains several of these groups, while the broader $k=25$ design retains all six groups with nonzero prevalence.
    }
    \label{fig:supp-anyfix-policy-retention-path}
\end{figure}

\begin{figure}[!t]
    \centering
    \includegraphics[width=0.65\linewidth]{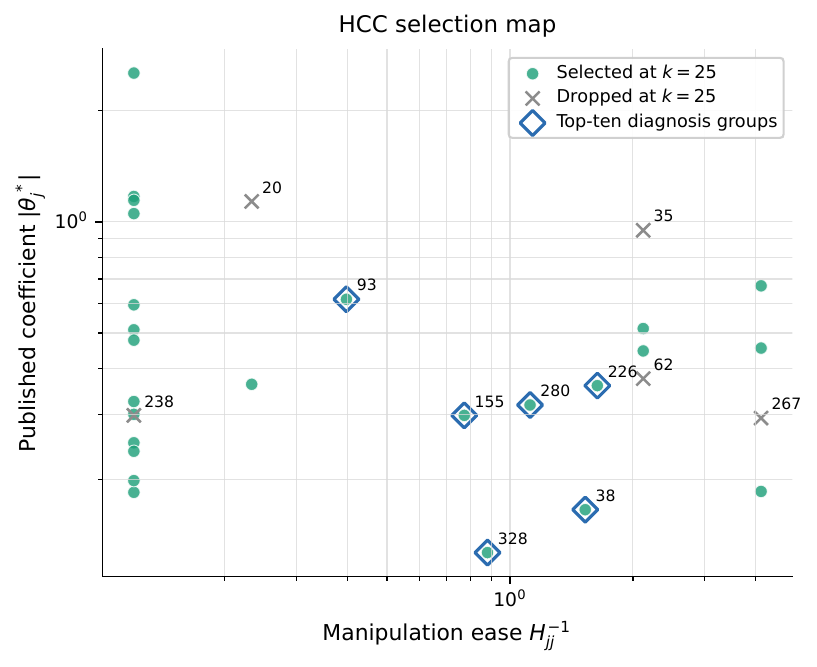}
    \caption{
    \textbf{Signal--manipulability selection map.}
    Each point is an HCC with nonzero prevalence in the data.
    Diamonds mark HCCs among the top-ten diagnosis groups identified by \citet{kronick2025are}, and labels are shown for the diagnosis groups that are dropped or belong to the top-ten set.
    The selected support is not determined by manipulation ease alone: the $k=25$ design retains several of these top-ten diagnosis groups with nontrivial signal while dropping other HCCs.
    }
    \label{fig:supp-anyfix-signal-manipulability}
\end{figure}

\section{Theoretical Details and Proofs for \Cref{sec:cost-uncertainty}}
\label{sec:app-cost-uncertainty}

Throughout this section, recall that we write $K:=H^{-1}$ and make the dependence of the strategic MSE on $K$ explicit:
\[
\MSE_K(\theta,b)
=
(\theta-\theta^*)^\top\Sigma(\theta-\theta^*)
+
(b+\theta^\top K\theta)^2
+
\sigma^2.
\]

\subsection{Proof of \Cref{prop:fixed-rule-worst-case}}

Fix $(\theta,b)$. 
The prediction term $(\theta-\theta^*)^\top\Sigma(\theta-\theta^*)$
does not depend on $K$. 
Thus
\[
\sup_{K\in\mathcal K}
\left\{
\MSE_K(\theta,b)-\sigma^2
\right\}
=
(\theta-\theta^*)^\top\Sigma(\theta-\theta^*)
+
\sup_{K\in\mathcal K}
(b+\theta^\top K\theta)^2 .
\]
Let $\underline q(\theta):=\inf_{K\in\mathcal K}\theta^\top K\theta$ and $\overline q(\theta):=\sup_{K\in\mathcal K}\theta^\top K\theta$.
Since $K\mapsto\theta^\top K\theta$ is continuous and $\mathcal K$ is compact, the extrema are attained. 
Moreover, for any $q\in[\underline q(\theta),\overline q(\theta)]$, the function $q\mapsto(b+q)^2$ is convex, so its maximum over the range is attained at an endpoint. 
Therefore
\[
\sup_{K\in\mathcal K}
(b+\theta^\top K\theta)^2
=
\max\{
(b+\underline q(\theta))^2,\,
(b+\overline q(\theta))^2
\}.
\]
This proves the first display. 
If $b=-\theta^\top\widehat K\theta,$ then $b+\theta^\top K\theta=\theta^\top(K-\widehat K)\theta$,
which gives the plug-in expression.

\subsection{Details for \Cref{ex:robustness-regularization-alone}}

Let $d=1$, $\Sigma=1$, $\theta^*=1$, and $\mathcal K=[1/2,5/2]$. 
The nominal estimate is $\widehat K=3/2$. 
The plug-in intercept correction deploys $(\theta,b)=(1,-3/2)$. 
Its worst-case excess strategic MSE is
\[
\sup_{K\in[1/2,5/2]}(K-3/2)^2
=
1.
\]

For ridge with $\lambda=1$, the coefficient is
\[
\theta^{\OR}(\lambda)=\frac{1}{1+\lambda}=\frac12 .
\]
Its worst-case excess strategic MSE therefore is
\[
\begin{aligned}
\left(\frac12-1\right)^2
+
\sup_{K\in[1/2,5/2]}
\left(K\left(\frac12\right)^2\right)^2
=
\frac14
+
\left(\frac{5}{2}\cdot\frac14\right)^2
=
\frac14+\frac{25}{64}
=
\frac{41}{64}
<1.
\end{aligned}
\]

\subsection{Details for \Cref{ex:robustness-feature-selection-regularization}}

Consider
\[
    \theta^*=(1,0),
    \qquad
    \Sigma=
    \begin{pmatrix}
    1 & 49/50\\
    49/50 & 1
    \end{pmatrix},
\]
and
\[
    K(k)=
    \begin{pmatrix}
    k & 0\\
    0 & 1/2
    \end{pmatrix},
    \qquad
    k\in[1/2,5/2].
\]
The nominal estimate is $\widehat K=K(3/2)$.

For the plug-in intercept correction with slope $\theta^*$, the deployed rule is $(\theta,b)=\left((1,0),-3/2\right)$.
Since $\theta^{*\top}K(k)\theta^* = k$, the worst-case excess strategic MSE is
\[
\sup_{k\in[1/2,5/2]}(k-3/2)^2=1.
\]

Now consider full-support ridge with $\lambda=1$. 
Since $\Sigma\theta^*=(1, 49/50)$, we have
\[
\theta^{\OR}([2],1)
=
(\Sigma+I)^{-1}\Sigma\theta^*
=
\left(
\frac{2599}{7599},
\frac{2450}{7599}
\right).
\]
The worst case occurs at $k=5/2$, and direct substitution gives the worst-case excess strategic MSE of approximately $0.240$.

Next consider support-restricted ridge on $\mathsf S=\{2\}$ with $\lambda=1/4$. 
The retained coefficient is
\[
\beta
=
\frac{\Sigma_{2y}}{\Sigma_{22}+\lambda}
=
\frac{49/50}{1+1/4}
=
\frac{98}{125}.
\]
Thus
\[
\theta^{\OR}(\{2\},1/4)
=
\left(0,\frac{98}{125}\right).
\]
The prediction term is
\[
1
-
2\left(\frac{49}{50}\right)\left(\frac{98}{125}\right)
+
\left(\frac{98}{125}\right)^2 .
\]
The strategic term is independent of $k$, because the first coordinate is dropped:
\[
\left[
\left(0,\frac{98}{125}\right)^\top
K(k)
\left(0,\frac{98}{125}\right)
\right]^2
=
\left[
\frac12
\left(\frac{98}{125}\right)^2
\right]^2 .
\]
Combining the two terms yields the worst-case excess strategic MSE of approximately $0.172 < 0.240$.

Finally, support-restricted least squares on $\mathsf S=\{2\}$ gives $\theta^{\LS}(\{2\}) = (0, 49/50)$.
Its worst-case excess strategic MSE is
\[
1-\left(\frac{49}{50}\right)^2
+
\left[
\frac12
\left(\frac{49}{50}\right)^2
\right]^2
\approx 0.270.
\]
Thus, support restriction alone is not as effective as combining support restriction with ridge shrinkage in this example.

\end{document}